\documentclass{article}
\usepackage{iclr2027_conference,times}
\usepackage{amsmath,amsfonts,bm}

\def\eqref#1{equation~\ref{#1}}
\def\1{\bm{1}}

\DeclareMathAlphabet{\mathsfit}{\encodingdefault}{\sfdefault}{m}{sl}
\SetMathAlphabet{\mathsfit}{bold}{\encodingdefault}{\sfdefault}{bx}{n}

\newcommand{\KL}{D_{\mathrm{KL}}}

\usepackage{amsmath,amssymb,amsthm}
\usepackage{hyperref,url}
\usepackage{booktabs,array,graphicx,xcolor,multirow,wrapfig,enumitem}
\usepackage{textcomp,flafter}
\graphicspath{{./}}
\input{ts1ptm.fd}
\DeclareFontShape{TS1}{ptm}{m}{sc}{<->ssub*ptm/m/n}{}

\newtheorem{theorem}{Theorem}
\newtheorem{proposition}[theorem]{Proposition}
\newtheorem{lemma}[theorem]{Lemma}
\newtheorem{corollary}[theorem]{Corollary}
\theoremstyle{definition}
\newtheorem{definition}[theorem]{Definition}

\newtheorem{fact}{Fact}
\theoremstyle{remark}
\newtheorem{remark}[theorem]{Remark}
\newcommand{\gauss}{\gamma}
\newcommand{\TV}[1]{\left\lVert #1 \right\rVert_{\mathrm{TV}}}
\newcommand{\norm}[1]{\left\lVert #1 \right\rVert}
\newcommand{\Law}{\mathrm{Law}}
\newcommand{\ie}{i.e.\ }
\newcommand{\eg}{e.g.\ }
\newcommand{\Kbar}{\bar{K}}
\newcommand{\Exp}{\mathbb{E}}
\newcommand{\dd}{\mathrm{d}}

\makeatletter
\let\wm@underscore\_
\renewcommand{\_}{\wm@underscore\discretionary{}{}{}}
\newenvironment{wmtable}[2]
 {\par\addvspace{\intextsep}\noindent\begin{minipage}{\linewidth}%
  \def\@currentlabelname{#1}\refstepcounter{table}\label{#2}%
  \@makecaption{\tablename~\thetable}{#1}\par\nobreak
  \begingroup\parindent\z@\centering}
 {\par\endgroup\end{minipage}\par\addvspace{\intextsep}}
\makeatother

\title{World Models with Predictable\\Long-Horizon Marginals}
\iclrfinalcopy
\author{Yuhao DU$^{1,2}$ \quad Shunian Chen$^1$\\
\textnormal{\small $^1$The Chinese University of Hong Kong, Shenzhen
\quad $^2$Shenzhen Loop Area Institute}\\
\textnormal{\small\texttt{yuhaodu1@link.cuhk.edu.cn}}}
\hypersetup{
  pdftitle={World Models with Predictable Long-Horizon Marginals},
  pdfauthor={Yuhao DU; Shunian Chen}
}

\begin{document}
\maketitle
\lhead{Preprint}
\begin{abstract}
Accurate one-step predictions do not ensure that a world model's rollouts retain the data distribution.
We make the model's decoded stationary law explicit by learning a decoder of a fixed Gaussian
reference and constraining the behaviour-averaged transition to preserve that reference.
For controlled systems, a joint transition uses a conditional action chart to preserve behaviour occupancy without
requiring invariance at each fixed action. Joint state--action rotations and parallel Gaussian noise give an exactly
preserving transition with a tractable conditional density. We derive an absolute convergence
bound from finite initialization banks and control departure from the reference through
conditional action-space divergence. Across $216$ fitted pixel checkpoints on twelve control
tasks, the occupancy model with a reference mixture retains every evaluated chain at $10^5$ steps in all $36$
task--seed cells, with a rollout-minus-reference energy-statistic difference of $-0.0002\pm0.0003$ (training-seed standard error).
Each of the four nonpreserving comparison arms loses chains, although the Gaussian arm is more
accurate at ten steps. An offline DreamerV3 reference also achieves better short-horizon accuracy.
These results distinguish three properties of a world model: the distribution it approaches,
the rate of approach, and the conditional dynamics it learns.
\end{abstract}

\section{Introduction}

World models support prediction and planning by repeatedly applying a learned transition.
Their outputs become future inputs, so accurate one-step predictions are insufficient: a rollout
can collapse, expand, or leave the region represented in the data. Conversely, trajectories in a
chaotic system can separate even when both obey the correct dynamics. The long-horizon marginal
is therefore a distinct object of prediction.

Teacher-forced likelihood scores transitions under the data law, whereas a rollout visits the
model's evolving distribution (Appendix~\ref{app:path-laws}). Even in a first-order autoregressive
process, a coefficient perturbation can incur second-order conditional divergence while changing
the invariant variance to first order (Appendix~\ref{app:knife}). Small local error alone need not
control the long-run law.

We address this problem by \textbf{making the model's stationary law explicit}.
A stationary data law satisfies $\mu\Kbar^\star=\mu$, where $\Kbar^\star$ averages the true transition
under the behaviour policy. In suitable latent coordinates, this identity becomes preservation
of a standard Gaussian reference $\gauss$. We learn a decoder of this reference and preserve it
in the transition; matching the resulting observation law $p_\theta$ to data remains a learned
approximation. The transition determines the conditional dynamics and how quickly other
initial laws approach that reference. For a measured discrepancy from data, this yields
\[
\underbrace{D_T}_{\text{rollout discrepancy}}
=\underbrace{D_\infty}_{\text{decoded-reference discrepancy}}
+\underbrace{r_T}_{\text{signed transient}},\qquad |r_T|\le\text{an explicit bound}.
\]
Independent reference draws estimate this model-implied limiting level before any long rollout.
The transition must still learn useful dynamics.

Invariant-measure objectives fit long-run statistics \citep{schiff2024dyslim}, and learned
Markov chain Monte Carlo transitions preserve a specified target \citep{song2017anicemc}.
A controlled world model must also fit action-dependent dynamics. Our construction combines
behaviour occupancy preservation, an exact conditional density, and finite-initialization bounds
in one parallel transition.

Requiring invariance at each action can exclude valid controlled dynamics. We instead Gaussianize
behaviour actions and preserve the joint state--action reference, enforcing the correct averaged
constraint (Figure~\ref{fig:architecture}).

Our contributions are:
\begin{itemize}[leftmargin=*]
\item \textbf{A controlled preserving architecture.} A joint state--action mixer and parallel noise
stage enforce behaviour occupancy while retaining an exact conditional density. Representation
and affine-coupling analysis identify the structural constraints (\S\ref{sec:construction}--\ref{sec:architecture}).
\item \textbf{Computable marginal guarantees.} Applying Gaussian entropy contraction to the joint
innovation gives an absolute finite-bank bound and controls policy deviation in physical
action-space Kullback--Leibler (KL) divergence (\S\ref{sec:certificates}).
\item \textbf{A separation of model properties.} Static-reference comparisons, controlled examples,
and $216$ pixel checkpoints distinguish marginal predictability from conditional accuracy and
planning return (\S\ref{sec:experiments}).
\end{itemize}

\section{Stationary data as a preserved reference}
\label{sec:construction}
\label{sec:leakage}

Let $\mu$ be a stationary state law and $K^\star(x,a,\dd x')$ its controlled transition under a fixed
behaviour policy $\pi_b$. The relevant constraint is
\begin{equation}
\iint\mu(\dd x)\pi_b(\dd a\mid x)K^\star(x,a,\dd x')=\mu(\dd x').
\label{eq:occupancy-main}
\end{equation}
It does not require every action kernel to preserve $\mu$: independent uniform $a\in\{0,1\}$ and
$x'=a$ give stationary uniform states, although each fixed action collapses the state law.

\paragraph{A stationary representation.}
Let $\gauss_d=\mathcal N(0,I_d)$ and let $E_\#\mu$ denote the pushforward of $\mu$ by $E$.
For an invertible chart satisfying $E_\#\mu=\gauss_d$, transport of the kernel gives
$\mu\Kbar=\mu\iff\gauss_d\tilde K=\gauss_d$.
With arbitrary measurable charts and kernels, this represents precisely the pair laws with equal
atomless marginals (Proposition~\ref{prop:exact-class}). Differentiable density representations require
the transport regularity stated in Appendix~\ref{app:transport}. Finite networks cover only part of
this constraint set; the theorem characterizes pair laws, not the reach of any fixed architecture.

\paragraph{The conditional action chart.}\label{sec:controlled}
In latent coordinates, transform the behaviour action law by
\begin{equation}
\alpha=G(a\mid z),\qquad G(\cdot\mid z)_\#\pi_b(\cdot\mid z)=\gauss_m.
\label{eq:chart-body}
\end{equation}
Then $\alpha$ is Gaussian independently of $z$. A transformation preserving the \emph{joint}
reference of $(z,\alpha)$ can move information from action to state while preserving occupancy.
With a state-independent chart, joint preservation strictly contains fixed-action preservation.
With state-dependent charts, these classes need not be nested: fixing a chart coordinate and
fixing a physical action are different interventions (Proposition~\ref{prop:nesting},
Remark~\ref{rem:nesting-state}).

For the recurrent pixel model, a frozen tanh-Gaussian policy reads the current frame $o$:
$a=\tanh(m_b(o)+s_b\odot\alpha)$, with known $s_b>0$. Recorded pre-tanh actions provide the exact
training chart. During generation the policy reads the decoded observation, so its innovation
remains independent of the latent state (Appendix~\ref{app:controlled}). A learned approximate
chart requires additional residual-error control.

\section{Gaussian-preserving transitions}
\label{sec:architecture}

\begin{figure}[t]
\centering
\includegraphics[width=\textwidth]{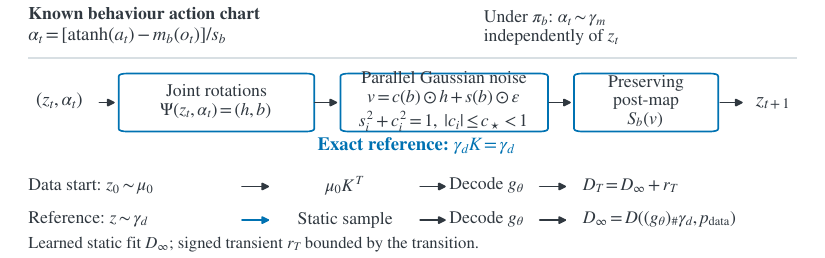}
\caption{\textbf{A transition with an exactly preserved Gaussian reference.}
Latent state $z$ and charted action $\alpha$ feed joint rotations under behaviour $\pi_b$, a parallel Gaussian noise
stage with $s_i^2=1-c_i^2$, and a preserving post-map, so the Gaussian reference holds at every
parameter value. Decoding the reference gives the
static observation law $p_\theta$. The rollout discrepancy from data equals the static discrepancy
plus a bounded signed transient $r_T$ (Section~\ref{sec:certificates}).}
\label{fig:architecture}
\end{figure}

For a diffeomorphism $T$, Gaussian preservation is equivalent to
\begin{equation*}
\log|\det J_T(x)|=\tfrac12\bigl(\|T(x)\|^2-\|x\|^2\bigr).
\tag{$\dagger$}\label{eq:dagger}
\end{equation*}
It couples radial displacement to volume change and determines the admissible form of affine couplings.

\begin{theorem}[Gaussian-preserving affine couplings]\label{thm:H3}
The block-affine diffeomorphism $(u,v)\mapsto(u,A(u)v+b(u))$ preserves a standard Gaussian if and only if
$b(u)=0$ and $A(u)A(u)^\top=I$. Thus every admissible coupling is
\begin{equation}
(u,v)\longmapsto(u,Q(u)v),\qquad Q(u)\in\mathrm O(n_v).
\label{eq:rotcoup}
\end{equation}
\end{theorem}

This characterization (Appendix~\ref{app:preserving-diffeomorphisms}) gives the building blocks
(Figure~\ref{fig:architecture}). Conditioned Givens rotations
implement this identity without projection: we update disjoint coordinate
pairs in parallel, change the fixed partition between layers, and with a fixed number of layers the
depth no longer grows with latent width, as it would in a scan. These maps preserve each sample's norm.
Complementary norm-changing maps take the form $\Phi^{-1}\circ S\circ\Phi$, where $\Phi$ applies the standard Gaussian cumulative distribution function coordinatewise
and $S$ preserves volume in the quantile cube. Their approximation theorem allows unrestricted smooth profiles and depth
(Appendix~\ref{app:ideal-representation}); finite-network capacity is a separate constraint. The stochastic architecture
below uses rotations for both deterministic stages, with radial change supplied by its noise.

\paragraph{Joint mixing and simultaneous refresh.}
Mix state and action innovations using a norm-preserving Gaussian bijection $\Psi$:
\begin{equation}
(h,b)=\Psi(z,\alpha),\qquad
v_i=c_i(b)h_i+\sqrt{1-c_i(b)^2}\,\varepsilon_i,\qquad z'=S_b(v),
\label{eq:kernel}
\end{equation}
where $\varepsilon\sim\gauss_d$ is fresh, $|c_i(b)|\le c_\star<1$, and $S_b$ preserves $\gauss_d$.
At the reference law, $h$ and $b$ are independent Gaussians. Conditional on $b$, the simultaneous noise
stage and the post-map retain $\gauss_d$. Every recurrent coordinate receives noise; leaving deterministic memory outside this state
can prevent contraction (Appendix~\ref{app:e26}). The coefficient network reads $b$, not the
coordinates being refreshed. In code, $c_i=c_\star\tanh\ell_i(b)$.

Given $(z,a,o)$, including any observation context used by the action chart, both $h$ and $b$
are known. For a diffeomorphic post-map $S_b$, let $v=S_b^{-1}(z')$; the exact conditional log density is
\begin{equation}
\log k_\theta(z'\mid z,a,o)
=\sum_i\log\mathcal N\!\left(v_i;c_i(b)h_i,1-c_i(b)^2\right)
+\log|\det J_{S_b^{-1}}(z')|.
\label{eq:parallel-density-main}
\end{equation}
The inverse-Jacobian term vanishes for the implemented rotational post-map. The input mixer has no
Jacobian term in this \emph{conditional} density because it transforms known conditioning variables.
The sequential block scan used in the fixed-latent transition benchmark is a separate preserving
parameterization, described in Appendices~\ref{app:arch} and~\ref{app:cert-mix}.

\paragraph{An optional reference component.}
The mixture $k_r=(1-r)k_\theta+r\gauss_d$ preserves the reference and refreshes the entire
state with probability $r$. Its additional conditional loss is at most $-\log(1-r)$:
$0.001001$ nats per transition at the recurrent study's default $r=10^{-3}$.
The reference component also bounds extreme log-loss relative to the fitted marginal
(Appendix~\ref{app:reference-component}). We examine its density responsibility in frozen-latent
checkpoints and evaluate a reset-free recurrent variant.

\paragraph{Training with observations and recurrent inference.}
Invertible state charts use the pair objective
$-\log p_\theta(x'\mid x,a)-w\log p_\theta(x)$, with marginal weight $w$.
For pixels, a convolutional decoder and recursive Gaussian posterior
$q_\varphi(z_t\mid z_{t-1},o_t,a_{t-1})$ are learned jointly with the transition and reward head.
Training minimizes the negative evidence lower bound (ELBO) of the joint behaviour-generated
sequence model, up to known behaviour-action log densities. It combines observation and reward
reconstruction, an initial KL to $\gauss_d$, and transition KL terms using the exact prior density.
Observation $o_t$ is emitted before $z_{t+1}$ and supplies its action-chart context; the posterior
approximates the latent posterior. Appendix~\ref{app:recurrent-model} gives the
factorization and objective; Appendix~\ref{app:gpu-protocol} specifies the comparisons.

\section{Marginal certificates}
\label{sec:certificates}

Write $g_\theta$ for the observation map, including an inverse chart and any decoder.
The static reference image is $p_\theta=(g_\theta)_\#\gauss_d$. For a stochastic observation model the
same argument applies to its observation kernel; pixel-mean metrics below use the deterministic decoder
mean. This formulation accommodates noninvertible pixel decoders.

\begin{theorem}[Zero marginal drift]\label{thm:A}
Under the behaviour action chart---including observation-conditioned charts, by
Lemma~\ref{lem:observation-chart}---the transition in \eqref{eq:kernel}, with or without the
full-state
reference mixture, preserves $\gauss_d$ for every parameter value. If $z_0\sim\gauss_d$, then
$\Law(z_T)=\gauss_d$ and $\Law(g_\theta(z_T))=p_\theta$ at every horizon.
\end{theorem}
The guarantee concerns the generative prior; recurrent inference learns an approximate posterior.

\paragraph{Static level and signed transient.}
Let $K$ be the behaviour-averaged latent kernel, $\mu_0$ its initialized law, and $A_T=(g_\theta)_\#(\mu_0K^T)$.
For a metric $D$ satisfying $D(P,Q)\le C_D\TV{P-Q}$, set
$D_T=D(A_T,p_{\rm data})$ and $D_\infty=D(p_\theta,p_{\rm data})$,
using the total-variation (TV) norm in the $[0,2]$ convention.

\begin{proposition}[Static prediction and transient bound]\label{thm:R}
For any reference-preserving transition,
\begin{equation}
D_T=D_\infty+r_T,\qquad
|r_T|\le C_D\TV{\mu_0K^T-\gauss_d}\le C_D\TV{\mu_0-\gauss_d}.
\label{eq:plateau}
\end{equation}
Consequently, total-variation convergence to the reference implies $D_T\to D_\infty$.
\end{proposition}

Reverse triangle inequality and data processing give the bound (Appendix~\ref{app:static-transient}).
For observations of diameter $\Delta$, the squared energy statistic instead satisfies
$|\mathcal E(A_T,Q)-\mathcal E(p_\theta,Q)|\le2\Delta\TV{\mu_0K^T-\gauss_d}$
(Appendix~\ref{app:finite-warm-start}). Normalized root-mean-square (RMS) pixel distance has $\Delta\le1$.
Unbounded state-space distances require additional moment control.

Preservation alone permits the identity kernel; the parallel noise stage additionally supplies relaxation.

\begin{theorem}[Joint contraction and finite-start convergence]\label{thm:parallel-main}
Let $M$ be the kernel from $(z,\alpha)$ to $z'$ in \eqref{eq:kernel}, with Gaussian-preserving
$\Psi$ and $S_b$ and uniform cap $|c_i(b)|\le c_\star<1$. For any joint input law $\nu$,
\begin{equation}
\KL(\nu M\|\gauss_d)\le c_\star^2\KL(\nu\|\gauss_{d+m}).
\label{eq:rate-body}
\end{equation}
Under an exact behaviour chart, $\nu=\mu\otimes\gauss_m$, giving state-kernel KL contraction by
$c_\star^2$. If $\Psi$ also preserves norms pointwise and $M_2=\Exp_{\mu_0}\|z\|^2<\infty$, define
\begin{equation}
B_\star=\tfrac12\left[c_\star^2(M_2+m)-dc_\star^2-d\log(1-c_\star^2)\right].
\label{eq:warm-main}
\end{equation}
Then, for $T\ge1$,
\begin{equation}
\TV{\mu_0K^T-\gauss_d}\le
\min\{2,\sqrt{2B_\star}\,c_\star^{T-1}\}.
\label{eq:absolute-main}
\end{equation}
\end{theorem}

\emph{Proof sketch.} After $\Psi$, the KL chain rule separates the residual action $b$ from
the conditional state $h\mid b$. At fixed $b$, Gaussian entropy contraction and post-map data processing
give $c_\star^2$; convexity removes $b$. Norm preservation bounds $\Exp\|h\|^2$ by $M_2+m$.
Convexity bounds first-step KL by $B_\star$, even from atomic initial laws. Iteration and Pinsker give
\eqref{eq:absolute-main} (Appendices~\ref{app:joint-entropy}--\ref{app:finite-warm-start}).

For a specified finite initialization bank, $M_2$ is exactly computable; a population certificate
requires a population moment bound. The optional reference mixture gives the independent bound $2(1-r)^T$ for \emph{any} initial
law and strengthens the entropy contraction to $(1-r)c_\star^2$.
Across the fitted initialization banks in \S\ref{sec:exp:recurrent},
\eqref{eq:absolute-main} becomes informative after $3\times10^3$--$10^4$ steps.
At $10^5$ steps, the reference-mixture bound puts the population marginal gap below
$1.5\times10^{-43}$, far below evaluation sampling error. The terminal comparison therefore
tests the predicted static level; intermediate horizons assess the finite-start bound.
These are real-arithmetic guarantees, with numerical scope detailed in
Appendix~\ref{app:conditional-numerics}.

\paragraph{Policy deviation in physical action space.}
For a bijective exact chart, divergence is invariant under its action transformation.
Let $\lambda_t=\operatorname{Law}(Z_t,O_t)$, with state marginal $\mu_t$, and define
$H_t=\KL(\mu_t\|\gauss_d)$ and
$\kappa_t=\Exp_{\lambda_t}\KL(\pi_t(\cdot\mid z,o)\|\pi_b(\cdot\mid z,o))$.
The transition reads $o$ only through the action innovation.
For finite $H_0$, the joint-input argument gives
\begin{equation}
H_{t+1}\le c_\star^2(H_t+\kappa_t),\qquad
H_T\le c_\star^{2T}H_0+\sum_{t<T}c_\star^{2(T-t)}\kappa_t.
\label{eq:policy-kl-main}
\end{equation}
Pinsker gives a marginal TV bound. When both policies are nonsingular tanh-Gaussians with the
same tanh transform, their conditional KL is the Gaussian KL before tanh; deterministic action
selection has infinite KL against this behaviour law. The bound concerns departure from the
behaviour reference, not predictive error. Fitted charts require divergence to their induced policy
(Appendix~\ref{app:policy-physical}; numerics in Figure~\ref{fig:certificate}).

\section{Experiments}
\label{sec:experiments}

We evaluate static marginal predictions, finite-start bounds, conditional fit, and planning.
A chain \emph{escapes} if its state or decoded observation becomes non-finite; recurrent evaluation
also stops a chain when any latent coordinate exceeds $\pm10^6$. \emph{Survival} is the fraction
remaining at the stated horizon, a numerical criterion distinct from distributional fidelity.
We report negative log-likelihood (NLL), mean-squared error (MSE), and marginal energy statistics;
standard errors (SEs) distinguish evaluation sampling from training-seed variation
(Appendix~\ref{app:refbias}).

Pixel studies use the DeepMind Control Suite \citep{tassa2018dmc}. The marginal statistic is
$\mathcal E(P,Q)=2\mathbb E\delta(X,Y)-\mathbb E\delta(X,X')-\mathbb E\delta(Y,Y')$,
with independent draws. The distance $\delta$ is normalized root-mean-square pixel distance in the
recurrent/native study; fixed-latent studies use Euclidean distances in the stated latent or pixel
coordinates (Appendix~\ref{app:refbias}). Finite unbiased estimates can be negative.
$D_{\rm ref}$ estimates the static level $D_\infty$; pixel scores use decoder means.

\subsection{Static prediction of the long-horizon marginal}
\label{sec:exp:plateau}

Can a model's limiting discrepancy be predicted before running it to a long horizon?
Figure~\ref{fig:plateau} compares static reference estimates with $10^5$-step rollouts.

\begin{figure}[t]
\centering
\includegraphics[width=\textwidth]{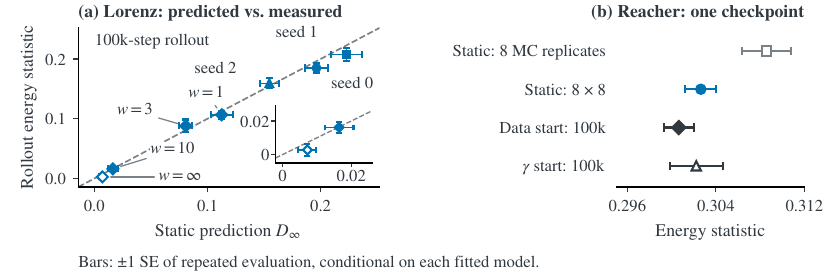}
\caption{\textbf{Static reference samples predict model-specific long-horizon discrepancies.}
\textbf{(a)} Seven Lorenz models span a $31\times$ range of predicted energy statistics;
$w$ is the marginal-loss weight.
The inset enlarges the two smallest values; an open marker denotes a value below evaluation
resolution. \textbf{(b)} One capped Reacher checkpoint, using Euclidean distance in 256-dimensional
frozen-autoencoder coordinates: the static estimate ($8$ Monte Carlo replicates), resampling
(eight 2,048-point clouds per reference block; SE over eight blocks), and data- and prior-initialized
$10^5$-step rollouts. The offset is consistent with static-estimate sampling variability.}
\label{fig:plateau}
\end{figure}

\paragraph{Model-specific levels.}
Seven constrained Lorenz-63 models---three training seeds and four marginal-weight fits---span a
$31\times$ range of static energy statistics. Each $10^5$-step measurement is within one nominal
combined SE of its prediction; this SE omits the unavailable prediction--rollout covariance.
Every constrained chain remains finite. An unconstrained flow improves held-out pair likelihood
by $0.83$--$0.99$ nats per dimension, yet only $14\%$--$94\%$ of its chains survive.
Survivor distances describe the retained subpopulation
(Appendix~\ref{app:seeds}).

\paragraph{Marginal fidelity and conditional fit.}\label{sec:exp:inversion}
At fixed architecture and training budget, raising the marginal-loss weight from $1$ to $10$
reduces long-horizon error by $6.6\times$ while worsening pair NLL by $0.19$ nats per dimension
(Appendix~\ref{app:wsweep}). Conditional and marginal objectives select different tradeoffs.
Soft multistep, Jacobian, and invariant-measure penalties \citep{schiff2024dyslim} are compared
in Appendix~\ref{app:stabilizers}.

\paragraph{Fixed-latent pixel models.}\label{sec:exp:tier2}
The sequential preserving transition retains every chain on twelve tasks with frozen
256-dimensional autoencoder latents. Flow matching retains every chain on ten tasks, with smaller
pixel distance on five and larger on five. The sequential model preserves fixed-action kernels
(data and fitting: Appendices~\ref{app:screen}--\ref{app:fitting-budgets}; outcomes and alternate
metrics: Appendices~\ref{app:fixed-latent-results} and~\ref{app:metrics}). For the capped Reacher
checkpoint in Figure~\ref{fig:plateau}b, resampling and independent rollout replication support
sampling variability as the explanation of the static offset
(Appendices~\ref{app:static-resampling} and~\ref{app:reacher-replication}).

\subsection{Finite-start bounds and conditional likelihood}

Do the computable bounds hold for singular initial laws and known policy shifts, and what do they
leave unconstrained? We first check the implemented inequalities, then examine fitted conditional densities.

\begin{figure}[t]
\centering
\includegraphics[width=\textwidth]{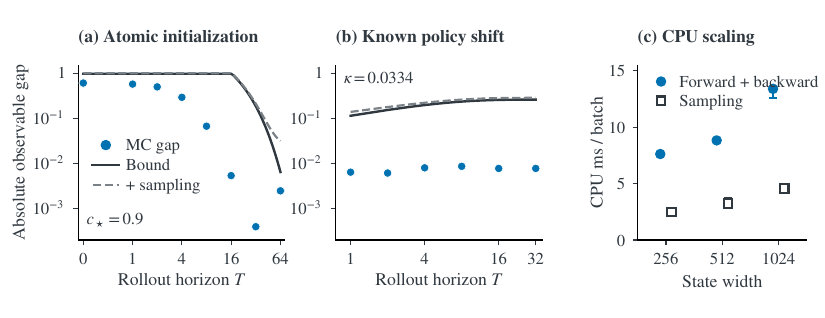}
\caption{\textbf{Finite-start and policy-shift checks for the parallel transition.}
\textbf{(a)} A 16-point initialization bank supplies the moment in \eqref{eq:absolute-main}.
\textbf{(b)} A known policy shift has conditional action KL $\kappa=0.0334$.
Dots are Monte Carlo gaps for $f(z)=\exp(-\|z\|^2/16)$, with 8,192 independent chains per
horizon. Solid curves are population bounds; dashed curves add a simultaneous $99\%$ Hoeffding
allowance. \textbf{(c)} Median forward/backward and sampling times per batch of 32, on four CPU threads,
with min--max whiskers over five repetitions (prior only).
These are randomized-kernel diagnostics without dynamics fitting (Appendix~\ref{app:cpu-diagnostics}).}
\label{fig:certificate}
\end{figure}

\paragraph{Atomic starts and policy shift.}
A finite initialization bank has infinite KL to a Gaussian but a known second moment.
Across contraction caps, reset probabilities, and tanh-Gaussian policy shifts, all $114$
bounded-observable comparisons satisfy the population bounds with the simultaneous sampling
allowance (Figure~\ref{fig:certificate}). The diagnostics in Appendix~\ref{app:atomic-diagnostics} check the implementation of the proved
inequalities through bounded observables; predictive accuracy is evaluated separately below.
Panel (c) isolates prior computation: sampling takes $2.50$--$4.57$ ms per batch of 32 as
state width increases from $256$ to $1024$. These CPU timings exclude the decoder and simulator
(Table~\ref{tab:cpu-scaling}).

\paragraph{Conditional fit in frozen-latent models.}
The frozen-latent sequential-scan checkpoints illustrate what preservation leaves unconstrained.
On $479{,}968$ held-out Reacher transitions, the cap reduces maximum conditional NLL from
$1.41\times10^8$ to $193$ nats per dimension (Appendix~\ref{app:conditional-tails}).
The capped Swimmer model's median remains $294$. Adding the reference mixture without retraining
reduces that median to $6.79$, but the reference branch dominates $98.3\%$ of transition densities.
Most of this reduction in scored loss comes from the static reference component
(Table~\ref{tab:cpu-guard}, Appendix~\ref{app:reference-mixture}).

\subsection{Controlled occupancy and planning}
\label{sec:exp:controlled}

Does preserving occupancy avoid the restrictions of preserving every fixed-action kernel?
With independent Gaussian behaviour actions ($\tau=0$), the stationary action-driven
system has a fixed-action-preserving
conditional-loss floor $-\tfrac12\log(1-\varrho^2)$, where $\varrho$ controls the
action-to-state correlation. Across four $\tau=0$ settings spanning a $25\times$ range of floors,
its learned excess loss is $0.99$--$1.00$ times the prediction, while the occupancy model's
excess is $0.0011$ nats (Appendix~\ref{app:controlled-gaussian}).
Control also depends on the fitted dynamics. On a controlled Duffing oscillator,
occupancy-based cross-entropy-method (CEM) planning approaches the oracle at short horizons but
incurs mean excess cost $0.242$ at horizon $500$; the Gaussian head's excess is $0.001$.
The per-action comparison includes differences in fit and transition size
(Appendix~\ref{app:planning}).

\subsection{Jointly trained recurrent pixel models}
\label{sec:exp:recurrent}

Does the separation persist when the representation, inference, and dynamics are learned jointly?
We train six arms on twelve screened control tasks with three paired training seeds, yielding
$216$ checkpoints. All share the observation, recursive posterior, and reward networks.
The occupancy transition uses joint state--action mixing; a conditioner-only ablation removes
that mixing. Four additional arms test a Gaussian transition, conditional RealNVP, latent
overshooting, and an invariant-measure penalty. The latter two test multistep and marginal
regularization against structural preservation. Collection uses a fixed frame-conditioned
tanh-Gaussian behaviour policy (Appendix~\ref{app:gpu-protocol}).

\paragraph{Comparison protocol.}
The arms use common data and paired seeds, but differ in achieved conditional fit.
None of the five alternative families meets the suite-wide validation matching criteria
(prediction-error, likelihood, parameter, and compute gates in Appendix~\ref{app:recurrent-evaluation}).
We therefore report fitted-model contrasts under these common budgets.
A same-checkpoint output-gain intervention meets these criteria, with gain $1.001$
selected on validation data. Table~\ref{tab:recurrent-main} reports equal-task means; the
nominal $95\%$ percentile intervals resample three paired seeds within each fixed task, assuming
independent task--seed cells. SEs follow the same convention; Appendix~\ref{app:refbias}
examines the reused seed IDs and small-sample calibration.

\begin{table}[ht]
\centering
\caption{\textbf{Conditional accuracy and long-horizon stability across recurrent models.}
Equal-task means, twelve tasks $\times$ three paired seeds; NLL, MSE, and gaps show training-seed SE.
NLL is a posterior-initialized particle score under recorded actions, in nats per pixel channel.
$D_{\rm ref}$ scores decoded Gaussian reference states, a limiting level only for preserving
models under their convergence assumptions. ``Escape'': no gap estimated after censoring.
The last row scales the fitted occupancy model's output. Bold marks best primary-arm means;
the signed gap has no directional optimum.}
\label{tab:recurrent-main}
\setlength{\tabcolsep}{3.3pt}
\begin{tabular}{@{}lcccc@{}}
\toprule
Transition / objective & Survival & Obs.\ NLL $\downarrow$ & 10-step MSE $\downarrow$ & $D_{10^5}-D_{\rm ref}$ \\
\midrule
Occupancy & \textbf{1.0000} & $0.6843\!\pm\!0.6832$ & $0.1271\!\pm\!0.0211$ & $-0.0002\!\pm\!0.0003$ \\
Conditioner only & 0.9999 & $\mathbf{0.5648\!\pm\!0.5422}$ & $0.1416\!\pm\!0.0247$ & Escape \\
Gaussian & 0.5149 & $1.7013\!\pm\!1.1747$ & $\mathbf{0.0708\!\pm\!0.0218}$ & Escape \\
Conditional RealNVP & 0.2394 & $3.1391\!\pm\!1.5809$ & $0.1119\!\pm\!0.0264$ & Escape \\
Latent overshooting & 0.3380 & $7.0333\!\pm\!1.0566$ & $0.1692\!\pm\!0.0212$ & Escape \\
Invariant penalty & 0.4399 & $4.0644\!\pm\!1.6440$ & $0.1120\!\pm\!0.0327$ & Escape \\
\midrule
Occupancy $\times 1.001$ & 1.0000 & $0.6845\!\pm\!0.6832$ & $0.1273\!\pm\!0.0211$ & $-0.0003\!\pm\!0.0003$ \\
\bottomrule
\end{tabular}
\end{table}

\paragraph{Long-horizon survival and marginal agreement.}
The occupancy arm retains all $294{,}912$ evaluated chains across its $36$ task--seed cells.
Its task-mean static discrepancies span $0.0033$--$0.0758$, with mean rollout-minus-reference
gap $-0.0002\pm0.0003$ (SE; Figure~\ref{fig:marginal-gap}). The Gaussian, conditional RealNVP, latent-overshooting, and
invariant-penalty arms have survival fractions $0.51$, $0.24$, $0.34$, and $0.44$.
All four paired survival contrasts against occupancy have intervals excluding zero
(Figure~\ref{fig:survival-forest}). All three hopper stand Gaussian cells survive with marginal
gaps of $1.10$, $1.06$, and $-0.16$.
The conditioner-only arm reaches survival $0.99993$, with $21$ threshold crossings in one
Swimmer cell; its extreme posterior initialization is examined below.

\begin{figure}[t]
\centering
\includegraphics[width=\textwidth]{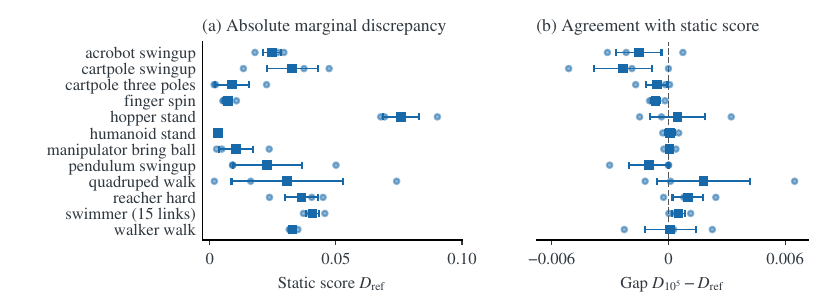}
\caption{\textbf{Absolute marginal discrepancy and static-prediction agreement.}
\textbf{(a)} Occupancy reference energy statistics against data, using normalized RMS pixel distance.
\textbf{(b)} Rollout-minus-reference discrepancies against the same data.
Points: three training seeds; squares: task means $\pm$ SE. The equal-task gap is
$-0.0002$, with SE $0.0003$ for independent task--seed cells and $0.0006$ when
grouping the reused seeds (Appendix~\ref{app:refbias}).}
\label{fig:marginal-gap}
\end{figure}

\paragraph{Conditional accuracy remains distinct.}
The Gaussian arm has the smallest measured mean ten-step MSE among the six arms (contrast against occupancy
$-0.0563$ $[-0.1067,-0.0024]$), despite losing almost half its chains.
Conversely, ten occupancy checkpoints have observation NLL above $1$, reaching $17.37$ on
pendulum swingup, while retaining every chain with marginal gaps within $\pm0.007$.
The conditioner-only variant also exceeds NLL $1$ in the same ten task--seed cells.
Its $21$ threshold crossings occur at the first step from an initialization bank with second
moment $7.4\times10^{10}$, compared with $256$ for the Gaussian reference; its reference-initialized
control retains every chain. These measurements distinguish posterior initialization from
reference-law preservation (Appendix~\ref{app:posterior-excursion}). The conditioner-only model also has
higher ten-step MSE than occupancy (contrast $+0.0145$ $[0.003,0.026]$).

\paragraph{Output gain and architectural ablations.}
Multiplying each prior output by $1.001$ removes exact preservation but leaves survival unchanged
at $10^5$ steps. The changes are small: observation NLL $+0.00019$ $[0.00012,0.00028]$,
ten-step MSE $+0.00015$ $[-0.00015,0.00041]$, and marginal gap $-0.00016$
$[-0.00029,-0.00004]$. An exploratory gain of $1.003$ also retains every chain in all $36$
cells, while the worst-cell second-moment deviation grows from $0.15$ to $0.54$ to $6.9$
for gains $1$, $1.001$, and $1.003$. Moment diagnostics thus distinguish these models even when
finite-horizon survival is identical (Table~\ref{tab:gain1003}).
The selected gain covers all $36$ primary cells; six retrained cap, reset, and width variants
cover four priority tasks ($12$ cells each). All have survival $1.0$; across these scopes, mean
gaps range from $-0.00052$ to $+0.00047$.
The default reset probability is $r=10^{-3}$; the reset-free variant retains all chains in its
$12$ evaluated cells with gap $-0.00025\pm0.00035$
(Figure~\ref{fig:recurrent-results}, Appendices~\ref{app:gain-control} and~\ref{app:recurrent-results}).

\paragraph{Planning tests a stronger requirement.}
Of $348$ planning evaluations, $172$ fail because imagination produces fewer finite candidates
than required CEM elites. Primary-arm counts range from $14$ to $19$ failures among $36$ cells;
occupancy and conditioner-only each have $19$. Five tasks account for $77\%$ of all failures
(Figure~\ref{fig:crash-census}). On commonly completed cells, equal-task mean return contrasts
against occupancy range from $-0.0128$ to $+0.0089$; these descriptive comparisons condition on
completion. The native reference completes all $12$ evaluations and leads in three of the four
cells where both it and occupancy complete. Completion and return measure the additional
requirements of action-selected planning (Appendix~\ref{app:common-planning}).

\paragraph{A native DreamerV3 reference.}
An offline DreamerV3 baseline \citep{hafner2023dreamerv3} retains the authors' categorical recurrent
state and observation networks on four priority tasks. Exposure is matched at $25.6$M frames;
capacity and context differ ($8.85$M versus $3.26$--$3.29$M parameters; $64$ versus $16$ frames
per window), and wall time is not matched. With a common external Gaussian image score, it achieves lower observation NLL and ten-step MSE
($0.0024\pm0.00004$ versus $0.1206\pm0.0244$). Both models retain every chain at the shared
$10^3$-step horizon. The native model's $10^5$-step marginal was not measured, and our Gaussian
certificate does not apply to its categorical state. It supplies a strong conditional-prediction
reference alongside the long-horizon property studied here
(Table~\ref{tab:native-results}, Figure~\ref{fig:native-panel}, Appendix~\ref{app:native-baseline}).

\section{Related work}
\label{sec:related}

\paragraph{World models and rollout error.}
Latent world models support prediction and planning
\citep{ha2018world,hafner2019planet,hafner2020dreamer,hafner2021dreamerv2,hafner2023dreamerv3}.
Multistep objectives and distribution-shift mitigation address errors along generated trajectories
\citep{bengio2015scheduled,ross2011reduction,talvitie2017selfcorrecting,chen2024diffusionforcing},
while control-oriented analyses distinguish likelihood from task-relevant fidelity
\citep{lambert2020objective}. Our constraint acts on the marginal law itself, complementing
objectives that improve conditional predictions. Occupancy methods estimate stationary-distribution
corrections for off-policy evaluation \citep{nachum2019dualdice,zhang2020gendice}; here a conditional
action chart makes occupancy preservation an identity of the learned generative model.

\paragraph{Invariant measures and structured dynamics.}
Hamiltonian, Lagrangian, symplectic, and equivariant networks encode structure in learned dynamics
\citep{greydanus2019hamiltonian,cranmer2020lagrangian,jin2020sympnets,cohen2016group,satorras2021egnn}.
Invariant-measure objectives fit long-run statistics
\citep{li2021dissipative,jiang2023invariant,schiff2024dyslim,cheng2025pfnn}.
We instead learn coordinates with an exactly preserved latent reference and bound relaxation to it.
Target-preserving transitions also appear in learned Markov chain Monte Carlo samplers
\citep{song2017anicemc,hoffman2019neutra}; a world model additionally fits the observed conditional
dynamics. Koopman-operator projections
\citep{koopman1931hamiltonian,williams2015edmd} provide finite-basis diagnostics
(Appendix~\ref{app:operator-estimates}); our certificate follows from the architecture directly.

\paragraph{Flows, conditional transport, and mixing.}
Normalizing flows provide invertible density representations
\citep{dinh2015nice,dinh2017realnvp,kingma2018glow,papamakarios2017maf,kingma2016iaf}.
We identify which block-affine couplings preserve a Gaussian reference exactly; density checks,
finite-network capacity, and copula examples appear in Appendices~\ref{app:rigidity},
\ref{app:finite-capacity}, and~\ref{app:chartsub}.
Quantile constructions connect this constraint to measure-preserving geometry and copulas
\citep{arnold1998topological,nelsen2006copulas}, and the action chart uses conditional
Rosenblatt transport \citep{rosenblatt1952}. Gaussian entropy contraction \citep{gentil2010}
applied to the joint state--action innovation yields our finite-start and policy-deviation bounds.

\section{Conclusion and limitations}
\label{sec:limitations}

An explicitly preserved latent reference makes a model's long-horizon marginal available by
static sampling. Joint state--action mixing fits controlled dynamics under this constraint,
while parallel noise supplies finite-start bounds. The experiments distinguish marginal
predictability, conditional accuracy, and planning.

\paragraph{Limitations.}
Stationarity and the action-chart conditions delimit the occupancy interpretation; approximate
charts and policy shifts require error control (\S\ref{sec:certificates}). Finite-start bounds
use bank-specific moments and real arithmetic. Architecture contrasts combine preservation with
achieved fit; the matched intervention isolates output scaling. The measured planning completion
rates and posterior excursions motivate improved inference and finite-candidate planning
(Appendix~\ref{app:limitations}).

\label{wm:main-end}
\clearpage
\subsection*{AI use statement}
Generative AI assisted code development, manuscript writing, mathematical and reference checks, the authors reviewed and verified all AI-assisted work.

\subsection*{Reproducibility statement}
Appendix~\ref{app:experiments} specifies the systems, data splits, training budgets, and evaluation
procedures. The theoretical supplement states assumptions and provides proofs. Detailed experimental
sections report per-seed results, estimator conventions, and diagnostic controls. Experiment and analysis
scripts generate the reported statistics and tables from saved outputs.

\bibliography{refs}
\bibliographystyle{iclr2027_conference}
\clearpage
\appendix
\raggedbottom
\section{Representation, preservation, and relaxation proofs}
\label{app:proofs}

Throughout, $\gauss_d$ is the standard Gaussian law on $\mathbb R^d$;
$\varphi$ and $\Phi$ are its scalar density and cumulative distribution function (CDF). A diffeomorphism is
$C^1$ with a $C^1$ inverse. State, action, and observation spaces are standard
Borel, so the conditional laws used below exist. Total variation uses the
signed-measure $[0,2]$ convention of Section~\ref{sec:certificates}.
Reference preservation is a property of a specified kernel, while accuracy
of its decoded reference is a separate statistical question.

\subsection{Transport of stationary pair laws}
\label{app:transport}

\begin{definition}[Reference-preserving kernel]\label{def:pp}
A state kernel $\tilde K$ preserves the reference if
$\gauss_d\tilde K=\gauss_d$. For a controlled kernel, we distinguish this
property after behaviour-policy averaging from the distinct requirement
that every fixed-action kernel preserve $\gauss_d$.
\end{definition}

\begin{proposition}[Transport of the law]\label{prop:C}
Let $T$ be a bimeasurable bijection with $T_\#\mu = \gauss$, and let $\tilde K$ be the pushforward of
$K$ under $T$. Then $\mu K = \mu \iff \gauss\tilde K = \gauss$.
\end{proposition}

\begin{proof}[Proof of Proposition~\ref{prop:C}]
For measurable $B$,
\[
  (\gauss\tilde K)(B) = \int \tilde K(w,B)\,\gauss(\dd w)
  = \int K(z, T^{-1}B)\,\mu(\dd z) = (\mu K)(T^{-1}B) = T_\#(\mu K)(B),
\]
using $T_\#\mu = \gauss$ and the definition $\tilde K(T z,\cdot) := T_\#K(z,\cdot)$. Since $T$ is a
bijection, $T_\#(\mu K) = \gauss = T_\#\mu$ iff $\mu K = \mu$.
\end{proof}

\begin{proposition}[Representation of stationary pair laws]\label{prop:exact-class}
Let $\mathcal X$ be a standard Borel space and let $\mathcal P_=$ be the probability laws on
$\mathcal X\times\mathcal X$ with equal atomless marginals.
As $E$ ranges over measure-space isomorphisms to the Gaussian reference, modulo null sets, and
$\tilde K$ over Gaussian-preserving kernels, the induced laws of
$(E^{-1}(z),E^{-1}(z'))$ range over precisely $\mathcal P_=$.

For a smooth realization, suppose $\mathcal X=\mathbb R^d$ and the common marginal has strictly positive
conditional densities on the full real coordinate line. Assume each conditional CDF $F_i(x_i\mid x_{<i})$ is $C^k$ jointly in its arguments,
$k\ge1$, has positive derivative in $x_i$, and tends to $0$ and $1$ at the two ends of each conditional
support. Then
$E_i(x)=\Phi^{-1}(F_i(x_i\mid x_{<i}))$ defines a $C^k$ diffeomorphism onto $\mathbb R^d$.
When the pair law has a density, its representation is
\[
p(x,x')=\gauss(Ex)|\det J_E(x)|\,
          \tilde p(Ex'\mid Ex)|\det J_E(x')|.
\]
\end{proposition}

\begin{proof}
A Gaussian-preserving kernel gives a pair law whose two marginals are $\gauss$.
Applying $E^{-1}$ to each coordinate therefore gives equal marginals.
Conversely, an atomless probability measure on a standard Borel space is isomorphic, modulo null sets,
to the Gaussian probability space \citep[Theorem~17.41]{kechris1995}.
Transporting a stationary pair through this isomorphism gives two
Gaussian marginals. Disintegrating that pair yields a Gaussian-preserving kernel, and transporting back
recovers the original pair law.

Under the smooth assumptions, the conditional probability integral transform
\citep{rosenblatt1952} maps the common marginal to independent uniforms, and $\Phi^{-1}$ maps these to
$\gauss_d$. Each coordinate is strictly increasing in $x_i$ and onto $\mathbb R$ for fixed $x_{<i}$.
Coordinates can therefore be inverted successively for every point of $\mathbb R^d$.
The Jacobian is triangular with positive diagonal, and the inverse function theorem gives a $C^k$
inverse locally at every point. Together with global bijectivity, this establishes the claimed
diffeomorphism. Applying the change-of-variables formula to both coordinates gives the density expression.
\end{proof}

\begin{remark}[Transport regularity and finite architectures]\label{rem:exact-class-reg}
The measurable statement characterizes pair laws; the differentiable statement supplies their likelihood
formula. Joint smoothness of conditional CDFs is an explicit hypothesis. Locally uniform integrable bounds
on density derivatives suffice to justify differentiation of the marginal integrals; smoothness of the
joint density alone need not suffice. Global bounds on the Jacobian or its inverse are additional
assumptions only where stated.

These are statements about unrestricted transports and kernels. The frozen
feature experiments and the recurrent pixel model use different finite
representations; neither is asserted to be an exact chart of pixel space.
\end{remark}

\begin{corollary}[Exact representation with a perfect chart]\label{cor:R1}
For a stationary pair law with common marginal $\mu$ and a joint density, if $E$ is a
diffeomorphism with $E_\#\mu=\gauss$ exactly, the class of arbitrary Gaussian-preserving
transition kernels represents that density, with zero marginal and conditional excess loss.
Finite parameterizations retain their approximation restrictions.
\end{corollary}

\begin{proof}[Proof of Corollary~\ref{cor:R1}]
Transporting the true stationary pair through an exact chart gives its true
Gaussian-preserving conditional (Proposition~\ref{prop:exact-class}). The
unrestricted preserving-kernel class therefore attains zero excess marginal
and conditional loss. Restricting either transport or transition to a finite
parameterization can introduce approximation or optimization error.
\end{proof}

\subsection{Rigidity and preserving diffeomorphisms}
\label{app:preserving-diffeomorphisms}

\begin{proof}[Proof of the identity \eqref{eq:dagger}]
Since $T$ is a diffeomorphism and both densities are continuous, $T_\#\gauss = \gauss$ holds if and
only if the change-of-variables identity $\gauss(T(x))\,|\det J_T(x)| = \gauss(x)$ holds at every $x$.
Taking logarithms of both sides and cancelling the common constant $-\tfrac{d}{2}\log2\pi$ leaves
$-\tfrac12|T(x)|^2 + \log|\det J_T(x)| = -\tfrac12|x|^2$, which is exactly \eqref{eq:dagger}.
\end{proof}

\begin{corollary}[Volume and norm preservation]\label{cor:H4}
For a $\gauss$-preserving diffeomorphism, $|\det J_T|\equiv1\iff|T(x)|=|x|$ for all $x$.
\end{corollary}

\begin{proof}[Proof of Corollary~\ref{cor:H4}]
Both directions are immediate from \eqref{eq:dagger}: the left side vanishes iff the right side does.
\end{proof}

\begin{proof}[Proof of the identity \eqref{eq:ddagger}]
The density of $T_\#\gauss$ at $w$ is $\gauss(T^{-1}w)|\det J_{T^{-1}}(w)|$. Substituting $w = T(z)$
with $z\sim\gauss$,
\[
\log\frac{\dd T_\#\gauss}{\dd\gauss}(T(z))
= \log\gauss(z) - \log|\det J_T(z)| - \log\gauss(T(z))
= \tfrac12\big(|T(z)|^2 - |z|^2\big) - \log|\det J_T(z)| ,
\]
and $\mathrm{KL}(T_\#\gauss\Vert\gauss) = \mathbb{E}_{w\sim T_\#\gauss}[\log\frac{\dd T_\#\gauss}{\dd\gauss}(w)]
= \mathbb{E}_{z\sim\gauss}[\cdots]$. Gibbs' inequality gives non-negativity and equality exactly when the two laws
agree. Continuity then gives the pointwise identity \eqref{eq:dagger}.
\end{proof}

\begin{lemma}[Relative entropy of volume-preserving maps]\label{lem:H5}
If $|\det J_T|\equiv1$ then $\mathrm{KL}(T_\#\gauss\Vert\gauss) = \tfrac12(\mathbb{E}_\gauss|T(z)|^2 - d)$.
\end{lemma}

\begin{proof}
Substitute $\log|\det J_T| = 0$ into \eqref{eq:ddagger} and use $\mathbb{E}_\gauss|z|^2 = d$.
\end{proof}

For an exactly volume-preserving diffeomorphism, equality of the
\emph{population} second moment to $d$ makes this KL zero and therefore
implies Gaussian preservation. A near-zero Monte Carlo estimate does not
establish that equality: cancellation and sampling can conceal nonzero
pointwise residuals. Appendix~\ref{app:rigidity} reports both diagnostics.

\begin{lemma}[Scalar rigidity]\label{lem:scalar}
A $\gauss_1$-preserving diffeomorphism of $\mathbb{R}$ is $x\mapsto x$ or $x\mapsto -x$.
\end{lemma}

\begin{proof}
$T$ is a diffeomorphism of $\mathbb{R}$, hence strictly monotone. If increasing, $T_\#\gauss_1=\gauss_1$
forces $\Phi(T(x)) = \Phi(x)$ for all $x$ (equality of CDFs), so $T = \mathrm{id}$. If decreasing,
$\Phi(T(x)) = 1-\Phi(x) = \Phi(-x)$, so $T(x) = -x$.
\end{proof}

\begin{theorem}[Triangular rigidity]\label{thm:H1}
Let $T$ be a $\gauss_d$-preserving diffeomorphism whose Jacobian is triangular for a fixed coordinate
order. Then $T_i(x)=\sigma_ix_i$ with $\sigma_i\in\{\pm1\}$. The conclusion also applies to compositions that retain this common triangular order
(Corollary~\ref{cor:H1p}).
\end{theorem}

\begin{proof}
Order coordinates so that $J_T$ is lower triangular; then $T_i$ depends only on $x_{1:i}$, and
$\partial_iT_i\ne0$ everywhere (else $\det J_T = 0$), so $T_i(x_{1:i-1},\cdot)$ is strictly monotone.
We show $T_i(x) = \sigma_ix_i$ by induction on $i$.

$i=1$: the first coordinate of $T_\#\gauss_d = \gauss_d$ is $\gauss_1$, and $T_1$ is a function of
$x_1$ alone with $x_1\sim\gauss_1$, so $(T_1)_\#\gauss_1 = \gauss_1$ and Lemma~\ref{lem:scalar}
gives $T_1 = \sigma_1x_1$.

Inductive step: assume $T_j = \sigma_jx_j$ for $j<i$. The map $x_{1:i}\mapsto T_{1:i}(x_{1:i})$ pushes
$\gauss_i$ to $\gauss_i$ (a marginal of $T_\#\gauss_d = \gauss_d$). Condition on
$x_{1:i-1} = \xi$; then $T_{1:i-1} = (\sigma_1\xi_1,\dots,\sigma_{i-1}\xi_{i-1})$ is determined, and
under $\gauss_i$ the conditional law of the $i$-th output coordinate given the first $i-1$ outputs is
$\gauss_1$ by independence. For almost every $\xi$, $T_i(\xi,\cdot)$ therefore pushes $\gauss_1$ forward to $\gauss_1$ and is a
monotone diffeomorphism of $\mathbb{R}$, so by Lemma~\ref{lem:scalar}
$T_i(\xi, x_i) = \sigma_i(\xi)x_i$ with $\sigma_i(\xi)\in\{\pm1\}$. The derivative $\partial_iT_i$ is continuous and nonzero on a connected
domain, so its sign is constant. Continuity extends the resulting identity
from almost every $\xi$ to every $\xi$.
\end{proof}

\begin{corollary}\label{cor:H1p}
A composition of triangular diffeomorphisms in a common fixed coordinate order
can preserve $\gauss_d$ only by a coordinate-wise sign flip. Depth and conditioner
expressiveness do not alter this conclusion.
\end{corollary}

\begin{proof}
A composition in one fixed triangular order is again triangular, so
Theorem~\ref{thm:H1} applies to the complete map. Changing the order between
layers removes this hypothesis.
\end{proof}

\begin{theorem}[Block reduction]\label{thm:H2}
Let $T(u,v) = (u, g(u,v))$ be a $C^1$ diffeomorphism with $g(u,\cdot)$ a diffeomorphism of $\mathbb{R}^{n_v}$ for each
$u\in\mathbb{R}^{n_u}$. Then $T$ is $\gauss_{n_u+n_v}$-preserving iff $g(u,\cdot)$ is
$\gauss_{n_v}$-preserving for every $u$. In particular for $n_v = 1$, $g(u,v) = \sigma v$ with
$\sigma\in\{\pm1\}$ constant, so a coupling layer with scalar blocks is trivial.
\end{theorem}

\begin{proof}
The density of $T_\#\gauss$ at $(u,w)$ is
$\gauss_{n_u}(u)\,\gauss_{n_v}(g_u^{-1}(w))\,|\det J_{g_u^{-1}}(w)|$, since the first block is
carried through unchanged. Equating to $\gauss_{n_u}(u)\gauss_{n_v}(w)$ and cancelling the strictly
positive factor $\gauss_{n_u}(u)$ gives, for every $u$,
$\gauss_{n_v}(g_u^{-1}(w))|\det J_{g_u^{-1}}(w)| = \gauss_{n_v}(w)$, which is exactly
$(g_u)_\#\gauss_{n_v} = \gauss_{n_v}$. For $n_v=1$, Lemma~\ref{lem:scalar} gives
$g(u,v) = \sigma(u)v$, and continuity plus connectedness makes $\sigma$ constant.
\end{proof}

\begin{proof}[Proof of Theorem~\ref{thm:H3}]
By Theorem~\ref{thm:H2}, $T$ is $\gauss$-preserving iff for every $u$ the affine map
$v\mapsto A(u)v + b(u)$ pushes $\gauss_{n_v}$ to $\gauss_{n_v}$. The pushforward of $\gauss_{n_v}$
under an affine map is the Gaussian with mean $b(u)$ and covariance $A(u)A(u)^\top$, and a Gaussian
is determined by its mean and covariance, so this holds iff $b(u) = 0$ and $A(u)A(u)^\top = I$.
Requiring it for \emph{every} value of the conditioner makes these identities in $u$. Then
$\log|\det J_T| = \log|\det A(u)| = 0$ since $|\det A| = 1$ for orthogonal $A$. Scalar-block rigidity follows from Theorem~\ref{thm:H2}.
\end{proof}

\begin{proof}[Proof of the quantile-swirl identity]
Let $T = \Phi^{-1}\circ S\circ\Phi$ with $S$ Lebesgue-preserving on $(0,1)^d$, and write
$v' = T(v)$, $\Phi(v) = q$, $S(q) = q'$. By the chain rule,
$J_T = J_{\Phi^{-1}}(q')\,J_S(q)\,J_\Phi(v)$, with $J_\Phi$ and $J_{\Phi^{-1}}$ diagonal. Hence
\[
\log|\det J_T| = \sum_i\log\varphi(v_i) + \log|\det J_S| - \sum_i\log\varphi(v_i')
= \tfrac12\big(|v'|^2 - |v|^2\big),
\]
using $|\det J_S| = 1$ and $\log\varphi(a) - \log\varphi(b) = (b^2-a^2)/2$. This is
\eqref{eq:dagger}, so $T$ is $\gauss_d$-preserving; equivalently, $\Phi$ conjugates the
$\gauss_d$-preserving diffeomorphisms onto the Lebesgue-preserving diffeomorphisms of the cube.
For the swirl itself: in polar coordinates $(\varrho,\alpha)$ about the learned center, the map is
$(\varrho,\alpha)\mapsto(\varrho, \alpha + w(\varrho)\theta)$, whose Jacobian is unit
lower-triangular, so $|\det| = 1$; conjugating by a diagonal scaling $D$ multiplies the determinant by
$|\det D|\cdot|\det D|^{-1} = 1$, turning the disc into an axis-aligned ellipse at no cost.
\end{proof}

\subsection{Ideal approximation and augmented representation}
\label{app:ideal-representation}

\begin{theorem}[Approximation by quantile swirls]\label{thm:G}
Let $d\ge2$ and let $T$ be a smooth Gaussian-preserving diffeomorphism joined
to the identity by a smooth Gaussian-preserving isotopy $T_t$, $0\le t\le1$,
whose support lies in one compact set. Allow pairwise quantile swirls with
arbitrary interior centers, radii, smooth compactly supported radial profiles,
and smooth coefficients depending on the untouched coordinates. Finite
compositions of these swirls approximate $T$ in $C^1$ on compact sets.
This is an approximation result for the stated ideal layer family, with no
fixed bound on its depth or on the number of profiles.
\end{theorem}

\begin{proof}
Conjugate by the coordinatewise Gaussian CDF. The isotopy becomes a smooth
volume-preserving isotopy $S_t$ of $(0,1)^d$ with common compact support.
Its velocity $X_t=\dot S_t\circ S_t^{-1}$ is smooth, compactly supported,
and divergence free, by differentiating the volume-preservation identity.
We give three steps for approximating its time-one flow.

\emph{Pairwise decomposition.}
Every smooth divergence-free field $X$ supported in an interior rectangle has
a compactly supported antisymmetric potential $A$ satisfying
$X_i=\sum_j\partial_jA_{ij}$. An elementary induction establishes this
without imposing boundary conditions on a noncompact potential. The
one-dimensional base case is zero: a compactly supported divergence-free
scalar field is constant and hence vanishes.
For $d\ge2$, write $x=(x_1,u)$ and define
\[
 Y_j(u)=\int_{\mathbb R}X_j(t,u)\,\dd t\quad(j>1).
\]
Then $\sum_{j>1}\partial_jY_j=0$. Choose a smooth bump $\chi$ in the
first-coordinate interval with integral one, and set
$Z_1=X_1$, $Z_j=X_j-\chi(x_1)Y_j(u)$ for $j>1$.
The field $Z$ is divergence free and its components $Z_j$, $j>1$, integrate
to zero along $x_1$. Consequently
\[
 A_{j1}(x)=\int_{-\infty}^{x_1}Z_j(t,u)\,\dd t,\qquad
 A_{1j}=-A_{j1}
\]
have compact support in the rectangle, with
$\partial_1A_{j1}=Z_j$ and
$\sum_{j>1}\partial_jA_{1j}=Z_1$.
By induction, the field $Y$ has an antisymmetric potential $B$ in the
remaining coordinates. Adding $A_{jk}=\chi(x_1)B_{jk}(u)$ for $j,k>1$
represents the subtracted field and proves the claim.
Thus $X$ is a finite sum of pairwise Hamiltonian fields
$X^{(ij)}=\partial_jA_{ij}\,e_i-\partial_iA_{ij}\,e_j$, with untouched
coordinates acting as parameters.

\emph{Approximation by radial Hamiltonians.}
Fix a pair $v=(x_i,x_j)$ and write its Hamiltonian as $H(v,u)$, with $u$ the
untouched coordinates.
Convolve only in $v$ with a smooth radial mollifier $\varphi_\epsilon$,
choosing $\epsilon$ small enough to keep its support inside the pair's
square. The convolution converges to $H$ in the joint $C^2(v,u)$ norm:
all derivatives of total order at most two commute with convolution and
converge uniformly on the common compact support. For fixed $\epsilon$,
a Riemann sum gives
\[
 H_\eta(v,u)=\sum_{m=1}^M |Q_m|\,H(v_m,u)\,
                        \varphi_\epsilon(v-v_m).
\]
The integrand and its derivatives of total order at most two are uniformly
continuous, so this finite sum converges to the convolution in $C^2$ as
the mesh tends to zero. Each term is a radial bump in the updated pair,
with a smooth compactly supported coefficient in $u$.
Its Hamiltonian flow preserves radius and rotates the pair by a
radius-dependent angle; $u$ remains fixed. It is therefore an admissible
pairwise swirl. Smooth radial bumps give a smooth angular profile at the
center, since their radial derivative divided by radius extends smoothly.
The associated finite sum of swirl vector fields approximates the pairwise
field in $C^1$.

\emph{Approximation of the flow.}
First freeze $X_t$ on sufficiently short time intervals. Smoothness and
common compact support give convergence of the resulting piecewise
constant flows in $C^1$, by the flow and variational equations and
Gr\"onwall's inequality. On each interval, approximate its frozen field
by the finite sum above. Continuous dependence of those same equations
gives $C^1$ convergence of the flows as the fields converge in $C^1$.
Finally, Lie--Trotter splitting approximates the flow of each finite sum
by repeated compositions of the individual swirl flows: for these smooth
compactly supported fields the one-step $C^1$ error is $O(h^2)$ and the
accumulated error is $O(h)$ over a fixed interval. Each approximation is
a finite composition. Choosing the three errors successively small yields
the required $C^1$ approximation of $S_1$. Conjugating back preserves
$C^1$ convergence on the corresponding compact sets, proving the claim.
\end{proof}

Theorem~\ref{thm:G} concerns the ideal family with a compactly supported
isotopy, arbitrary smooth profiles, and unrestricted composition length.
Density of a fixed profile set or neural conditioner family, parameter counts,
and approximation rates fall outside its scope. Each implemented layer's
Gaussian preservation follows directly from the change-of-variables identity.

\begin{theorem}[Measurable augmented representation]\label{thm:E}
Let $\tilde K$ be a Gaussian-preserving Markov kernel on $\mathbb R^d$
with positive conditional densities. There is a Gaussian-preserving Borel
isomorphism $\Psi$ between full-measure subsets of $\mathbb R^{2d}$ such that,
for $\gauss_d$-almost every $z$,
$\Pi_{1:d}\Psi(z,\varepsilon)\sim\tilde K(z,\cdot)$ when
$\varepsilon\sim\gauss_d$. Thus $d$ auxiliary Gaussian coordinates suffice
for this representation modulo null sets.
\end{theorem}

\begin{theorem}[Smooth augmented representation]\label{thm:Esmooth}
In addition, suppose all scalar conditional CDFs in the forward
Knothe--Rosenblatt recursion for $\tilde K(z,\cdot)$ and in the reverse
recursion for $\Law(z\mid y)$ are $C^1$ jointly in the scalar argument and
all conditioning variables. Assume each has positive derivative in the
scalar argument and limits zero and one at $-\infty$ and $+\infty$, for
every conditioning value. Then the construction \eqref{eq:thmE} defines
a $C^1$ diffeomorphism of $\mathbb R^{2d}$ preserving $\gauss_{2d}$.
\end{theorem}

\begin{proof}[Proof of Theorem~\ref{thm:E}]
Conditional probability integral transforms and their triangular inverses
provide measurable transports $R_z$ with
$(R_z)_\#\gauss_d=\tilde K(z,\cdot)$, invertible modulo conditional null
sets. Positive densities ensure the scalar conditional laws used by the
recursion are atomless with strictly increasing CDFs on their supports.
These constructions can be chosen jointly measurable, using generalized
inverses of the measurable conditional CDFs.

Set $y=R_z(\varepsilon)$. Under independent Gaussian $z,\varepsilon$,
the pair $(z,y)$ has joint law $\gauss(\dd z)\tilde K(z,\dd y)$ and both
marginals are Gaussian. Its reverse conditional also has a density and is
atomless for almost every $y$. Let $S_y$ be its conditional triangular
transport to $\gauss_d$. Then $\eta=S_y(z)$ is Gaussian independently of
$y$. Define
\begin{equation}
 \Psi(z,\varepsilon)=\bigl(R_z(\varepsilon),S_{R_z(\varepsilon)}(z)\bigr).
 \label{eq:thmE}
\end{equation}
Its output has law $\gauss_{2d}$. The inverse recovers
$z=S_y^{-1}(\eta)$ and then $\varepsilon=R_z^{-1}(y)$, with all identities
understood on the full-measure sets where the conditional transforms are
invertible. The first output retains the desired conditional law, because
the second transport changes only the auxiliary output coordinate.
\end{proof}

\begin{proof}[Proof of Theorem~\ref{thm:Esmooth}]
For each scalar equation $F(y_j\mid z,y_{<j})=\Phi(\varepsilon_j)$,
the assumed joint $C^1$ regularity and positive derivative in $y_j$ give a
jointly $C^1$ solution by the implicit function theorem. The tail limits
make it unique and defined for every finite input. Applying this recursion
successively gives a jointly $C^1$ map $(z,\varepsilon)\mapsto R_z(\varepsilon)$
with a jointly $C^1$ inverse in its second variable. The same argument
applies to $(y,z)\mapsto S_y(z)$ and its inverse.

The construction is the composition of two global triangular changes of
variables,
$(z,\varepsilon)\mapsto(z,R_z(\varepsilon))$ and
$(z,y)\mapsto(y,S_y(z))$.
Each is a $C^1$ diffeomorphism, so their composition is one as well.
The measurable argument establishes Gaussian preservation. No global bound
on the Jacobian or its inverse is required.
\end{proof}

\begin{remark}[Regularity and representation scope]\label{rem:thmE-split}
The conditional-CDF hypotheses ensure joint smoothness of the forward and
reverse transports. Locally uniform integrable envelopes for the density
derivatives provide sufficient conditions, as in
Remark~\ref{rem:exact-class-reg}. The measurable theorem describes attainable
laws, and the smooth theorem supplies differentiable representations. Both
concern unrestricted transports; capacity and optimization of finite
architectures are separate questions.
\end{remark}

\begin{corollary}[Controlled conditional representation]\label{cor:controlled-representation}
Let $u=(z,\alpha)\in\mathbb R^{d+m}$ and let $K(u,\dd y)$ have
positive conditional densities and satisfy
$\int K(u,\cdot)\gauss_{d+m}(\dd u)=\gauss_d$.
There is a Gaussian-preserving Borel isomorphism of $\mathbb R^{2d+m}$,
modulo null sets, whose first $d$ outputs, on input $(z,\alpha,\varepsilon)$
with fresh $\varepsilon\sim\gauss_d$, have conditional law
$K(z,\alpha,\cdot)$ for almost every $(z,\alpha)$.
Under the corresponding forward and reverse conditional-CDF hypotheses of
Theorem~\ref{thm:Esmooth}, the construction is a $C^1$ diffeomorphism.
\end{corollary}

\begin{proof}
The forward conditional transport gives $y=R_u(\varepsilon)$.
The pair $(u,y)$ has a positive joint density and Gaussian $y$ marginal.
Gaussianize its reverse conditional by $\eta=S_y(u)\in\mathbb R^{d+m}$.
Then $(y,\eta)$ is product Gaussian and the inverse recovers
$u=S_y^{-1}(\eta)$, followed by $\varepsilon=R_u^{-1}(y)$.
This is the proof of Theorem~\ref{thm:E} with unequal input and output
state dimensions. The same scalar implicit-function argument proves the
smooth statement. In particular it retains the entire controlled
conditional, not only its action-averaged state kernel.
\end{proof}

\subsection{The controlled conservation law}
\label{app:controlled}

Stationarity constrains the behaviour-averaged kernel. Preservation at each
fixed action is a stronger requirement under independent actions, and need
not hold for the true controlled dynamics.

\begin{proposition}[Stationarity does not imply per-action invariance]\label{prop:noperaction}
There is a controlled Markov chain whose data are exactly stationary under its behaviour policy and
for which none of its fixed-action transition kernels $K_a$ preserves the state marginal.
\end{proposition}

\begin{proof}
Let $z\in\{0,1\}$, let $a\in\{0,1\}$ be drawn uniformly and independently of the state, and let the
environment be $z_{t+1}=a_t$. If $z_0\sim\mu:=\mathrm{Unif}\{0,1\}$ then $z_t\sim\mu$ for every $t$,
so the data are stationary and \eqref{eq:occupancy-main} holds: $\mu\Kbar^\star=\mu$ with
$\Kbar^\star=\tfrac12(K_0+K_1)$. But $K_0$ maps every state to $0$ and $K_1$ maps every state to $1$,
so $\mu K_0=\delta_0\ne\mu$ and $\mu K_1=\delta_1\ne\mu$.
\end{proof}

\begin{fact}[Controlled stationarity is occupancy invariance]\label{fact:occ}
If the data are stationary under a fixed behaviour policy $\pi_b(\dd a\mid z)$, then
\begin{equation}
\iint \mu(\dd z)\,\pi_b(\dd a\mid z)\,K^\star(z,a,\dd z')\;=\;\mu(\dd z') .
\label{eq:occ}
\end{equation}
\end{fact}

The identity is the one-step marginal law under the stationary state--action occupancy.

\begin{definition}[Action chart]\label{def:chart}
An \emph{action chart} is a family of measurable maps $G(\cdot\mid z):\mathcal{A}\to\mathbb{R}^m$ with
\begin{equation}
G(\cdot\mid z)_\#\pi_b(\cdot\mid z)\;=\;\gauss_m \qquad\text{for every } z ,
\label{eq:chart}
\end{equation}
The family is assumed jointly measurable in $(z,a)$, and we write
$\alpha:=G(a\mid z)$. A transition is chart-driven if it reads the action
only through $\alpha$. For a differentiable action likelihood, assume the
conditional-CDF regularity of Proposition~\ref{prop:exact-class}, so each
$G(\cdot\mid z)$ is a diffeomorphism. A positive joint density alone does
not supply the required joint regularity of the conditional transport.
\end{definition}

\begin{remark}[Randomized charts for atomic actions]\label{rem:chart-discrete}
A deterministic measurable map sends an atom to an atom, so it cannot map a
discrete action law to an atomless Gaussian. For ordered discrete actions,
augment the action with an independent $u\sim\mathrm{Unif}(0,1)$ and set
\[
 \alpha=\Phi^{-1}\!\left(F(a^-\mid z)+u\,\pi_b(a\mid z)\right).
\]
This gives a scalar Gaussian chart exactly under the behaviour law. It is a
map of $(a,u)$, extending Definition~\ref{def:chart}'s deterministic-action
form. The invariance induction still applies because it uses only the
conditional output law. Discrete-action likelihoods require probability
masses and integration over the auxiliary variable, rather than the
continuous chart's pointwise Jacobian formula. The experiments do not
evaluate this extension.
\end{remark}

\begin{lemma}[Conditional Gaussianization]\label{lem:chart}
Under \eqref{eq:chart}, if $z\sim\mu$ and $a\sim\pi_b(\cdot\mid z)$,
then $\alpha\sim\gauss_m$ independently of $z$. In particular, if
$\mu=\gauss_d$, then $(z,\alpha)\sim\gauss_d\otimes\gauss_m$.
More generally, an exact state transport $E_\#\mu=\gauss_d$ gives
$(E(z),\alpha)\sim\gauss_d\otimes\gauss_m$.
\end{lemma}

\begin{proof}
The conditional law of $\alpha$ given $z$ is the same Gaussian at every
state, which establishes both its marginal and its independence from $z$.
Applying a deterministic state transport preserves that independence.
\end{proof}

\begin{lemma}[Observation-conditioned action charts]\label{lem:observation-chart}
Let $\lambda(\dd z,\dd o)=\mu(\dd z)L(\dd o\mid z)$ be any joint
state--context law. Suppose the jointly measurable chart $G_{z,o}$ obeys
$(G_{z,o})_\#\pi_b(\cdot\mid z,o)=\gauss_m$ for every relevant $(z,o)$.
Under this behaviour policy, $\alpha=G_{Z,O}(A)$ is independent of $(Z,O)$
and has law $\gauss_m$. Consequently, if a transition reads $o$ only through
$\alpha$, its behaviour-averaged state kernel is
\[
 K(z,B)=\int M(z,\alpha,B)\,\gauss_m(\dd\alpha),
\]
independently of $L$. The context may be a stochastic decoder output; it
need not itself have a Gaussian distribution.

If each $G_{z,o}$ is a bimeasurable bijection modulo common null sets, a
deployed policy $\pi_t$ instead gives a joint input law
$\nu_t=\Law(Z,\alpha)$ satisfying
\begin{equation}
 \KL(\nu_t\|\gauss_d\otimes\gauss_m)
 \le \KL(\mu\|\gauss_d)+\kappa_t,\qquad
 \kappa_t:=\int\KL(\pi_t(\cdot\mid z,o)\|\pi_b(\cdot\mid z,o))
                       \,\lambda(\dd z,\dd o).
 \label{eq:context-joint-kl}
\end{equation}
The divergence may be infinite. Thus Theorem~\ref{thm:parallel-kl} implies
$H_{t+1}\le c_\star^2(H_t+\kappa_t)$ for this observation-conditioned
setting, with $H_t=\KL(\mu_t\|\gauss_d)$.
\end{lemma}

\begin{proof}
For measurable $C\subseteq\mathcal Z\times\mathcal O$ and
$B\subseteq\mathbb R^m$, the behaviour joint probability is
\[
 \Pr((Z,O)\in C,\alpha\in B)
 =\int_C (G_{z,o})_\#\pi_b(B\mid z,o)\,\lambda(\dd z,\dd o)
 =\lambda(C)\gauss_m(B).
\]
This proves independence and the averaged-kernel statement.
For the policy deviation, let
$q_{z,o}=(G_{z,o})_\#\pi_t(\cdot\mid z,o)$ and
$q_z=\int q_{z,o}L(\dd o\mid z)$. Divergence invariance under the chart,
convexity, and the chain rule give
\[
 \int\KL(q_z\|\gauss_m)\,\mu(\dd z)
 \le\int\KL(q_{z,o}\|\gauss_m)\,\lambda(\dd z,\dd o)=\kappa_t,
\]
followed by \eqref{eq:context-joint-kl}. Averaging out the context may make
this inequality strict; equality holds when the context is a deterministic
function of the state. Applying the joint-input entropy theorem proves the
last claim.
\end{proof}

For the known tanh-Gaussian policy,
$G_{z,o}(a)=(\operatorname{atanh}a-m_b(o))\oslash s_b(o)$, with $s_b(o)>0$,
is exact on the open action cube. When the deployed policy is also a nonsingular
tanh-Gaussian with the same tanh transformation, their conditional KL equals
the Gaussian KL before tanh. Additional direct observation-dependent
transition parameters require a separate preservation argument: the lemma
covers context entering through the chart alone. Conditional likelihoods
then condition on $(z,a,o)$, unless $o$ is determined by $z$.

\begin{theorem}[Zero marginal drift, controlled]\label{thm:Achart}
Let $\Psi:\mathbb{R}^{d+m+k}\to\mathbb{R}^{d+m+k}$ satisfy $\Psi_\#\gauss_{d+m+k}=\gauss_{d+m+k}$, let
$\varepsilon_t\sim\gauss_k$ be fresh at each step, and set
\begin{equation}
z_{t+1}\;=\;\Pi_{1:d}\,\Psi(z_t,\alpha_t,\varepsilon_t),\qquad \alpha_t=G(a_t\mid z_t).
\label{eq:kernelchart}
\end{equation}
If $z_0\sim\gauss_d$ and the deployed policy's chart image is $\gauss_m$ at every state --- in
particular under the behaviour policy itself, by Lemma~\ref{lem:chart}, and under
observation-conditioned policies by Lemma~\ref{lem:observation-chart} --- then $\Law(z_T)=\gauss_d$
for every $T\ge0$, at every parameter value of $\Psi$, including at initialization.
\end{theorem}

\begin{proof}
Suppose $z_t\sim\gauss_d$. The deployed policy's
chart image is $\gauss_m$ at every state, so $\alpha_t\sim\gauss_m$ independently of $z_t$
(Lemma~\ref{lem:chart}), and $\varepsilon_t\sim\gauss_k$ is fresh, so
$(z_t,\alpha_t,\varepsilon_t)\sim\gauss_{d+m+k}$. Hence $\Psi(z_t,\alpha_t,\varepsilon_t)\sim
\gauss_{d+m+k}$, whose first $d$ coordinates are $\gauss_d$. Nothing beyond measure preservation of
$\Psi$ was used.
\end{proof}

\begin{proposition}[Strict inclusion under a state-independent chart]\label{prop:nesting}
Suppose the behaviour actions are independent of the Gaussian reference
state and have a state-independent invertible chart $\alpha=G(a)$.
Within the measurable augmented representation, positive-density kernels preserving
$\gauss_d$ for each action form a subclass of kernels preserving the
behaviour-averaged Gaussian law. The inclusion is strict, even for smooth
positive-density transitions realized by orthogonal augmented maps.
\end{proposition}

\begin{proof}
Every per-action-preserving kernel preserves the state law after averaging
over independent actions. For positive-density kernels, the construction
of Theorem~\ref{thm:E} can be applied conditionally on $\alpha$, with
measurable dependence on this parameter, and the augmented map leaves the
$\alpha$ block fixed. Thus this subclass has a joint preserving representation.

For strictness, take $d=m=k=1$, $\alpha=a\sim\gauss_1$, and
$z'=\varrho a+\sqrt{1-\varrho^2}\,\xi$ with $0<|\varrho|<1$.
The map
\[
 (z,\alpha,\xi)\longmapsto
 \bigl(\varrho\alpha+\sqrt{1-\varrho^2}\,\xi,\ z,
       -\sqrt{1-\varrho^2}\,\alpha+\varrho\xi\bigr)
\]
is orthogonal and has this first coordinate. Its state output is Gaussian
after averaging over actions, while its law at fixed $a$ is
$\mathcal N(\varrho a,1-\varrho^2)\ne\gauss_1$.
\end{proof}

\begin{remark}[Physical actions and state-dependent charts]\label{rem:nesting-state}
If $G$ depends on $z$, preservation at fixed $\alpha$ need not imply
preservation at fixed physical action $a$, because $G(a\mid z)$ then varies
across input states. Conversely, per-action preservation alone need not
preserve occupancy under a state-dependent policy, as
Remark~\ref{rem:closedloop-ce} shows. The two classes therefore need not be
nested for such a policy. The occupancy construction guarantees the
behaviour reference under its conditional-chart hypothesis; it does not
assert that every physically per-action-preserving kernel satisfies that
occupancy constraint.
\end{remark}

\begin{proposition}[Per-action likelihood floor for independent actions]\label{prop:controlled-floor}
Let $z_{t+1}=\varrho a_t+\sqrt{1-\varrho^2}\,\xi_t$ with $|\varrho|<1$,
independent standard Gaussian actions and innovations, and $z_0\sim\gauss_1$.
Among kernels satisfying $\gauss_1\hat K_a=\gauss_1$ for every $a$, the
minimum expected excess conditional negative log-likelihood (NLL) is
$-\tfrac12\log(1-\varrho^2)$ nats. The occupancy-preserving augmented
class contains the truth and has zero excess loss.
\end{proposition}

\begin{proof}
Write $p_a(y)$ for the true Gaussian conditional and $q_a(y\mid z)$ for
a candidate density. Under the truth, $z$ and $a$ are independent and
$p_a$ does not depend on $z$. Per-action preservation gives
$\int q_a(y\mid z)\gauss(\dd z)=\varphi(y)$ for almost every $y$.
Convexity of $-\log$ therefore yields
\[
 \Exp_z[-\log q_a(y\mid z)]\ge-\log\varphi(y).
\]
Integrating over $a$ and $y\sim p_a$ gives the lower bound
$H(\gauss_1)=\tfrac12\log(2\pi e)$ on expected NLL, attained by the
independent-refresh kernel $q_a(y\mid z)=\varphi(y)$.
Subtracting the truth's entropy
$\tfrac12\log(2\pi e(1-\varrho^2))$ proves the floor. The orthogonal map
in Proposition~\ref{prop:nesting} realizes the truth in the augmented class.
\end{proof}

\begin{theorem}[Policy deviation in conditional total variation]\label{thm:Aprime}
Let $M(z,\alpha,\cdot)$ be a state kernel whose Gaussian-action average
$K$ preserves $\gauss_d$, and write
$q_t(\cdot\mid z)=\Law(\alpha_t\mid Z_t=z)$ under deployment. Define
\begin{equation}
 \eta_t=\int\TV{q_t(\cdot\mid z)-\gauss_m}\,\mu_t(\dd z),
 \qquad d_t=\TV{\mu_t-\gauss_d}.
 \label{eq:alphacond}
\end{equation}
If $\rho=\frac12\sup_{z,w}\TV{K(z,\cdot)-K(w,\cdot)}$ is its
Dobrushin coefficient, then
\begin{equation}
 d_{t+1}\le\rho d_t+\eta_t,\qquad
 d_T\le\rho^T d_0+\sum_{t<T}\rho^{T-1-t}\eta_t.
 \label{eq:apr}
\end{equation}
In particular $d_T\le d_0+\sum_{t<T}\eta_t$, while $\rho<1$ and
$\eta_t\le\bar\eta$ imply $\limsup_T d_T\le\bar\eta/(1-\rho)$.
For the observation-conditioned chart of Lemma~\ref{lem:observation-chart},
$\eta_t\le\min\{2,\sqrt{2\kappa_t}\}$.
\end{theorem}

\begin{proof}
Data processing through $M(z,\cdot,\cdot)$ at each fixed $z$ bounds
$\TV{\mu_{t+1}-\mu_tK}$ by $\eta_t$. Reference invariance and the
Dobrushin contraction give
$\TV{\mu_tK-\gauss_d}\le\rho d_t$.
The triangle inequality and induction yield \eqref{eq:apr}.
Pinsker, Jensen, and the conditional KL bound in the observation lemma give
$\eta_t\le\sqrt{2\kappa_t}$.
\end{proof}

\begin{remark}[Invariance and uniform contraction]\label{rem:apr-dobrushin}
A preserving deterministic bijection has Dobrushin coefficient one.
So does a nontrivial scalar Ornstein--Uhlenbeck (OU) kernel on the unbounded line: for
$0<|c|<1$, its transition laws have distance
$4\Phi(|c|\,|z-w|/(2\sqrt{1-c^2}))-2\to2$ as $|z-w|\to\infty$.
Thus a noise floor does not establish $\rho<1$.
\end{remark}

\begin{remark}[Matching action marginals is insufficient]\label{rem:closedloop-ce}
Let the state and behaviour action both be independent uniform bits, with
$K_0$ the identity and $K_1$ the flip. Both kernels preserve uniform states.
Deploying $a=z$ from that reference still gives uniform actions, but sends
both states to zero. The action marginal is unchanged while the state's
TV deviation is one. Its average conditional action deviation is also one.
\end{remark}

\begin{remark}[Fitted-chart error]\label{rem:chart-fit}
Suppose
$\TV{(G_{z,o})_\#\pi_b(\cdot\mid z,o)-\gauss_m}\le\varepsilon_G$
uniformly over $(z,o)$. This gives $d_T\le\min\{2,T\varepsilon_G\}$ from the
reference by Theorem~\ref{thm:Aprime}. For a fitted invertible chart,
KL comparisons instead use the chart-induced policy
$(G_{z,o}^{-1})_\#\gauss_m$.
A TV residual or a sampled correlation alone does not bound that KL.
\end{remark}

\subsection{Static marginals and the signed transient}
\label{app:static-transient}

\begin{proof}[Proof of Theorem~\ref{thm:A}]
Lemma~\ref{lem:observation-chart} includes the generated observation context
and gives $(z_t,\alpha_t)\sim\gauss_{d+m}$.
The preserving joint map therefore yields independent Gaussian blocks $(h,b)$.
Conditional on $b$, every refreshed coordinate has variance
$c_i(b)^2+1-c_i(b)^2=1$, with independent coordinates. The preserving
post-map $S_b$ retains $\gauss_d$ for each $b$; integrating over $b$ and
inducting over time proves the state claim. The Gaussian reference component
has the same output law, so its fixed-probability mixture also preserves it.
Applying the fixed observation map or observation kernel gives the decoded
statement. The uncontrolled case sets $m=0$.
\end{proof}

\begin{proof}[Proof of Proposition~\ref{thm:R}]
Reference preservation gives
$\mu_0K^T-\gauss_d=(\mu_0-\gauss_d)K^T$.
Every Markov kernel contracts total variation, as does the observation map
or a fixed observation kernel. The reverse triangle inequality therefore
implies
\[
 \begin{aligned}
 |D_T-D_\infty|
 &\le D(A_T,p_\theta)
 \le C_D\TV{A_T-p_\theta}\\
 &\le C_D\TV{\mu_0K^T-\gauss_d}
 \le C_D\TV{\mu_0-\gauss_d}.
 \end{aligned}
\]
Define $r_T=D_T-D_\infty$. This proves the identity and both bounds for
any initial law $\mu_0$. Total-variation convergence makes the middle
bound tend to zero; reference invariance alone does not imply convergence.
\end{proof}

\begin{remark}[Energy-distance convergence]\label{rem:Dsq}
For Euclidean observations with finite first moments, write
$\mathcal E(P,Q)=2\mathbb E\|X-Y\|-\mathbb E\|X-X'\|
-\mathbb E\|Y-Y'\|$, with independent draws. Its square root is a metric;
the squared statistic need not satisfy the metric triangle inequality,
and its unbiased sample estimate can be negative. On a space of diameter
$\Delta$, the direct signed-measure argument in
\eqref{eq:parallel-energy} gives
$|\mathcal E(P,R)-\mathcal E(Q,R)|\le2\Delta\TV{P-Q}$.
Thus the transient bound holds with $C=2\Delta$ for this statistic.
Unbounded observation distances require additional moment control.
\end{remark}

\begin{corollary}[Marginal weighting at the optimum]\label{cor:R2}
Let $M(\theta)$ and $C(\theta)$ be fixed marginal and conditional cross-entropy
objectives, and suppose $\theta_w$ is a global minimizer of
$J_w=wM+C$. For $w'>w\ge1$, $M(\theta_{w'})\le M(\theta_w)$,
$C(\theta_{w'})\ge C(\theta_w)$, and
$J_1(\theta_{w'})\ge J_1(\theta_w)$.
Adding $N_{\mathrm{unpaired}}$ single-state observations to
$N_{\mathrm{paired}}$ pairs motivates the population composite-likelihood weight
$w=1+N_{\mathrm{unpaired}}/N_{\mathrm{paired}}$.
\end{corollary}

\begin{proof}
Optimality at both weights gives
\[
 w'\Delta M+\Delta C\le0,\qquad
 w\Delta M+\Delta C\ge0,
\]
where $\Delta M=M(\theta_{w'})-M(\theta_w)$ and similarly for $C$.
Subtracting gives $(w'-w)\Delta M\le0$, hence $\Delta M\le0$.
The second inequality implies $\Delta C\ge-w\Delta M\ge0$ and
$\Delta J_1\ge-(w-1)\Delta M\ge0$.
For the weight, divide the population composite negative log-likelihood
$N_{\mathrm{paired}}(M+C)+N_{\mathrm{unpaired}}M$ by the number of pairs.
\end{proof}

The ordering applies to population optima of these cross-entropy objectives;
held-out energy distances and finite optimization results have no such ordering.

\subsection{One-step accuracy and path-law divergence}
\label{app:path-laws}

The next identities concern state--action Markov path laws with a shared
behaviour rule. Their state variable contains the information needed for the
Markov property; a pixel observation alone need not suffice.

\begin{proposition}[AR(1) sensitivity near the stability boundary]\label{prop:knife}
Let the truth be $z'=\lambda z+\sqrt{1-\lambda^2}\,\xi$ with $|\lambda|<1$, $\xi\sim\gauss_1$, and let a free
model $z'=mz+s\xi$ have $s^2=1-\lambda^2$ and $m=\lambda+\delta$.
Assume fresh independent innovations and a centered initial law of finite variance. Then
the stationary variance is $v_\infty=(1-\lambda^2)/(1-m^2)$ for $|m|<1$, so
$v_\infty-1=2\lambda\delta/(1-\lambda^2)+O(\delta^2)$, while $v_T\to\infty$ for $|m|\ge1$: the transition is
at $\delta=1-\lambda$ on the positive-memory boundary (the other boundary is $\delta=-1-\lambda$). The reverse one-step KL at marginal variance $v_t$ is
$\delta^2v_t/(2(1-\lambda^2))$.
\end{proposition}

\begin{proof}
The model's marginal variance obeys $v_{t+1} = m^2 v_t + s^2$ with $s^2 = 1-\lambda^2$. If $|m|<1$
the unique fixed point is $v_\infty = s^2/(1-m^2)$; if $|m|\ge1$ then $v_t \to\infty$ since
$v_{t+1}-v_t = (m^2-1)v_t + s^2 \ge s^2 > 0$. Substituting $m = \lambda+\delta$,
\[
v_\infty - 1 = \frac{(1-\lambda^2) - (1-m^2)}{1-m^2} = \frac{2\lambda\delta+\delta^2}{1-\lambda^2-2\lambda\delta-\delta^2}
= \frac{2\lambda\delta}{1-\lambda^2} + O(\delta^2),
\]
the expansion being local to $\delta=0$, with $|2\lambda\delta+\delta^2|<1-\lambda^2$. The denominator vanishes at
$2\lambda\delta+\delta^2 = 1-\lambda^2$, \ie at
$\delta^\star = -\lambda + \sqrt{\lambda^2+1-\lambda^2} = 1-\lambda = (1-\lambda^2)/(1+\lambda)$. For the KL, both kernels are Gaussian with the
same variance $s^2 = 1-\lambda^2$ and means $\lambda z$ and $mz$, so
$\mathrm{KL}(\mathcal{N}(mz,s^2)\Vert\mathcal{N}(\lambda z,s^2)) = (mz-\lambda z)^2/(2s^2)
= \delta^2 z^2/(2(1-\lambda^2))$; taking expectation under the centered model marginal, for which
$\mathbb E z^2=v_t$, gives the claim.
Finally, under $m=\cos\phi$, $s=\sin\phi$ we have $1-m^2 = s^2$ identically, so
$v_\infty = s^2/(1-m^2) = 1$ for every $\phi$ with $\sin\phi\ne0$, and $v_{t+1} = \cos^2\!\phi\,v_t +
\sin^2\!\phi$ has $v_t\equiv1$ whenever $v_0=1$.
\end{proof}

\begin{lemma}[Horizon decomposition]\label{lem:decomp}
Let $P^\theta_{1:T}$ and $P^\star_{1:T}$ be time-homogeneous Markov
state--action path laws with the same conditional policy $\pi_b(\dd a\mid z)$,
containing $T$ states and $T-1$ actions. Let the true initial law be $\mu$. Then
\begin{equation}
\KL(P^\theta_{1:T}\|P^\star_{1:T})=\KL(\mu_1^\theta\|\mu)
+\sum_{t=1}^{T-1}\Exp_{z\sim\mu_t^\theta}\Exp_{a\sim\pi_b(\cdot\mid z)}\ell_\theta(z,a),
\label{eq:decomp}
\end{equation}
where $\ell_\theta(z,a)=\KL(K_\theta(z,a,\cdot)\|K^\star(z,a,\cdot))$.
Relative entropies may take the value $+\infty$. When the terms are finite,
time dependence of the incremental reverse KL enters through $\mu_t^\theta$.
\end{lemma}

\begin{proof}
The joint path law factors as the initial state law times
$\prod_{t=1}^{T-1}\pi_b(\dd a_t\mid z_t)K(z_t,a_t,\dd z_{t+1})$.
The chain rule for relative entropy gives the initial KL plus the expected
sum of conditional transition KLs; the common action factors contribute zero.
The policy factors agree at each history, so each expectation is
$\Exp_{z\sim\mu_t^\theta}\Exp_{a\sim\pi_b(\cdot\mid z)}\ell_\theta(z,a)$.
If actions are marginalized out of the path law, data processing instead
bounds that state-only KL above by the joint-path expression; equality need
not hold.
\end{proof}

\begin{theorem}[Constant reverse relative-entropy rate]\label{thm:B}
Under Lemma~\ref{lem:decomp}, suppose $\mu_t^\theta=\mu$ for all $t$ and
$\ell_0^{\mathrm{rev}}:=\Exp_{z\sim\mu}\Exp_{a\sim\pi_b(\cdot\mid z)}\ell_\theta(z,a)<\infty$.
Then
$\KL(P^\theta_{1:T}\|P^\star_{1:T})=(T-1)\ell_0^{\mathrm{rev}}$.
The rate may be zero; otherwise the divergence grows linearly with horizon.
\end{theorem}

\begin{remark}[Forward sequence likelihood]\label{rem:blind}
For stationary true data and the same conditional behaviour policy,
\begin{align*}
 \KL(P^\star_{1:T}\|P^\theta_{1:T})
 &=\KL(\mu\|\mu_1^\theta)+(T-1)\ell_0^{\mathrm{fwd}},\\
 \ell_0^{\mathrm{fwd}}
 &:=\Exp_{z\sim\mu}\Exp_{a\sim\pi_b(\cdot\mid z)}
       \KL(K^\star(z,a,\cdot)\|K_\theta(z,a,\cdot)).
\end{align*}
When these terms are finite, forward sequence KL is affine in $T$ without
an invariance assumption on the model. Teacher-forced negative log-likelihood
estimates the corresponding cross-entropy, differing from KL by the
model-independent data entropy when that entropy is well defined.
\end{remark}

\begin{proof}[Proof of Theorem~\ref{thm:B} and Remark~\ref{rem:blind}]
For the reverse direction, every summand in Lemma~\ref{lem:decomp} equals
$\ell_0^{\mathrm{rev}}$ and the initial KL is zero. For the forward
direction, the chain rule takes expectations under the true path law.
Its state marginal is $\mu$ at every step, giving the stated constant
conditional contribution irrespective of the model's rollout marginals.
\end{proof}

\subsection{Sequential scan, copulas, and operator certificates}\label{app:cert-mix}
This parameterization is distinct from the parallel joint architecture in Section~\ref{sec:architecture}:
\begin{equation}
y_i=\cos\phi_i\odot x_i+\sin\phi_i\odot\varepsilon_i,\qquad
\phi_i=\phi_i(y_{<i},x_{>i},\alpha),\quad\varepsilon_i\sim\gauss.
\label{eq:noisestage}
\end{equation}
Here $i$ indexes blocks, and $\alpha$ is an independent conditioning variable.

\begin{proof}[Exactness of the sequential scan]
Fix the independent action coordinate $\alpha$ and induct on $i$.
Write $C_i := (y_{1:i-1}, x_{i+1:B},\alpha)$ for the conditioning set. The inductive
hypothesis is $(y_{1:i-1}, x_{i:B})\mid\alpha\sim\gauss_d$. Given $C_i$, the block $x_i$ is conditionally
$\gauss$ --- the inductive hypothesis is a product Gaussian law conditional on $\alpha$. Since $\varepsilon_i\sim\gauss$ is fresh and
independent of everything, and $\phi_i$ is a function of $C_i$ only,
$y_i = \cos\phi_i x_i + \sin\phi_i\varepsilon_i$ is, conditionally on $C_i$, Gaussian with mean $0$
and variance $\cos^2\phi_i + \sin^2\phi_i = 1$, and is independent of $C_i$. Hence
$(y_{1:i}, x_{i+1:B})\sim\gauss_d$, closing the induction. Including an earlier \emph{input}
$x_j$ with $j<i$ in the arguments of $\phi_i$ can invalidate this conditioning argument: the resulting angle need not be
$C_i$-measurable, and the conditional Gaussian calculation no longer applies. Appendix~\ref{app:exactness} exhibits a scalar counterexample in which each
branch is measure preserving and the mixture is not.
\end{proof}

\begin{lemma}[Invariance under noise-only reparameterization]\label{lem:N}
Let a transition use a fresh independent Gaussian noise vector $\varepsilon$.
Replacing it by $W(\varepsilon)$ for a fixed Gaussian-preserving map $W$
leaves the conditional transition kernel unchanged. Parameters affecting
only this replacement are therefore unidentifiable from that kernel.
\end{lemma}

\begin{proof}
For each fixed state and action, the noise argument still has law
$W_\#\gauss=\gauss$. Integrating the same transition function against
this unchanged law gives the same conditional output law. State-dependent
variability must enter the transition's dependence on state, rather than
an independent preserving transformation of its noise alone.
\end{proof}

\begin{theorem}[Gaussian-preserving scalar kernels and copulas]\label{thm:L}
Let $k(z,\cdot)$ be a Markov kernel on $\mathbb{R}$ with density. Substituting $p = \Phi(z)$,
$u = \Phi(y)$ and setting $c(p,u) := k(\Phi^{-1}p, \Phi^{-1}u)/\varphi(\Phi^{-1}u)$ defines a
bijection, modulo null sets, between
\begin{enumerate}[leftmargin=1.6em,itemsep=1pt,topsep=2pt]
\item Markov kernels $k$ with $\gauss_1 k = \gauss_1$, and
\item copula densities $c$ on $(0,1)^2$, \ie non-negative $c$ with
$\int_0^1 c(p,u)\,\dd u = 1$ for a.e.\ $p$ and $\int_0^1 c(p,u)\,\dd p = 1$ for a.e.\ $u$
(doubly stochastic kernels on the unit square).
\end{enumerate}
The OU rotation $y=\rho z+\sqrt{1-\rho^2}\,\xi$ with $|\rho|<1$ corresponds to the
\emph{Gaussian} copula with correlation $\rho$. The set in (2) is convex and
infinite-dimensional; the Gaussian copulas form a one-parameter curve inside it.
\end{theorem}

\begin{proof}
Normalization of $k$ is $\int k(z,y)\,\dd y = 1$; substituting $y = \Phi^{-1}u$, $\dd y = \dd u/\varphi(\Phi^{-1}u)$
gives $\int_0^1 c(p,u)\,\dd u = 1$. Prior preservation is
$\int k(z,y)\varphi(z)\,\dd z = \varphi(y)$; substituting $z = \Phi^{-1}p$ and dividing by
$\varphi(y) = \varphi(\Phi^{-1}u)$ gives $\int_0^1c(p,u)\,\dd p = 1$. Both substitutions are
invertible, so the correspondence is a bijection. For the OU case, $(z,y)$ is a centered bivariate
Gaussian with correlation $\rho$ and unit marginals, whose copula is by definition the Gaussian
copula with that correlation. Convexity of (2) is immediate since both constraints are linear;
infinite-dimensionality follows because for any mean-zero $h\in L^\infty((0,1)^2)$ with
$\int h(p,u)\dd u = \int h(p,u)\dd p = 0$, $1+\epsilon h$ is a copula density for small $\epsilon$,
and such $h$ span an infinite-dimensional space (\eg the products
$\cos(2\pi jp)\cos(2\pi ku)$, $j,k\ge1$).
\end{proof}

The scalar correspondence identifies a concrete stochastic-family
restriction. Conditional coefficients can vary with admissible untouched
coordinates in a multivariate scan, so the full architecture is not a
single global one-parameter family. The example in
Appendix~\ref{app:chartsub} isolates the constant-correlation scalar case.

\begin{definition}[Markov operator and mean-zero norm]\label{def:transfer}
For a $\gauss$-preserving kernel $\tilde K$ define $P_\theta:L^2(\gauss)\to L^2(\gauss)$ by
$(P_\theta f)(z):=\int f(z')\,\tilde K(z,\dd z')$, and let
\begin{equation}
\sigma_2\;:=\;\norm{P_\theta}_{L^2_0(\gauss)}\;=\;
 \sup\Big\{\norm{P_\theta f}_{L^2(\gauss)} : \Exp_\gauss f=0,\ \norm{f}_{L^2(\gauss)}=1\Big\},
\label{eq:sigma2}
\end{equation}
the operator norm on the mean-zero subspace $L^2_0(\gauss)=\mathbf{1}^\perp$.
\end{definition}

\begin{theorem}[One-step contraction rate]\label{thm:M}
Let $\tilde K$ be $\gauss$-preserving. Then:
\textbf{(i)} $P_\theta\mathbf{1}=\mathbf{1}$, $\norm{P_\theta}_{L^2(\gauss)}=1$, and $P_\theta$ maps
$L^2_0(\gauss)$ into itself, so \eqref{eq:sigma2} is well posed and $\sigma_2\le1$.
\textbf{(ii)} For any initial law $\mu_0=(1+h)\gauss$ with $h\in L^2_0(\gauss)$ --- \ie any $\mu_0$ of
finite $\chi^2$ divergence from $\gauss$ --- the $\chi^2$ divergence contracts geometrically,
\begin{equation}
\chi^2(\mu_0\tilde K^T\,\|\,\gauss)^{1/2}\;\le\;\sigma_2^{\,T}\,\chi^2(\mu_0\,\|\,\gauss)^{1/2} ,
\label{eq:chi2rate}
\end{equation}
and hence $\norm{\mu_0\tilde K^T-\gauss}\le\sigma_2^{\,T}\chi^2(\mu_0\|\gauss)^{1/2}$ in the
total-variation norm Proposition~\ref{thm:R} uses.
\textbf{(iii)} $\sigma_2$ is a \emph{one-step} functional: it is determined by the population joint law
of $(z,z')$ with $z\sim\gauss$. Its definition is independent of the rollout horizon.
\textbf{(iv)} Restricting the supremum in \eqref{eq:sigma2} to any finite-dimensional
$V\subset L^2_0(\gauss)$ gives a lower bound $\sigma_2(V)\le\sigma_2$, and
$\sup_n\sigma_2(V_n)=\sigma_2$ along any nested family $V_1\subset V_2\subset\cdots$ whose union is dense in
$L^2_0(\gauss)$. The operator need not be compact, so the supremum in
\eqref{eq:sigma2} need not be attained and a fixed $V$ need not achieve equality.
\end{theorem}

\begin{proof}
(i) $P_\theta\mathbf{1}=\mathbf{1}$ is $\tilde K$'s normalization. By Jensen applied to the conditional
expectation, $\Exp_\gauss[(P_\theta f)^2]\le\Exp_\gauss[P_\theta(f^2)]=\Exp_\gauss[f^2]$, the last
equality being exactly $\gauss$-preservation; so $\norm{P_\theta}\le1$, with equality at $f=\mathbf{1}$.
For invariance of the mean-zero subspace, $\Exp_\gauss[P_\theta f]=\Exp_\gauss[f]=0$, again by
$\gauss$-preservation. 

(ii) Write $\mu_0=(1+h)\gauss$ with $\Exp_\gauss h=0$, so $\chi^2(\mu_0\|\gauss)=\norm{h}^2$. The
density of $\mu_0\tilde K^T$ against $\gauss$ is $1+P_\theta^{*T}h$, where $P_\theta^*$ is the adjoint
in $L^2(\gauss)$ --- the time reversal of $\tilde K$ --- and $\norm{P_\theta^*}_{L^2_0}=\sigma_2$
because an adjoint has the same norm. Iterating gives $\norm{P^{*T}_\theta h}\le\sigma_2^T\norm{h}$,
which is \eqref{eq:chi2rate}; the total-variation consequence is
$\norm{\mu_0\tilde K^T-\gauss}=\Exp_\gauss|P^{*T}_\theta h|\le\norm{P^{*T}_\theta h}$ by
Cauchy--Schwarz.

(iii) By disintegration the one-step joint law $\gauss(\dd z)\tilde K(z,\dd z')$ determines $\tilde K(z,\cdot)$
for $\gauss$-almost every $z$, hence determines $P_\theta$ as an operator on $L^2(\gauss)$ and therefore
$\sigma_2$.

(iv) A supremum over a subset is no larger, which gives $\sigma_2(V)\le\sigma_2$ for every $V$. For the nested
statement, let $\bigcup_n V_n$ be dense in $L^2_0(\gauss)$ and fix $\epsilon>0$. By definition of the supremum
there is $f\in L^2_0(\gauss)$, $\norm{f}=1$, with $\norm{P_\theta f}>\sigma_2-\epsilon$; by density there is
$g\in V_n$ for some $n$ with $\norm{f-g}$ small, and $\norm{P_\theta g}\ge\norm{P_\theta f}-
\norm{P_\theta}\norm{f-g}\ge\norm{P_\theta f}-\norm{f-g}$ since $\norm{P_\theta}\le1$ by (i). Normalising $g$
changes its Rayleigh quotient by $O(\norm{f-g})$, so $\sigma_2(V_n)>\sigma_2-2\epsilon$ for $n$ large, and the
supremum over the family is $\sigma_2$. Note this argument needs only $\norm{P_\theta}\le1$ and never that a
maximiser exists.
\end{proof}

\begin{remark}[Comparing contraction criteria]\label{rem:M-vs-dobrushin}
For scalar OU with $0<|c|<1$, the Dobrushin coefficient is one but
$\sigma_2=|c|$. A deterministic preserving bijection is unitary and has
$\sigma_2=1$; this excludes strict one-step density contraction but does
not exclude decay of correlations of individual observables.
\end{remark}

\begin{corollary}[Static level and convergence rate]\label{cor:M2}
Let $\Law(\hat x_T)=(g_\theta)_\#(\mu_0\tilde K^T)$. Suppose $\mu_0\ll\gauss$ and $\beta_\chi=\chi^2(\mu_0\|\gauss)^{1/2}<\infty$.
For a total-variation-dominated metric, or for energy distance on a bounded observation space,
\begin{equation}
\big|D(\Law(\hat x_T),p_{\mathrm{data}})-D(p_\theta,p_{\mathrm{data}})\big|
 \le C\sigma_2^T\beta_\chi,
\label{eq:plateau-rate}
\end{equation}
where $C$ is the corresponding domination constant; for energy distance one may take
$C=2\operatorname{diam}(\mathcal X)$ as in Remark~\ref{rem:Dsq}.
If $0<\sigma_2<1$ and $\beta_\chi>0$, a sufficient integer horizon for error at most $\epsilon$ is
\[
T\ge\max\left\{0,\left\lceil
\frac{\log(\epsilon/(C\beta_\chi))}{\log\sigma_2}\right\rceil\right\}.
\]
This is a population upper bound on the required horizon, not an equality for any particular observable.
Substituting a finite-basis estimate for $\sigma_2$ does not preserve the guarantee and does not yield a
lower bound on the actual stopping time.
\end{corollary}

\begin{proof}
Theorem~\ref{thm:M} bounds the total-variation deviation by $\sigma_2^T\beta_\chi$.
Apply Proposition~\ref{thm:R}'s metric argument, or Remark~\ref{rem:Dsq} for bounded-space energy distance,
and solve the resulting inequality for $T$. If $\beta_\chi=0$, the deviation is zero already; if
$\sigma_2=0$, it is zero after one step.
\end{proof}

\begin{remark}[Finite-basis estimation]\label{rem:certest}
For an orthonormal finite basis $\{h_i\}$ of $V\subset L^2_0(\gauss)$, one-step samples estimate
$M_{ij}=\Exp[h_i(z)h_j(z')]$. Its largest singular value is the norm of the compression
$\Pi_V P_\theta\Pi_V$, and
\[
\norm{\Pi_V P_\theta\Pi_V}\le\norm{P_\theta|_V}\le\sigma_2.
\]
Compression and restriction agree when $V$ is invariant under $P_\theta$.
For dense nested $V_n$, their projections $\Pi_n$ converge strongly to the
identity. For fixed unit $f$, $\Pi_nP\Pi_n f\to Pf$ in $L^2$; taking
$f$ arbitrarily close to a norm-maximizing direction proves
$\|\Pi_nP\Pi_n\|\to\|P\|$. No maximizing function or compactness of $P$
is required.

Sample singular values can exceed the population norm because of sampling
and selection error; cross-fitting does not convert them to certified upper
bounds. Coordinatewise Hermites without interaction terms are not a dense
multivariate family. The absolute bounds in Appendix~\ref{app:parallel-certificate}
use architectural caps and a specified initial second moment instead.
\end{remark}

\begin{theorem}[Architectural contraction bound]\label{thm:C}
Let the transition be $\Psi_{\mathrm{post}}\circ N\circ\Psi_{\mathrm{pre}}$, where both deterministic
maps are Gaussian-preserving diffeomorphisms and $N$ is the full $B$-block scan
\eqref{eq:noisestage}. Suppose a uniform cap $c_\star\in(0,1)$ satisfies
\begin{equation}
\mathop{\mathrm{ess\,sup}}_{i,j,\,\mathrm{context}}|\cos\phi_{i,j}|
 \le c_\star.
\label{eq:cstar}
\end{equation}
Then \textbf{(i)} $\sigma_2=\norm{N}_{L^2_0(\gauss)}$ and
\textbf{(ii)} $\norm{N}_{L^2_0(\gauss)}\le c_\star$.
For example, an architectural range
$\phi_{i,j}\in[\phi_{\min},\pi-\phi_{\min}]$, with
$\phi_{\min}\in(0,\pi/2)$, permits $c_\star=\cos\phi_{\min}$.
\textbf{(iii)} For $0<\beta_\chi=\chi^2(\mu_0\|\gauss)^{1/2}<\infty$,
\begin{equation}
T\ge\max\left\{0,\left\lceil\frac{\log(\epsilon/\beta_\chi)}{\log c_\star}\right\rceil\right\}
\quad\Longrightarrow\quad
\chi^2(\mu_0\tilde K^T\|\gauss)^{1/2}\le\epsilon.
\label{eq:certhorizon}
\end{equation}
The averaged controlled operator has the same upper bound if, for every fixed independent chart input
$\alpha$, its state transition has this factorization and the cap holds uniformly in $\alpha$.
Invariance of an arbitrary augmented map alone does not imply this contraction bound.
\end{theorem}

\begin{proof}
(i) $U_Bf:=f\circ B$ for a $\gauss$-preserving bijection $B$ is a unitary of $L^2(\gauss)$ fixing $\mathbf 1$,
hence a unitary of $L^2_0(\gauss)$. Writing $g=f\circ\Psi_{\mathrm{post}}$,
$(P_\theta f)(z)=\Exp_\varepsilon\big[g\big(N(\Psi_{\mathrm{pre}}(z),\varepsilon)\big)\big]
=\big(U_{\Psi_{\mathrm{pre}}}Ng\big)(z)$, so $P_\theta=U_{\Psi_{\mathrm{pre}}}\,N\,U_{\Psi_{\mathrm{post}}}$
and the outer factors, being unitary, leave the norm unchanged.

(ii) Write $N=K_1K_2\cdots K_B$, where $K_i$ replaces block $i$ by
$\cos\phi_i\odot(\cdot)+\sin\phi_i\odot\varepsilon_i$ with $\phi_i$ a function of the \emph{other} blocks. Each
$K_i$ preserves $\gauss$. Let $A_i:=\Exp[\,\cdot\mid z_{-i}]$ be the projection averaging out block $i$, and
$\Pi_{<i}:=\Exp[\,\cdot\mid z_1,\dots,z_{i-1}]$; note $A_B=\Pi_{<B}$ and $\Pi_{<1}=\Exp_\gauss$, which is $0$ on
$L^2_0$. Two facts:

\emph{(a) One step contracts exactly the part that depends on the block it refreshes.} For fixed context $z_{-i}$,
$K_i$ acts on block $i$ as a tensor product of scalar Ornstein--Uhlenbeck kernels with coefficients
$\cos\phi_{i,j}(z_{-i})$, which fixes constants and contracts the mean-zero part by
$\max_j|\cos\phi_{i,j}(z_{-i})|\le c_\star$. Integrating over $z_{-i}$,
\begin{equation}
\norm{K_ih}^2\;\le\;\norm{A_ih}^2+c_\star^2\big(\norm{h}^2-\norm{A_ih}^2\big).
\label{eq:onestep}
\end{equation}

\emph{(b) A partial scan preserves zero conditional mean over its updated blocks.} Let $N_{i+1}:=K_{i+1}\cdots K_B$
and suppose $\Pi_{\le i}h=0$. Starting from $z\sim\gauss$, the output $W$ of $N_{i+1}$ agrees with $z$ on blocks
$\le i$, and the exactness induction below \eqref{eq:noisestage} gives $(z_{\le i},W_{>i})\sim\gauss$ with the two
groups independent. Hence
$\Exp[h(W)\mid z_{\le i}]=\int h(z_{\le i},u)\,\gauss(\dd u)=(\Pi_{\le i}h)(z_{\le i})=0$, \ie
\begin{equation}
\Pi_{\le i}h=0\;\Longrightarrow\;\Pi_{\le i}N_{i+1}h=0 .
\label{eq:noleak}
\end{equation}

Now prove by downward induction on $i$ that for all $f\in L^2(\gauss)$,
\begin{equation}
\norm{N_if}^2\;\le\;\norm{\Pi_{<i}f}^2+c_\star^2\big(\norm{f}^2-\norm{\Pi_{<i}f}^2\big),
\qquad N_i:=K_i\cdots K_B .
\label{eq:IH}
\end{equation}
For $i=B$ this is \eqref{eq:onestep} with $A_B=\Pi_{<B}$. Assume \eqref{eq:IH} at $i+1$ and set $g:=N_{i+1}f$,
$P:=\norm{\Pi_{<i}f}^2$, $Q:=\norm{\Pi_{\le i}f}^2$, $F:=\norm{f}^2$, so $P\le Q\le F$. Split
$f=f_1+f_2$ with $f_1:=\Pi_{\le i}f$. Since $f_1$ does not depend on blocks $>i$ and each $K_j$, $j>i$, alters
only block $j$, we have $N_{i+1}f_1=f_1$ and therefore $A_iN_{i+1}f_1=A_i\Pi_{\le i}f=\Pi_{<i}f$. For $f_2$,
$\Pi_{\le i}f_2=0$, so the induction hypothesis gives $\norm{N_{i+1}f_2}^2\le c_\star^2\norm{f_2}^2$ and
\eqref{eq:noleak} gives $\Pi_{<i}A_iN_{i+1}f_2=\Pi_{<i}N_{i+1}f_2=0$. The two pieces of $A_ig$ are thus
orthogonal --- one is $\Pi_{<i}$-measurable, the other is annihilated by $\Pi_{<i}$ --- and
\begin{equation}
\norm{A_ig}^2=\norm{\Pi_{<i}f}^2+\norm{A_iN_{i+1}f_2}^2\;\le\;P+c_\star^2(F-Q).
\label{eq:crossterm}
\end{equation}
Applying \eqref{eq:onestep} to $N_i=K_ig$ and then
\eqref{eq:IH} at $i+1$ and \eqref{eq:crossterm},
\[
\begin{aligned}
\norm{N_if}^2\;&\le\;c_\star^2\norm{g}^2+(1-c_\star^2)\norm{A_ig}^2\\
&\le\;c_\star^2\big[Q+c_\star^2(F-Q)\big]+(1-c_\star^2)\big[P+c_\star^2(F-Q)\big]
\;=\;P+c_\star^2(F-P),
\end{aligned}
\]
which is \eqref{eq:IH} at $i$. At $i=1$, $\Pi_{<1}f=0$ for $f\in L^2_0$, giving
$\norm{Nf}\le c_\star\norm{f}$.

(iii) Combine (i), (ii) and Theorem~\ref{thm:M}(ii) and solve for $T$.

For the controlled extension, the theorem assumes the state factorization
and cap for every fixed chart input $\alpha$. Parts (i)--(ii) therefore give
$\|P_{\theta,\alpha}\|_{L^2_0(\gauss)}\le c_\star$ uniformly in $\alpha$.
This is an additional condition on the state kernel; it does not follow
merely from preservation of an augmented state--action law.

Now average. With $\alpha\sim\gauss_m$ independent of $z$ (Lemma~\ref{lem:chart}), the joint law is
$\gauss_d\otimes\gauss_m$, so by Jensen and then Fubini
\[
\norm{P_\theta f}^2=\Exp_z\big[(\Exp_\alpha P_{\theta,\alpha}f(z))^2\big]
 \le\Exp_z\Exp_\alpha\big[(P_{\theta,\alpha}f(z))^2\big]
 =\Exp_\alpha\norm{P_{\theta,\alpha}f}^2\le c_\star^2\norm{f}^2 .
\]
This averaging requires the same Gaussian action law at every state.
\end{proof}

\begin{remark}[Sharpness and scope]\label{rem:C-tight}
Constant coefficients equal to $c_\star$ give a tensor product OU kernel
with norm $c_\star$, so the class bound is sharp. The cap must control the
whole conditioner range; a sampled maximum does not establish it.
The frozen-latent initial data banks are atomic and have infinite initial
$\chi^2$. This theorem supplies a relative $L^2$ rate, while the finite-bank
absolute bound in Appendix~\ref{app:parallel-certificate} applies to the
separate parallel architecture and its stated hypotheses.
\end{remark}

\section{An explicit convergence certificate for parallel controlled transitions}
\label{app:parallel-certificate}

The following construction combines joint Gaussian-preserving maps with a
parallel Gaussian noise stage. Its convergence bound uses a second moment of
the specified initialization law, including a finite empirical law. The
analytic ingredients are Gaussian entropy contraction and relative-entropy
data processing; the contribution here is their application to a transition
with an exact conditional likelihood and a controlled occupancy constraint.
Theorem~\ref{thm:parallel-main} combines the joint-input result and finite-start
corollary proved below.

\subsection{Joint mixing and parallel conditional noise}
\label{app:joint-entropy}

Let $\Psi:\mathbb R^{d+m}\to\mathbb R^{d+m}$ be a Gaussian-preserving
bijection. The construction is the parallel transition \eqref{eq:kernel} of
Section~\ref{sec:architecture}, whose notation we reuse: given a state $z$ and
action chart coordinate $\alpha$, the joint mixer produces
$(h,b)=\Psi(z,\alpha)$, the noise-stage coefficients satisfy
$|c_i(b)|\le c_\star<1$ uniformly with fresh $\varepsilon\sim\gauss_d$, and
each post-map $S_b$ preserves $\gauss_d$. The families are jointly measurable.
For conditional likelihoods,
we use diffeomorphisms $S_b$ with computable inverses and determinants.
Joint orthogonal couplings provide $\Psi$; preserving couplings provide
$S_b$. The parallel implementation uses rotations in both positions.
Norm-changing preserving post-maps are also admissible. For $m=0$, the
action and residual-action blocks are absent, with $\gauss_0$ the one-point
probability space.

Under an exact behaviour action chart, $\alpha$ is Gaussian independently of
$z$, including for observation-conditioned policies by
Lemma~\ref{lem:observation-chart}. Observation context enters this transition
only through that chart. At the joint reference, $(h,b)$ is therefore product Gaussian.
Conditional on $b$, the noise stage and $S_b$ retain $\gauss_d$.
The state output consequently preserves the reference, even though its law
at a fixed physical action need not do so. Every state coordinate receives
noise in one parallel stage; the coefficients read $b$, not the
coordinates being refreshed.

\begin{theorem}[Joint-input entropy contraction]
\label{thm:parallel-kl}
Let $M$ be the kernel from $(z,\alpha)$ to $z'$ defined by
\eqref{eq:kernel}. For any joint input law $\nu$,
\begin{equation}
 D_{\mathrm{KL}}(\nu M\|\gauss_d)
 \le c_\star^2 D_{\mathrm{KL}}(\nu\|\gauss_{d+m}).
 \label{eq:parallel-joint-kl}
\end{equation}
In particular, under an exact behaviour action chart --- including
observation-conditioned charts, by Lemma~\ref{lem:observation-chart} --- the
joint input is $\mu\otimes\gauss_m$ and the behaviour-averaged state kernel $K$ obeys
$D_{\mathrm{KL}}(\mu K\|\gauss_d)
 \le c_\star^2D_{\mathrm{KL}}(\mu\|\gauss_d)$.
\end{theorem}

\begin{proof}
If $c_\star=0$, the conditional noise stage is a complete refresh, which
proves the claim directly. Otherwise, assume the input divergence is finite;
the infinite-divergence case is immediate.
Write $\rho=\Psi_\#\nu$. Divergence invariance under a preserving bijection
and the chain rule give
\[
 D_{\mathrm{KL}}(\nu\|\gauss_{d+m})
 =D_{\mathrm{KL}}(\rho_b\|\gauss_m)
  +\mathbb E_{\rho_b}
      D_{\mathrm{KL}}(\rho_{h\mid b}\|\gauss_d).
\]
For a constant Ornstein--Uhlenbeck (OU) coefficient $c=e^{-s}$, Gaussian logarithmic Sobolev
inequality and entropy dissipation yield
$D_{\mathrm{KL}}(\lambda Q_c\|\gauss_d)
 \le c^2D_{\mathrm{KL}}(\lambda\|\gauss_d)$
\citep{gentil2010}. Indeed, $s\mapsto Q_{e^{-s}}$ is the
Ornstein--Uhlenbeck flow, and along $\mu_s=\lambda Q_{e^{-s}}$ the de Bruijn
identity gives
$\frac{d}{ds}D_{\mathrm{KL}}(\mu_s\|\gauss_d)=-I(\mu_s\|\gauss_d)$, with
$I(\cdot\|\gauss_d)$ the relative Fisher information. The Gaussian
logarithmic Sobolev inequality with constant one states
$I(\mu_s\|\gauss_d)\ge2D_{\mathrm{KL}}(\mu_s\|\gauss_d)$, so integration over
$s$ multiplies the entropy by at most $e^{-2s}=c^2$. The inequality extends
from smooth densities to finite-entropy laws by approximation and lower
semicontinuity.

At fixed $b$, factor the diagonal OU stage into an OU kernel with scalar
coefficient $c_\star$ followed by coordinatewise OU kernels with coefficients
$c_i(b)/c_\star$. The latter preserve the reference, including a
reflection for a negative coefficient. Data processing therefore gives the
factor $c_\star^2$. Apply data processing through $S_b$ and convexity over
the actual law $\rho_b$ to obtain
\[
 D_{\mathrm{KL}}(\nu M\|\gauss_d)
 \le c_\star^2\mathbb E_{\rho_b}
      D_{\mathrm{KL}}(\rho_{h\mid b}\|\gauss_d),
\]
which proves \eqref{eq:parallel-joint-kl}. The chain-rule term includes
the dependence between $h$ and $b$ away from the reference. Under the behaviour chart the joint input
is $\mu\otimes\gauss_m$, giving the state-kernel statement.
\end{proof}

\subsection{A finite warm start and a decoded plateau bound}
\label{app:finite-warm-start}

\begin{corollary}[Absolute convergence from a finite-second-moment law]
\label{cor:parallel-absolute}
Under the exact behaviour action chart, let $\Psi$ also preserve the Euclidean norm pointwise, and assume
$M_2=\mathbb E_{\mu_0}\|Z\|^2<\infty$. With $B_\star$ as defined in
\eqref{eq:warm-main}, for every integer $T\ge1$,
\begin{equation}
 D_{\mathrm{KL}}(\mu_0K^T\|\gauss_d)
   \le c_\star^{2(T-1)}B_\star,\label{eq:parallel-absolute-kl}
\end{equation}
and the total-variation bound \eqref{eq:absolute-main} applies with the same
$B_\star$. Initial absolute continuity is unnecessary. For a specified uniform
finite bank $\{z_j\}_{j=1}^n$, one may use its exact moment
$M_2=n^{-1}\sum_j\|z_j\|^2$.
\end{corollary}

\begin{proof}
Condition on $(h,b)$ at the first transition. Convexity of relative
entropy and data processing through $S_b$ bound its output divergence by
\[
 \frac12\mathbb E\sum_i
 \left[c_i(b)^2h_i^2-c_i(b)^2
       -\log(1-c_i(b)^2)\right].
\]
The function $s\mapsto-s-\log(1-s)$ is increasing on $[0,1)$.
Furthermore, norm preservation and the fresh Gaussian action coordinate give
$\mathbb E\|h\|^2\le M_2+m$. The expression is consequently at most
$B_\star$. This Gaussian-mixture argument applies to atomic initial laws as
well. Theorem~\ref{thm:parallel-kl} contracts the finite first-step divergence
at all later transitions. Pinsker's inequality in the total-variation norm
used here is $\TV{P-Q}\le\sqrt{2D_{\mathrm{KL}}(P\|Q)}$.
\end{proof}

For $c_\star\in(0,1)$, a sufficient horizon for total-variation error at
most $\epsilon>0$ is
\[
 T\ge1+\max\left\{0,\left\lceil
 \frac{\log(\sqrt{2B_\star}/\epsilon)}{-\log c_\star}
 \right\rceil\right\}.
\]
At $c_\star=0$ the noise stage is complete refresh and the output is exactly
Gaussian after one step. A finite bank gives a certificate for that
specified initialization law. Replacing $M_2$ by a sample estimate for an
unknown population requires a valid population moment upper bound.

Let $F$ be a measurable decoder with output in a metric space of diameter
$\Delta$. For the population energy statistic $\mathcal E$ defined in
Remark~\ref{rem:Dsq},
\begin{equation}
 |\mathcal E(P,R)-\mathcal E(Q,R)|\le2\Delta\TV{P-Q}.
 \label{eq:parallel-energy}
\end{equation}
To see this, put $\xi=P-Q$ and
$f_S(x)=\int\mathrm{dist}(x,y)S(\dd y)$. The difference equals
$2\int f_R\dd\xi-\int(f_P+f_Q)\dd\xi$.
For a zero-mass signed measure,
$|\int f\dd\xi|\le\operatorname{osc}(f)\TV{\xi}/2$;
the two oscillations are bounded by $\Delta$ and $2\Delta$.
Since decoder pushforward contracts total variation, the static reference
$p_\theta=F_\#\gauss_d$ thus satisfies
\[
 |\mathcal E(F_\#(\mu_0K^T),p_{\mathrm{data}})
        -\mathcal E(p_\theta,p_{\mathrm{data}})|
 \le2\Delta\min\{2,\sqrt{2B_\star}\,c_\star^{T-1}\}.
\]
The distance must be bounded in the units actually used by the evaluator;
an unbounded feature or probe distance needs separate moment control.

\subsection{Policy deviation in the original action space}
\label{app:policy-physical}

Let $O_t$ denote the policy context, including a generated observation when
the policy reads frames, and let
$\lambda_t(\dd z,\dd o)=\mu_t(\dd z)L_t(\dd o\mid z)$ be its joint law.
Assume the exact chart $G_{z,o}$ is a bimeasurable bijection modulo common
null sets, maps $\pi_b(\cdot\mid z,o)$ to $\gauss_m$, and supplies the
transition's only dependence on $o$. At each fixed $(z,o)$,
\[
 D_{\mathrm{KL}}((G_{z,o})_\#\pi_t(\cdot\mid z,o)\|\gauss_m)
 =D_{\mathrm{KL}}(\pi_t(\cdot\mid z,o)\|\pi_b(\cdot\mid z,o)).
\]
The same equality holds for total variation. Define
\[
 H_t=D_{\mathrm{KL}}(\mu_t\|\gauss_d),\qquad
 \kappa_t=\mathbb E_{(Z,O)\sim\lambda_t}
 D_{\mathrm{KL}}(\pi_t(\cdot\mid Z,O)\|\pi_b(\cdot\mid Z,O)).
\]
Lemma~\ref{lem:observation-chart} gives a joint $(Z_t,\alpha_t)$ input
divergence at most $H_t+\kappa_t$; marginalizing a stochastic observation
can make the inequality strict. If the context is a function of the state,
this reduces to the state-conditioned equality. Theorem~\ref{thm:parallel-kl}
therefore gives the policy recurrence \eqref{eq:policy-kl-main}; its second
expression assumes $H_0<\infty$. From Gaussian initialization
and $\kappa_t\le\kappa$, it implies
\[
 \TV{\mu_T-\gauss_d}\le
 \min\left\{2,\sqrt{\frac{2c_\star^2(1-c_\star^{2T})\kappa}
                               {1-c_\star^2}}\right\}.
\]
An atomic state initialization may instead receive a behaviour-policy warm
start before introducing the policy shift. For a fitted invertible chart,
the same formulas use its induced policy
$\widehat\pi_b(\cdot\mid z,o)=(G_{z,o}^{-1})_\#\gauss_m$; identifying this with the true
behaviour policy requires checking the fitted conditional model. A small
total-variation chart residual alone does not upper-bound a KL residual.
The policy expectation is over the joint state and context law visited by
the deployed model. For two nonsingular tanh-Gaussian policies with the same
tanh transformation, the conditional divergence is the Gaussian KL before tanh.
Deterministic action selection, including a deterministic cross-entropy-method (CEM) choice, has
infinite KL against a nonsingular tanh-Gaussian behaviour law; the resulting
KL bound is vacuous for that policy.

For the more general controlled kernel of Theorem~\ref{thm:Aprime}, Pinsker
and Jensen still give $\eta_t\le\min\{2,\sqrt{2\kappa_t}\}$. This
generic perturbation result and the entropy contraction of
\eqref{eq:policy-kl-main} have different architectural hypotheses.

\subsection{An optional reference component}
\label{app:reference-component}

Let $\Gamma(z,\cdot)=\gauss_d$ and choose a constant $0<r<1$. For any
reference-preserving base kernel $K_0$, define
$K_r=(1-r)K_0+r\Gamma$. This standard independent-refresh construction
gives, for every initialization,
\begin{equation}
 \TV{\mu_0K_r^T-\gauss_d}\le2(1-r)^T.
 \label{eq:parallel-reset-tv}
\end{equation}
Indeed, every zero-mass signed measure $\xi$ satisfies
$\xi\Gamma=0$, and hence $\xi K_r=(1-r)\xi K_0$.
This argument requires a reset of the entire state appearing in the bound,
including any recurrent memory.

For a conditional density the mixture has
$k_r=(1-r)k_0+r\gauss_d$, so pointwise
\begin{equation}
 -\log k_r\le
 \min\{-\log k_0-\log(1-r),\ -\log\gauss_d-\log r\}.
 \label{eq:parallel-density-guard}
\end{equation}
With an invertible observation chart, the second bound becomes the model's
marginal negative log likelihood minus $\log r$. It bounds conditional
loss relative to the fitted marginal, whose accuracy remains determined by
the chart. Sampling uses an actual reference draw with probability $r$, and
the density is evaluated by a log-sum-exp, without clipping likelihoods.

For the parallel base kernel, convexity further replaces the entropy
contraction coefficient by $q=(1-r)c_\star^2$ and the warm-start bound by
$(1-r)B_\star$. Thus its absolute KL bound is
$(1-r)B_\star q^{T-1}$, and its policy recurrence is
$H_{t+1}\le q(H_t+\kappa_t)$. Under the behaviour policy, one may take
the smaller of the resulting Pinsker bound and \eqref{eq:parallel-reset-tv}. Independent refresh can
interrupt paths; its effect on short-horizon predictions and control must
be evaluated separately from its marginal guarantee.

\subsection{Conditional likelihood and numerical scope}
\label{app:conditional-numerics}

For fixed $(z,a,o)$, the chart coordinate and the values $(h,b)$ in
\eqref{eq:kernel} are deterministic. The exact conditional log density is
\eqref{eq:parallel-density-main} of Section~\ref{sec:architecture}. The bounds assume
the stated mathematical kernel, exact behaviour innovations, and a uniform
architectural coefficient cap. Gaussian diagnostics and floating-point
comparisons assess the implementation under these hypotheses; a global
roundoff certificate is outside their scope.

\section{Reproducible experimental methods}
\label{app:experiments}

Four studies connect the construction to model behaviour. Low-dimensional
systems isolate conditional representation and preservation. The twelve-task
fixed-latent pixel study evaluates transitions under an external OU action
process. CPU diagnostics check conditional loss and the parallel kernel's
finite-start bound. The recurrent study
jointly trains the inference network, decoder, reward head, and transition from pixels.

\subsection{Data and occupancy screening}
\label{app:cert}\label{app:screen}\label{app:dmcscreen}

\paragraph{Lorenz-63.}
We integrate $\dot x=10(y-x)$, $\dot y=x(28-z)-y$, and
$\dot z=xy-(8/3)z$ by Euler--Maruyama with independent additive Wiener
forcing of amplitude $0.5$ per coordinate and step size $0.01$.
After 10,000 integration steps of burn-in, we retain every fifth state.
The 200,000-state trajectory supplies 199,999 consecutive pairs. A contiguous
90/10 split discards the pair crossing the boundary and leaves 19,999
held-out pairs. Consecutive held-out observations remain dependent.
The nonstationary control ramps $\rho$ from 28 to 45 over the recorded window,
with $\rho=28$ during burn-in.

\paragraph{Sequential-scan pixel collection.}
DeepMind Control \citep{tassa2018dmc} observations are $64\times64$ RGB images,
collected with action repeat two. The driver is an OU process with time scale
8 and noise scale $0.8$, clipped to the action box; episodes last 60,000
steps. The simulator together with the driver is time homogeneous. The
observed frame alone need not be Markov, and clipped OU actions are not exact
Gaussian innovations. These experiments use the sequential-scan parameterization,
which preserves the
reference for every fixed external action. A reference start
independent of the entire external action sequence and transition noise then
retains its reference marginal conditional on that sequence. This is a narrower
experiment than the state-dependent occupancy construction in Section~\ref{sec:architecture}.

The screen covers 22 tasks with 96 independent chains each. Eighteen pass its
temporal and coverage checks. Finger turn hard and fish upright have unresolved
transients; cheetah run and stacker stack 2 develop concentrated occupancies.
The twelve evaluated tasks are the headline subset of the eighteen, spanning
state dimensions $3$--$53$ and $1$--$21$ actuators; the six other healthy tasks
(walker stand, ball in cup catch, cartpole two poles, point mass easy, swimmer
with six links, and fish swim) remain screened and available but are not fitted.
The twelve evaluated tasks are all reported in Table~\ref{tab:e11e}.
Passing this finite-data screen does not prove stationarity or a latent Markov
representation. The screen compares first and second moments in 2,000-step
windows against late windows beginning after 30,000 steps, using chain-level
averages and Bonferroni adjustment across windows. A matched leave-one-out
window-to-window energy-distance baseline additionally checks distributional
homogeneity. Autocorrelation times and effective sample sizes describe temporal
dependence. Late-to-widest coordinate spread and frozen-coordinate counts
check coverage separately from temporal stability.

Swimmer (15 links) is retained as an ensemble occupancy with a metastability flag:
its between-chain distance is $0.0186$ against a matched late-window baseline
of $0.00065$. Individual chains need not explore that ensemble.
A separate cheetah run with 128 chains and 200,000 steps still concentrates
its torso pose while its limbs move. Its exclusion concerns this exploration
policy's occupancy, rather than the task's value as a control benchmark.
For a homogeneous process, the $2/N$ invariance-defect bound for a length-$N$
time-averaged population law is separate from these diagnostics; repeated
chains do not increase $N$ in that statement.

\subsection{Sequential-scan fitting objectives and budgets}
\label{app:fitting-budgets}

For low-dimensional systems, the state chart $E$ has eight flow layers with
hidden width 256. Four deterministic layers precede and follow a block-scan
noise stage. Each deterministic layer combines rotation couplings and quantile
swirls with three rounds, three profiles per round, and radial levels $(2,3,4)$.
The sequential block scan is distinct from the parallel construction of
Section~\ref{sec:architecture}; its conditioning and operator bound are described
in Appendix~\ref{app:arch}.

The pair-likelihood experiments minimize
\[
 -\frac{1}{d}\,\Exp_{(x,a,x')\sim\widehat\pi}
 \left[w\log p_E(x)+\log p_\theta(x'\mid x,a)\right].
\]
Here $d$ is the modeled state dimension, including $d=256$ for the frozen pixel
latent, and $w=1$ unless swept. Marginal and conditional density changes include
the appropriate learned-chart Jacobians. The low-dimensional pair trainer and
pixel trainer use AdamW with zero weight decay, initial learning rate $10^{-3}$,
cosine decay, and gradient clipping at 10. Nonfinite updates are recorded and
training aborts if their allowed fraction is exceeded. The default batch sizes
are 16,384 and 2,048, respectively. Raw-latent Gaussian heads minimize only the
conditional NLL; conditional flow matching minimizes its regression objective.
Thus a flow-matching loss is not a comparable likelihood value.

\begin{wmtable}{Training budgets. Counts are gradient updates per configuration. A training seed is a fitted-model replication; repeated rollout samples are not additional training seeds.}{tab:retained-budgets}
\small\setlength{\tabcolsep}{5pt}
\begin{tabular}{@{}>{\raggedright\arraybackslash}p{0.44\linewidth}r>{\raggedright\arraybackslash}p{0.36\linewidth}@{}}
\toprule
Comparison & Updates & Replication and scope \\
\midrule
Lorenz prediction--rollout calibration & 40,000 & Three training seeds; two noise blocks \\
Lorenz marginal-weight sweep & 30,000 & Four constrained weights and two free arms; three blocks \\
Lorenz soft stabilization & 8,000 & Nine penalty/weight cells at seed 0; selected arm at three seeds \\
Frozen-latent pixel transitions & 30,000 & Twelve tasks; one fitted seed per task/arm \\
Capped pixel transitions & 30,000 & Twelve tasks at $.999$; three-task cap sweep and retraining controls \\
Controlled Gaussian and Duffing & 20,000 & One seed per Gaussian configuration; three planning seeds \\
Finite-divergence memory comparison & 4,000 & Six synthetic configurations \\
Non-Gaussian copula example & 12,000 & Three training seeds \\
\bottomrule
\end{tabular}
\end{wmtable}

The seven-model Lorenz calibration combines three 40,000-update replicates and
four 30,000-update weight settings. It tests each fitted model's own static
prediction, rather than comparing seven models at identical capacity and budget.
The soft-stabilization objective instead uses total three-dimensional pair NLL,
without division by $d$, plus its weighted penalty. Its table therefore reports
nats per pair, and its numerical weights cannot be transferred directly to the
per-dimension objective above.

\subsection{Distances, uncertainty, and reference dependence}
\label{app:refbias}

The recurrent and native studies use root-mean-square (RMS) distance over normalized RGB channels.
The fixed-latent studies use unnormalized Euclidean distance, either in 256-dimensional
frozen-autoencoder coordinates (including the Reacher static calibration) or on flattened,
downsampled decoded images, as labeled. These scores have different units; the diameter-one
bound applies only to the normalized pixel geometry.

The population energy functional is estimated by a complete two-sample U-statistic in the
synthetic and fixed-latent studies. The recurrent study instead uses an unbiased linear
estimator: each repetition splits 1,024 samples into 512 disjoint pair blocks, averaging
their cross- and within-sample distance terms; eight repetitions provide paired evaluation
SEs. Unbiasedness in each case assumes independent draws. Finite estimates can be negative. At extreme latent scales,
subtraction of large distance terms can also become numerically unreliable.
All rollout comparisons report the retained finite fraction; a distance computed
after discarding failed chains describes their surviving subpopulation.
The alternate-metric study additionally reports the fraction inside its specified
box. These quantities cannot be read as unconditional rollout accuracy.

Monte Carlo SEs condition on the fitted model and collected dataset. Where both
levels share a reference repetition, the paired SE is computed from their
within-repetition differences. The Lorenz seed table has only separate saved
level summaries, so its quadrature-combined SE omits the unknown covariance
and is labeled nominal. The twelve-task paired comparison uses eight shared
reference repetitions; detailed Walker rollout summaries use four. Training-seed
variation is reported separately.

Recurrent equal-task SEs and paired percentile intervals condition on the fixed task suite and
assume independent task--training-seed cells. The same three numerical seed IDs initialize
the random streams across tasks, so cross-task independence is an approximation. Grouping
each shared seed across all tasks instead gives an occupancy-gap SE of $0.000602$, compared
with the reported within-task SE of $0.000325$; the mean remains $-0.000182$.
Both summaries are compatible with zero, but three shared-seed groups provide limited
uncertainty calibration. Neither summary covers collection uncertainty or a population of tasks.

Temporal dependence in the reference also matters. Within-reference terms in
either estimator can be biased by correlated observations. That common term
cancels in a paired difference using the same reference pool, but the cross
terms still have conditional sampling variation. The Lorenz disjoint-subset
baseline is $0.0007\pm0.0020$ (SE over eight splits), with single-split SD
$0.0057$. A separate 64-record calibration estimates a correlation-induced
offset of $0.01032\pm0.00022$ at lag 128. Pixel reference diagnostics find offsets
as large as $0.196$ on Swimmer. These are estimand-specific diagnostics, not
universal corrections to be subtracted from other scores. Reference sampling
uncertainty, fitted-seed variation, and collection-to-collection uncertainty
remain different quantities.

\subsection{Jointly trained recurrent pixel model and GPU protocol}
\label{app:gpu-protocol}

The model and objective are specified in Appendix~\ref{app:recurrent-model},
and data separation and evaluation in Appendix~\ref{app:recurrent-evaluation}.
Initialization, gain, and planning diagnostics appear in
Appendices~\ref{app:posterior-excursion}--\ref{app:common-planning};
the native baseline and complete results follow in
Appendices~\ref{app:native-baseline} and~\ref{app:recurrent-results}.

\subsubsection{Model and sequential objective}
\label{app:recurrent-model}
The implemented model jointly learns a four-stage convolutional neural network (CNN) observation encoder,
a recursive diagonal-Gaussian posterior
$q_\varphi(z_t\mid z_{t-1},o_t,a_{t-1})$, a convolutional decoder, a reward head,
and a transition prior. Every persistent recurrent coordinate belongs to the
sampled stochastic vector $z_t$. There is no additional deterministic recurrent
memory or privileged simulator state supplied to the model. The study uses latent,
feature, and shared hidden widths of 256, with CNN base width 32 and $64\times64$
RGB images. Only the prior-network hidden width is adjusted to match transition
parameter budgets; the shared backbone remains identical. Occupancy and conditioner-only
priors use width 256. Gaussian, overshooting, and invariant-penalty priors use 600,
or 608 on Humanoid, Quadruped, and Swimmer. Conditional RealNVP uses 168,
or 176 on Hopper, Humanoid, Manipulator, Quadruped, Swimmer, and Walker.

Collection uses a saved observation-conditioned tanh-Gaussian behaviour policy,
\[
 v_t=m_b(o_t)+s_b(o_t)\odot\alpha_t,\qquad
 a_t=\tanh v_t,\qquad \alpha_t\sim\gauss.
\]
The collection shards store $m_b(o_t)$, $s_b(o_t)>0$, and $v_t$, so the
training innovation is the known quantity
$\alpha_t=(v_t-m_b(o_t))/s_b(o_t)$. The loader checks these fields and the
consistency of $\tanh v_t$ with the recorded action. Missing policy context is
an error for occupancy training. The fixed policy and its content hash are
recorded with the dataset. This chart uses the recorded policy parameters;
clipped OU actions in the fixed-latent dataset do not supply these innovations.

Before collection, reward-weighted imitation fits the policy on separate
calibration trajectories and freezes the validation-selected checkpoint.
Boundary action targets use a recorded inverse-tanh clipping convention;
they are supervised targets, not inferred behaviour innovations. Fitting alone
does not establish improved return. Fresh independent chains supply the
benchmark, with return and within-chain motion reported alongside temporal
moment and spread diagnostics. Failed tasks remain in the attempted-task
denominator and are not replaced.

All six arms share the same CNN, posterior, decoder, and reward architecture,
initialized identically under a common seed. Their priors are the parallel
occupancy kernel; the same kernel without joint state--action mixing
(\emph{conditioner only}); a diagonal Gaussian; conditional RealNVP; the Gaussian
with overshooting; and the Gaussian with an invariant-measure penalty. All use
the same innovation input by default. The conditioner-only arm is not
fixed-physical-action preservation when the chart depends on the observation.
The occupancy prior uses four joint layers, $c_\star=0.999$, and reset probability
$10^{-3}$ by default. Its certificate concerns the generative prior, not the
aggregate variational posterior or an invertible chart of real images.

\paragraph{Sequential objective and units.}\label{app:recurrent-objective}
For $L$ observed frames, the negative evidence lower bound (ELBO) of the joint behaviour-generated sequence
model, up to known behaviour-action log densities, is
\begin{align}
 \mathcal L_{\mathrm{ELBO}}=\frac1L\Bigg(&
 \Exp_q\!\left[-\sum_{t=0}^{L-1}\log p_\theta(o_t\mid z_t)
 -\sum_{t=0}^{L-2}\log p_\theta(R_t\mid z_t,a_t,z_{t+1})\right]\notag\\
 &+\KL(q_\varphi(z_0\mid o_0)\|\gauss_d)
 +\sum_{t=0}^{L-2}\Exp_q\!\left[\log\frac{q_\varphi(z_{t+1}\mid z_t,o_{t+1},a_t)}
 {k_r(z_{t+1}\mid z_t,a_t,o_t)}\right]\Bigg].
 \label{eq:recurrent-objective}
\end{align}
Under the known behaviour chart this is the implemented conditioning: the
prior network reads $(z_t,\alpha_t)$ with the chart coordinate
$\alpha_t=G_{z_t,o_t}(a_t)$, an equivalent conditioning to $(a_t,o_t)$ given
$z_t$. The generative ordering supplies the observation context: $o_t$ is emitted from $z_t$ before $z_{t+1}$ is drawn, so
$p(o_{t+1},z_{t+1}\mid z_t,a_t,o_t)=p_\theta(o_{t+1}\mid z_{t+1})\,
k_r(z_{t+1}\mid z_t,a_t,o_t)$.
The initial KL is analytic; transition terms use reparameterized Monte Carlo
samples and the exact one-step density of the selected prior. A sampled KL
contribution can be negative and is not clipped. Pixel and latent coordinates
are summed; division by $L$ occurs after summation. Unit observation, KL, and
reward weights give this ELBO and are used in all primary fits.
The Gaussian emission has fixed standard deviation $0.1$
on normalized RGB intensities; its sigmoid decoder output is the mean. This is
a continuous density, not an exact discrete 8-bit image likelihood. Rewards use
a scalar Gaussian with standard deviation one. The outgoing reward is $R_t$, with timing
$o_t\xrightarrow{a_t,R_t}o_{t+1}$.

Overshooting samples intermediate prior states and uses detached posterior
targets; the resulting expected final conditional KL is an auxiliary bound,
not an exactly evaluated multistep mixture KL. The invariant penalty is biased
empirical RBF MMD squared between prior outputs driven by independent reference
states/innovations and fresh reference samples; it uses 128 samples and
bandwidths $(0.5,1,2)$ with squared distances divided by latent dimension.
Both auxiliary terms and their configured weights are logged separately.
The overshooting horizon is three; the primary overshooting and invariant-penalty weights are $0.1$ and $1$:
these are fixed-configuration mechanism controls, not a best-tuned comparison
of all regularization objectives.

\subsubsection{Data separation and evaluation}
\label{app:recurrent-evaluation}
Each task has 32 independently initialized collection trajectories. After 10,000
burn-in decisions per trajectory, we retain 20,000 transitions with action repeat
two and camera zero: 640,000 transitions per task. Shards store uint8 frames,
outgoing actions/rewards, reset flags, and policy fields. A fixed split seed (9041)
assigns 25/3/4 shards to training/validation/test. Training windows never cross
reset or shard boundaries and use no additional burn-in; filtering restarts at
each window. A training minibatch samples a shard proportional to its valid-window
count, then draws 16 windows within it. Held-out sampling uses independent windows.
We use Adam at constant learning rate $3\times10^{-4}$, zero weight decay, and
gradient clipping at 100. Validation runs every 1,000 updates on 16 batches;
reported evaluations use the final 100,000-update checkpoint.

\paragraph{Evaluation protocol.}
The recurrent study contains twelve tasks, six primary
arms, and three paired training seeds (1101--1103). Each fit uses 100,000 updates
of 16-frame windows.
The protocol separates held-out forecasts, paired static-versus-rollout distances, interventions
on preservation, and pixel-based planning. Policy shift is evaluated only for
randomized kernels, without fitted dynamics (Figure~\ref{fig:certificate}b);
fitted-model policy-shift performance is unmeasured. Posterior initialization, reference
sampling, actions, transitions, emissions, and planning have distinct recorded
random streams. Empirical initialization banks are frozen before evaluation;
their actual second moment enters the finite-start bound. The architectural
bound applies in real arithmetic under its stated behaviour innovations. At the
evaluation horizon $T=10^5$ with the default reset probability $r=10^{-3}$, the
population bound on the marginal gap supplied by \eqref{eq:parallel-reset-tv}
and \eqref{eq:parallel-energy} is at most
$2\Delta\cdot2(1-r)^T\le1.5\times10^{-43}$, far below the measurement's
$\pm3\times10^{-4}$ equal-task training-seed standard error. This SE includes variation
in checkpoint-level estimates after averaging evaluation repetitions; it does not separately
resolve their Monte Carlo error. The terminal gap checks numerical consistency with the
static reference, whereas intermediate horizons assess bound sharpness
(Figure~\ref{fig:certificate}, Appendix~\ref{app:cpu-diagnostics}).
Observation-space diagnostics and finite-particle likelihood estimates remain
separate from exact latent conditional densities and analytic bounds.

A physical-action forecast or planner must recompute the known policy's location
and scale from its current generated observation. Using a future true frame's
policy context would reveal future observations. Replaying recorded innovations
instead defines a distinct behaviour-noise replay experiment. CEM selection
uses only model predictions; the simulator executes chosen actions for evaluation.
Emission means and samples define different evaluated laws, as do physical-action
interventions and behaviour-innovation rollouts. %
Reported metrics use completed evaluations, with failed cells retained in the
relevant denominator.

Short forecasts impose the recorded future physical actions as interventions.
They do not condition on the information those actions reveal under the
behaviour policy. We evaluate 32 batches of 16 held-out windows, with prefix length
four, horizons 1/5/10, and nested 16/64 particles for the approximate filtered
observation scores. Long evaluation uses
1,024 chains in each of eight repetitions, a 4,096-state finite initialization
bank, and horizons through $10^5$. The metric compares bounded decoder means
with held-out pixels using RMS pixel distance; the Gaussian emission samples
are a distinct unbounded law. Shared references and all raw paired energy
terms are retained.

Matching uses validation observation NLL within $0.02$ nats and each 1/5/10-step mean squared error (MSE)
within $10^{-4}+0.05$ times the reference error. Parameters must be within $5\%$ and
measured training compute within $10\%$. These are eligibility gates, not equivalence tests.
All failed matches remain in the study.
A frozen output-gain intervention estimates that specific scaling effect;
it is not a general causal effect of stationarity.

\subsubsection{Posterior initialization and guard crossings}
\label{app:posterior-excursion}
The conditioner-only arm's single escape cell (swimmer, seed 1103; $21$ of $8{,}192$ chains)
crosses the latent guard at the first step from large posterior-bank initializations.
Every escape is present at horizon one, with per-repetition surviving fractions constant
thereafter, and the $\gauss_d$-initialized control cloud of the same checkpoint survives
$8{,}192$ of $8{,}192$ chains to $10^5$ steps. The escaped chains remain finite: their retained
pre-censoring states have norms $3.05$--$7.87\times10^6$, just above the $\pm10^6$ guard, and
they map onto the eight largest rows of the cell's $4{,}096$-state initialization bank, of
which $29$ rows have norms $1.06$--$7.98\times10^6$ (bank second moment $7.4\times10^{10}$
against $d=256$ under $\gauss_d$). Re-encoding the recorded source windows with the frozen
checkpoint reproduces posterior means of norm $1.3$--$3.7\times10^6$ at cosine $1.0000$ to the
bank rows; all $29$ rows come from one contiguous segment of a single held-out test episode.
This traces the large initial states to the fitted posterior mean on that segment.
Fitted transition coefficients stay strictly inside the $0.999$ cap, the rotation
stack preserves norms to float32 accuracy, and guard-free replays from the same initial states
re-enter the typical set within $\approx3\times10^3$ steps. Theorem~\ref{thm:A} covers the
conditioner-only variant as the specialization of \eqref{eq:kernel} whose joint mixer keeps the
action block fixed. The bank is not a Gaussian initialization: its second moment exceeds
the reference moment by eight orders of magnitude. The finite-start bound permits this
early transient (total variation $\le1.998$ at $T=1$). The Gaussian
arm's escapes on the same task and seed show the opposite signature---progressive leakage
appearing between $10^4$ and $10^5$ steps in both clouds.

\subsubsection{Output-gain controls}
\label{app:gain-control}
Table~\ref{tab:gain1003} compares the primary model with two output gains.
Gain $1.001$ is the validation-selected intervention; gain $1.003$ is exploratory.
Both use the same checkpoints, banks, initialization draws, reference pools, random streams,
and evaluator as the unscaled model, verified by per-job input hashes.
The gain multiplies the sampled output, including the reference-reset branch, and therefore
changes the Gaussian reference law. All cells retain their chains at the measured horizons,
while second-moment deviations increase with gain. With reset probability $10^{-3}$ retained,
these measurements characterize finite-horizon survival and marginal sensitivity.
\begin{wmtable}{\textbf{Output scaling and finite-horizon marginal sensitivity.}
All gains share checkpoints, initialization banks, and random streams; scaling applies to the
sampled prior output, including the reset branch ($r=10^{-3}$).
Survival counts cells retaining all chains at $10^5$ steps ($36$ cells, $1{,}024$ chains
$\times$ $8$ repetitions each). Marginal gap reports the cell mean [minimum, maximum].
The last column is the largest coordinate second-moment deviation from one at that horizon.}{tab:gain1003}
\small\setlength{\tabcolsep}{3pt}
\begin{tabular}{@{}llll@{}}
\toprule
Output gain & Survival & Marginal gap at $10^5$
& 2nd-moment dev. \\
\midrule
$1.000$ (primary) & $36/36$ & $-0.00018$ [$-0.0051$, $+0.0065$] & $0.15$ \\
$1.001$ (selected) & $36/36$ & $-0.00034$ [$-0.0048$, $+0.0072$] & $0.54$ \\
$1.003$ (exploratory) & $36/36$ & $-0.00081$ [$-0.0121$, $+0.0086$] & $6.92$ \\
\bottomrule
\end{tabular}
\end{wmtable}

\subsubsection{Planning protocol and common-cell contrasts}
\label{app:common-planning}
The primary CEM comparison uses horizon 20, 64 candidates, eight elites, three
iterations and four particles, over ten paired 200-decision episodes with four
warm-up decisions. The discount is $0.99$; the pre-tanh proposal starts at standard
deviation $1$ and has minimum standard deviation $0.05$.
It entails 30.72 million imagined transitions per fitted model. Horizons
5 and 50 form a separate prespecified sweep on the four priority tasks with
occupancy and Gaussian models. Full imagined-step cost, GPU memory, training
throughput and prior sampling are profiled separately.

On the intersection of completed primary plan cells, the equal-task mean normalized-return
difference against occupancy (candidate minus occupancy, with the cross-task standard error)
is $-0.0000\pm0.0007$ for the conditioner-only variant ($17$ common cells), $-0.0053\pm0.0068$
for the Gaussian ($13$), $+0.0089\pm0.0080$ for conditional RealNVP ($11$), $-0.0128\pm0.0135$
for latent overshooting ($12$), and $+0.0051\pm0.0066$ for the invariant penalty ($12$).
These descriptive contrasts condition on completion by both planners; their cross-task
SEs do not measure training-seed uncertainty on the full task--seed grid.
Against the native reference, four priority-suite
cells have both planners complete (the three swimmer cells and quadruped walk seed 1103); the
native model returns more on three of them ($0.1312$ and $0.1246$ on swimmer seeds 1101--1102
and $0.1098$ on quadruped, against $0.1109$, $0.1188$, $0.1041$) and less on swimmer seed 1103
($0.1152$ against $0.1211$). The selected common cells do not establish a
suite-wide planning comparison.

\subsubsection{Native recurrent baseline}\label{app:native-baseline}
The offline DreamerV3 comparison \citep{hafner2023dreamerv3} retains the original CNN encoder/decoder,
categorical recurrent state-space model (RSSM) with deterministic memory, KL balancing and free nats,
reward and continuation heads, and world-model optimizer. We use
\href{https://github.com/danijar/dreamerv3/tree/e3f02248693a79dc8b0ebd62c93683888ddaccfe}{\texttt{danijar/dreamerv3}, revision \texttt{e3f0224}},
with offline window sampling and no actor--critic updates.
The size12m configuration uses 64-frame training windows on the four priority
tasks, with the same three training seeds. Its world model has $8.85$ million
parameters, against $3.26$--$3.29$ million in the primary models; it trains
for 25,000 updates of 16 windows in bfloat16, using the native optimizer at
learning rate $4\times10^{-5}$ with 1,000 warm-up updates and adaptive gradient
clipping $0.3$. Evaluation uses the final checkpoint. The primary fits use
100,000 updates of 16 windows of length 16, so both protocols observe exactly
$25.6$ million frames per model; wall time is not equated.
Table~\ref{tab:native-results} reports the corresponding model sizes and measured outcomes.

The native image head uses the original bounded MSE objective. For common
observation scoring we explicitly evaluate a Gaussian of standard deviation
$0.1$ around its predicted image mean; this external score is distinct from
the native training objective. Outgoing rewards are shifted to the incoming
state convention, the first unobserved reward is masked, and shard boundaries
are not labeled as environment terminations.

Native forecasts and CEM use the same physical actions and evaluation budgets;
Table~\ref{tab:native-results} reports the measured comparison.
Under the frame-conditioned behaviour policy, native long rollouts decode each
step to obtain the next action context. This cost differs from the preserving
model's direct innovation rollout. Native rollouts start from filtered data
and extend through $10^3$ steps, the common comparison horizon with the primary models;
there is no Gaussian latent reference, static-limit identity, or convergence
certificate for that baseline.

\subsubsection{Detailed recurrent and native-model results}
\label{app:recurrent-results}
Figure~\ref{fig:survival-forest} reports paired survival contrasts.
\begin{figure}[htbp]
\centering
\includegraphics[width=\textwidth]{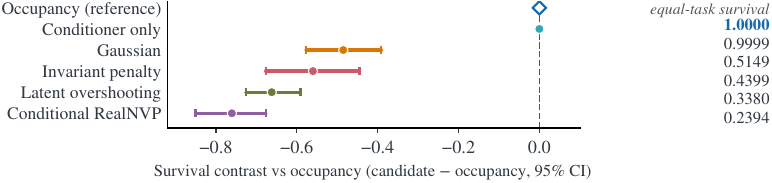}
\caption{\textbf{All four nonpreserving arms lose chains at $10^5$ steps.}
Paired survival contrasts (candidate minus occupancy), with nominal $95\%$ paired training-seed
percentile-bootstrap intervals over twelve tasks and three seeds (10,000 draws, fixed task suite).
The right column gives equal-task survival. Occupancy retains every evaluated chain.}
\label{fig:survival-forest}
\end{figure}

\paragraph{Recurrent evaluation coverage and uncertainty.}
The primary comparison contains 216 fitted task/seed/arm checkpoints; 168 have at least
one unavailable estimate because a rollout is censored or planning fails. Of 648 primary
evaluation stages, 545 complete and 103 have deterministic planning failures.
The corresponding counts are 178/216 completed stages for ablations, 12/24 for the
planning sweep, and 89/108 for the selected gain, with 38, 12, and 19 planning
failures, respectively. There are 172 failures in all. Each primary, ablation, and
selected-gain checkpoint has a short-horizon, long-horizon, and planning stage;
the sweep contains 24 task--arm--seed cells. The protocol does not estimate unconditional
distances from censored rollouts or full-grid paired returns when a required planner fails.

Gaussian minus occupancy one-step observation NLL is $1.017$ nats per pixel channel
(paired training-seed interval $[-1.324, 3.413]$). Returns are divided by the
planned physics-step budget, including early termination. Uncertainty uses training
seeds as replicates within the fixed task suite; rollout repetitions are averaged
within each checkpoint.

The validation-selected gain $1.001$ leaves survival unchanged (contrast $0$, nominal interval
$[0,0]$) and changes the marginal gap by $-0.00016$ (nominal interval
$[-0.00029,-0.00004]$). The prespecified paired comparisons use this selected gain;
the exploratory $1.003$ control appears in
Appendix~\ref{app:gain-control}.

Figure~\ref{fig:recurrent-results} summarizes the fitted-grid measurements.
\begin{figure}[p]
\centering
\includegraphics[width=\textwidth]{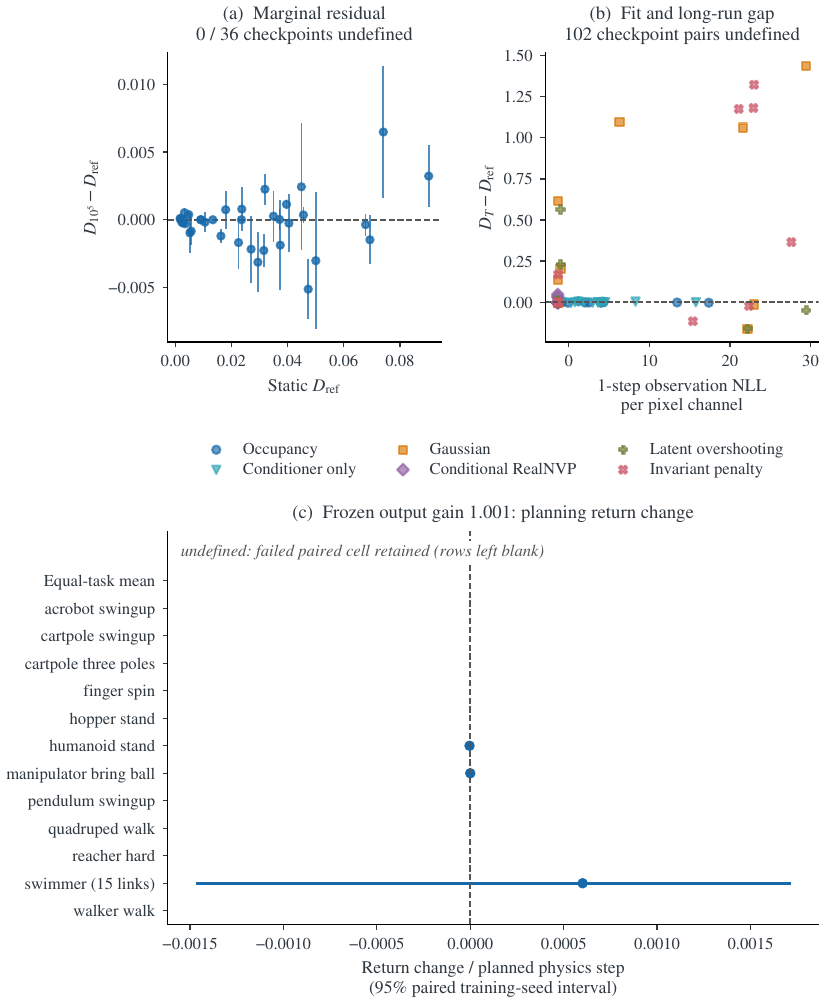}
\caption{\textbf{Marginal agreement, conditional fit, and planning are distinct outcomes.}
\textbf{(a)} The $10^5$-step rollout-minus-static energy gap versus the static score for occupancy checkpoints;
vertical bars show $\pm1$ paired evaluation SE of that gap; horizontal uncertainty is not shown.
\textbf{(b)} Observation NLL versus marginal gap for cells with uncensored estimates;
102 censored cells are excluded from this scatter and retained in survival summaries.
\textbf{(c)} Validation-selected gain return contrasts with nominal paired training-seed
percentile-bootstrap intervals; failed comparisons remain unestimated.
$D_{\rm ref}$ is a limiting level for preserving models under their convergence assumptions.}
\label{fig:recurrent-results}
\end{figure}

\begin{wmtable}{Native offline DreamerV3 reference on 4 eligible / 4 registered priority tasks. Values are equal-task means $\pm$ training-seed SE. NLL is a common external Gaussian image score, not the native training likelihood. Return uses the planned physics-step budget. Observed-frame exposure is matched at 25.6M frames per model; the native reference is granted the larger capacity and context (8.85M vs.\ 3.26--3.29M parameters, 64-frame vs.\ 16-frame training windows) at fewer updates (25,000 vs.\ 100,000); wall time is not equated. Bold marks the lowest measured NLL and MSE means.}{tab:native-results}
\small\setlength{\tabcolsep}{4pt}
\begin{tabular}{@{}lcccc@{}}
\toprule
Model & Parameters & Obs. NLL $\downarrow$ & 10-step MSE $\downarrow$ & Return \\
\midrule
Native DreamerV3 offline & 8.85M & $\mathbf{-1.2990\pm0.0008}$ & $\mathbf{0.0024\pm0.00004}$ & $0.0734\pm0.0013$ \\
Occupancy & 3.26--3.29M & $-0.0861\pm0.4763$ & $0.1206\pm0.0244$ & failed \\
Gaussian & 3.26--3.29M & $1.3022\pm2.5705$ & $0.0542\pm0.0514$ & failed \\
\bottomrule
\end{tabular}
\end{wmtable}

\begin{figure}[htbp]
\centering
\includegraphics[width=\textwidth]{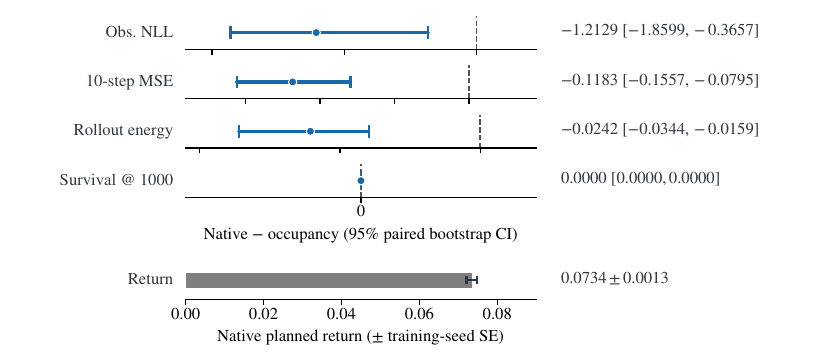}
\caption{\textbf{The native DreamerV3 reference is the stronger short-horizon predictor; survival
is tied at the shared evaluation horizon.} Native $-$ occupancy contrasts with nominal
95\% paired training-seed percentile-bootstrap intervals (4 priority tasks $\times$ 3 seeds,
horizon 1000): observation
NLL, 10-step MSE, and rollout energy all favor the native model with intervals excluding zero,
while survival is identical (1.0 vs.\ 1.0). Bar: native planned return $0.0734\pm0.0013$ with all
12 plan cells completed (a full-grid occupancy return is unavailable because of planner
failures, Fig.~\ref{fig:crash-census}). The native $10^5$-step marginal is unmeasured, and no
Gaussian-reference identity exists for its categorical state.}
\label{fig:native-panel}
\end{figure}
\begin{figure}[htbp]
\centering
\includegraphics[width=\textwidth]{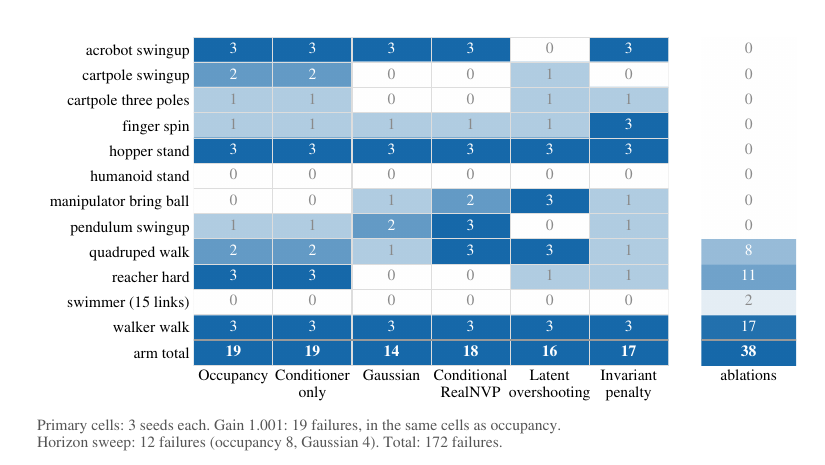}
\caption{\textbf{Planner completion varies across tasks and transition families.}
Each cell counts failures among three training seeds: the CEM planner has fewer finite candidates
than required elites. The margin strip counts $38$ ablation failures. The selected gain shares
the same $19$ failed task--seed cells as occupancy, and $12$ planning-sweep jobs fail,
giving $172$ failures among $348$ planning evaluations. Primary failure counts range from
$14/36$ for Gaussian to $19/36$ for occupancy and conditioner only.}
\label{fig:crash-census}
\end{figure}
\clearpage
\paragraph{Per-task results.}
Tables~\ref{tab:recurrent-task-summary}--\ref{tab:recurrent-sweep-summary} summarize primary
tasks, retrained variants, and planning horizons. The supplementary
\texttt{recurrent\_checkpoints.csv} retains all $324$ primary, selected-gain, and ablation
records; \texttt{recurrent\_planning\_sweep.csv} retains $48$ horizon records from $24$ sweep jobs.
The accompanying field definitions specify units and uncertainty. A \emph{failed} return
denotes a planner failure; a \emph{censored} gap denotes a rollout with escaped chains.
Completed-cell summaries use only measured returns.
\begingroup\raggedbottom
\begin{wmtable}{\textbf{Task-level recurrent outcomes.} Each cell lists mean survival at $10^5$ steps, mean marginal gap $D_{10^5}-D_{\rm ref}$, and completed planners out of three seeds (top to bottom). Gap means require all three seeds; cens. denotes chain censoring. The gap uses RMS pixel distance. Exact checkpoint scores and paired evaluation SEs: \texttt{recurrent\_checkpoints.csv}.}{tab:recurrent-task-summary}
\small\setlength{\tabcolsep}{3pt}
\begin{tabular}{@{}lrrrrrr@{}}
\toprule
Task & Occupancy & \shortstack{Conditioner\\only} & Gaussian & \shortstack{Conditional\\RealNVP} & \shortstack{Latent\\overshooting} & \shortstack{Invariant\\penalty} \\
\midrule
Acrobot swingup & \shortstack[r]{1.00000\\$-0.0015$\\0/3} & \shortstack[r]{1.00000\\$-0.0018$\\0/3} & \shortstack[r]{1.00000\\$-0.0011$\\0/3} & \shortstack[r]{0.66667\\\textit{cens.}\\0/3} & \shortstack[r]{1.00000\\$+0.0060$\\3/3} & \shortstack[r]{1.00000\\$-0.0017$\\0/3} \\
\addlinespace[3pt]
Cartpole swingup & \shortstack[r]{1.00000\\$-0.0023$\\1/3} & \shortstack[r]{1.00000\\$-0.0028$\\1/3} & \shortstack[r]{0.00000\\\textit{cens.}\\3/3} & \shortstack[r]{0.00045\\\textit{cens.}\\3/3} & \shortstack[r]{0.00000\\\textit{cens.}\\2/3} & \shortstack[r]{0.00590\\\textit{cens.}\\3/3} \\
\addlinespace[3pt]
Cartpole three poles & \shortstack[r]{1.00000\\$-0.0006$\\2/3} & \shortstack[r]{1.00000\\$-0.0004$\\2/3} & \shortstack[r]{0.04183\\\textit{cens.}\\3/3} & \shortstack[r]{0.00000\\\textit{cens.}\\3/3} & \shortstack[r]{0.04228\\\textit{cens.}\\2/3} & \shortstack[r]{0.33830\\\textit{cens.}\\2/3} \\
\addlinespace[3pt]
Finger spin & \shortstack[r]{1.00000\\$-0.0007$\\2/3} & \shortstack[r]{1.00000\\$-0.0006$\\2/3} & \shortstack[r]{0.33333\\\textit{cens.}\\2/3} & \shortstack[r]{0.62048\\\textit{cens.}\\2/3} & \shortstack[r]{0.00000\\\textit{cens.}\\2/3} & \shortstack[r]{0.61182\\\textit{cens.}\\0/3} \\
\addlinespace[3pt]
Hopper stand & \shortstack[r]{1.00000\\$+0.0005$\\0/3} & \shortstack[r]{1.00000\\$-0.0003$\\0/3} & \shortstack[r]{1.00000\\$+0.6653$\\0/3} & \shortstack[r]{0.66659\\\textit{cens.}\\0/3} & \shortstack[r]{0.33333\\\textit{cens.}\\0/3} & \shortstack[r]{0.66667\\\textit{cens.}\\0/3} \\
\addlinespace[3pt]
Humanoid stand & \shortstack[r]{1.00000\\$+0.0001$\\3/3} & \shortstack[r]{1.00000\\$+0.0001$\\3/3} & \shortstack[r]{0.66463\\\textit{cens.}\\3/3} & \shortstack[r]{0.00000\\\textit{cens.}\\3/3} & \shortstack[r]{0.00000\\\textit{cens.}\\3/3} & \shortstack[r]{0.33931\\\textit{cens.}\\3/3} \\
\addlinespace[3pt]
Manipulator bring ball & \shortstack[r]{1.00000\\$+0.0000$\\3/3} & \shortstack[r]{1.00000\\$+0.0001$\\3/3} & \shortstack[r]{0.33333\\\textit{cens.}\\2/3} & \shortstack[r]{0.09753\\\textit{cens.}\\1/3} & \shortstack[r]{0.00000\\\textit{cens.}\\0/3} & \shortstack[r]{0.33333\\\textit{cens.}\\2/3} \\
\addlinespace[3pt]
Pendulum swingup & \shortstack[r]{1.00000\\$-0.0010$\\2/3} & \shortstack[r]{1.00000\\$-0.0001$\\2/3} & \shortstack[r]{1.00000\\$-0.0027$\\1/3} & \shortstack[r]{0.66667\\\textit{cens.}\\0/3} & \shortstack[r]{1.00000\\$-0.0027$\\3/3} & \shortstack[r]{1.00000\\$-0.0027$\\2/3} \\
\addlinespace[3pt]
Quadruped walk & \shortstack[r]{1.00000\\$+0.0018$\\1/3} & \shortstack[r]{1.00000\\$+0.0011$\\1/3} & \shortstack[r]{0.55558\\\textit{cens.}\\2/3} & \shortstack[r]{0.02747\\\textit{cens.}\\0/3} & \shortstack[r]{0.33341\\\textit{cens.}\\0/3} & \shortstack[r]{0.00073\\\textit{cens.}\\2/3} \\
\addlinespace[3pt]
Reacher hard & \shortstack[r]{1.00000\\$+0.0010$\\0/3} & \shortstack[r]{1.00000\\$+0.0015$\\0/3} & \shortstack[r]{0.00000\\\textit{cens.}\\3/3} & \shortstack[r]{0.04008\\\textit{cens.}\\3/3} & \shortstack[r]{0.00000\\\textit{cens.}\\2/3} & \shortstack[r]{0.33333\\\textit{cens.}\\2/3} \\
\addlinespace[3pt]
Swimmer (15 links) & \shortstack[r]{1.00000\\$+0.0005$\\3/3} & \shortstack[r]{0.99915\\\textit{cens.}\\3/3} & \shortstack[r]{0.98295\\\textit{cens.}\\3/3} & \shortstack[r]{0.00000\\\textit{cens.}\\3/3} & \shortstack[r]{0.99996\\\textit{cens.}\\3/3} & \shortstack[r]{0.42175\\\textit{cens.}\\3/3} \\
\addlinespace[3pt]
Walker walk & \shortstack[r]{1.00000\\$+0.0001$\\0/3} & \shortstack[r]{1.00000\\$+0.0006$\\0/3} & \shortstack[r]{0.26685\\\textit{cens.}\\0/3} & \shortstack[r]{0.08728\\\textit{cens.}\\0/3} & \shortstack[r]{0.34717\\\textit{cens.}\\0/3} & \shortstack[r]{0.22815\\\textit{cens.}\\0/3} \\
\bottomrule
\end{tabular}
\end{wmtable}

\begin{wmtable}{\textbf{Retrained architectural variants.} Four tasks (Walker, Reacher, Swimmer, Quadruped), three seeds each. All variants retain every chain; gaps are equal-task means $\pm$ training-seed SE. Full NLL, MSE, gap, and return records: \texttt{recurrent\_checkpoints.csv}; aggregate SEs: \texttt{recurrent\_method\_summary.csv}.}{tab:recurrent-ablation-summary}
\small\setlength{\tabcolsep}{8pt}
\begin{tabular}{@{}lccc@{}}
\toprule
Variant & Surviving cells & Marginal gap & Completed planners \\
\midrule
$c_\star=0.99$ & 12/12 & $+0.00044\pm0.00065$ & 4/12 \\
$c_\star=0.9999$ & 12/12 & $+0.00047\pm0.00071$ & 4/12 \\
$r=0$ & 12/12 & $-0.00025\pm0.00035$ & 8/12 \\
$r=0.01$ & 12/12 & $+0.00042\pm0.00075$ & 3/12 \\
$d=512$ & 12/12 & $-0.00052\pm0.00055$ & 6/12 \\
$d=1024$ & 12/12 & $+0.00035\pm0.00032$ & 9/12 \\
\bottomrule
\end{tabular}
\end{wmtable}

\begin{wmtable}{\textbf{Planning-horizon sweep.} Counts are completed seeds out of three; a mean return is shown only when all three complete. Returns use the planned physics-step denominator. The 48 seed--horizon rows in \texttt{recurrent\_planning\_sweep.csv} represent 24 sweep jobs. Twelve jobs fail during $H=5$; their $H=50$ evaluations are unattempted.}{tab:recurrent-sweep-summary}
\small\setlength{\tabcolsep}{5pt}
\begin{tabular}{@{}llrrrr@{}}
\toprule
& & \multicolumn{2}{c}{Completed} & \multicolumn{2}{c}{Mean return} \\
\cmidrule(lr){3-4}\cmidrule(lr){5-6}
Task & Model & $H=5$ & $H=50$ & $H=5$ & $H=50$ \\
\midrule
Quadruped walk & Occupancy & 1/3 & 1/3 & --- & --- \\
Quadruped walk & Gaussian & 2/3 & 2/3 & --- & --- \\
Reacher hard & Occupancy & 0/3 & 0/3 & --- & --- \\
Reacher hard & Gaussian & 3/3 & 3/3 & 0.0293 & 0.0065 \\
Swimmer (15 links) & Occupancy & 3/3 & 3/3 & 0.1227 & 0.1174 \\
Swimmer (15 links) & Gaussian & 3/3 & 3/3 & 0.1184 & 0.1186 \\
Walker walk & Occupancy & 0/3 & 0/3 & --- & --- \\
Walker walk & Gaussian & 0/3 & 0/3 & --- & --- \\
\bottomrule
\end{tabular}
\end{wmtable}

\endgroup

\subsubsection{Scope of the recurrent comparisons}\label{app:limitations}
The comparisons condition on the collected data, eligible tasks, fixed fitted behaviour policy,
three paired seeds, and reported budgets. Because the matching gates reject every alternative
transition family, their contrasts combine architecture with achieved fit; the selected gain
isolates output scaling. Planning contrasts use common completed cells, with completion counts
reported alongside returns. The prior certificate applies to the specified initialization bank;
for second moment $6.9\times10^9$, its total-variation cap becomes informative near $10^4$ steps.

\newcommand{\CPUStaticAuditTable}{%
\begin{wmtable}{Independent static resampling for the capped Reacher checkpoint. Eight fresh Gaussian clouds of 2,048 points are averaged within each of eight existing reference pools. Mean $\pm$ sample SE uses the eight reference-level values. Recorded rollout samples are reused; Appendix~\ref{app:reacher-replication} reports the independent rollout replication.}{tab:cpu-static}
\small\setlength{\tabcolsep}{5pt}
\begin{tabular}{@{}lr@{}}
\toprule
Quantity & Energy distance \\
\midrule
Fresh static distance & $0.30264\pm0.00141$ \\
Stored static minus fresh static & $0.00595\pm0.00235$ \\
Stored data-start rollout minus fresh static & $-0.00199\pm0.00190$ \\
Stored reference-start rollout minus fresh static & $-0.00041\pm0.00250$ \\
\bottomrule
\end{tabular}
\end{wmtable}
}
\newcommand{\CPUGuardAuditTable}{%
\begin{wmtable}{Post-hoc reset-mixture rescoring on all held-out pairs: Reacher 479,968; Swimmer 247,968. Costs are nats per latent dimension ($d=256$). The common reset probability is $r=10^{-3}$. Ref.\ dom.\ is the percentage of pairs for which the reference component has posterior mixture responsibility above one half; it is not the reset sampling frequency. These CPU evaluations include no retraining or new rollout.}{tab:cpu-guard}
\small\setlength{\tabcolsep}{5pt}
\begin{tabular}{@{}lrrrrrr@{}}
\toprule
Checkpoint & Mean & Median & 99.9\% & Maximum & \% $>10$ & Ref.\ dom.\ \% \\
\midrule
Reacher, default & $1.118\!\times\!10^{4}$ & $0.8509$ & $1.211\!\times\!10^{4}$ & $1.415\!\times\!10^{8}$ & $0.3665$ & --- \\
Reacher, default $+r$ & $0.8547$ & $0.8433$ & $2.9334$ & $4.6512$ & $0.0000$ & $15.82$ \\
Reacher, cap .999 & $0.8389$ & $0.8300$ & $2.5030$ & $192.7131$ & $0.0002$ & --- \\
Reacher, cap .999 $+r$ & $0.8328$ & $0.8265$ & $2.4403$ & $4.0283$ & $0.0000$ & $8.76$ \\
Swimmer, default & $2.674\!\times\!10^{7}$ & $1.433\!\times\!10^{7}$ & $1.887\!\times\!10^{8}$ & $2.942\!\times\!10^{8}$ & $97.1081$ & --- \\
Swimmer, default $+r$ & $6.6779$ & $6.7025$ & $16.3152$ & $22.0964$ & $12.6641$ & $99.15$ \\
Swimmer, cap .999 & $314.0148$ & $294.4076$ & $1175.8230$ & $1987.7783$ & $94.7945$ & --- \\
Swimmer, cap .999 $+r$ & $6.7789$ & $6.7911$ & $16.4691$ & $22.0765$ & $14.7499$ & $98.26$ \\
\bottomrule
\end{tabular}
\end{wmtable}
}
\newcommand{\CPUAnalyticBoundTable}{%
\begin{wmtable}{Finite-start bounds in a two-dimensional OU specialization. The initial law is uniform on $\{(-3,0),(3,0)\}$, so its initial KL is infinite. The update coefficient is $0.75$ and the architectural cap used in the bound is $0.8$; no reset is used. KL and $L^1$-TV are computed by $220\times220$ Gauss--Legendre quadrature on $[-12,12]^2$. These are numerical integrals of an analytic mixture law, not high-dimensional empirical TV estimates.}{tab:cpu-ou-bound}
\small\setlength{\tabcolsep}{5pt}
\begin{tabular}{@{}rrrrr@{}}
\toprule
Horizon & KL integral & KL upper bound & $L^1$-TV integral & $L^1$-TV upper bound \\
\midrule
1 & $2.1033$ & $3.2617$ & $1.5094$ & $2.0000$ \\
2 & $0.8497$ & $2.0875$ & $0.9991$ & $2.0000$ \\
5 & $0.0418$ & $0.5472$ & $0.2119$ & $1.0461$ \\
10 & $0.0002$ & $0.0588$ & $0.0123$ & $0.3428$ \\
\bottomrule
\end{tabular}
\end{wmtable}
}
\newcommand{\CPUScalingTable}{%
\begin{wmtable}{CPU timing of the parallel prior, with four CPU threads, float32, batch size 32, six action coordinates, hidden width 128, four joint layers and two post layers. Times are medians of five repetitions. The measurement excludes the optimizer, data loading, CNN encoder/decoder, and simulator. Parameter memory is not peak process memory; these measurements do not predict GPU throughput or trained-model accuracy.}{tab:cpu-scaling}
\small\setlength{\tabcolsep}{5pt}
\begin{tabular}{@{}rrrrr@{}}
\toprule
State width & Parameters & Parameter MiB & Forward/backward ms & Sampling ms \\
\midrule
256 & 302,856 & $1.16$ & $7.63$ & $2.50$ \\
512 & 484,360 & $1.85$ & $8.84$ & $3.24$ \\
1024 & 847,368 & $3.23$ & $13.40$ & $4.57$ \\
\bottomrule
\end{tabular}
\end{wmtable}
}

\section{Lorenz calibration and competing stabilization objectives}
\label{app:seeds}

A fitted encoder determines its own decoded-reference distribution. The first
comparison asks whether that static score predicts the same model's long-run
distributional error. It does not assume that all encoders have the same
plateau or that a predictable plateau is small. Throughout the appendix,
$D_\infty$ is the static decoded-reference score of
Section~\ref{sec:certificates}; the recurrent-study tables print the same
quantity as $D_{\rm ref}$.

\begin{wmtable}{\textbf{Lorenz prediction--rollout comparison by training seed.} Static and terminal energy distances use mean $\pm$ Monte Carlo SE from eight repetitions, with sample SD. The standardized gap combines these SEs in quadrature; it is nominal because the saved summaries do not retain their covariance. The final column is an energy distance in Gaussian coordinates, not a TV bound.}{tab:seeds}
\small\setlength{\tabcolsep}{5pt}
\begin{tabular}{@{}rrrrr@{}}
\toprule
Seed & Static $D_\infty$ & Rollout $D_{10^5}$ & Nominal gap/SE & $D(\hat\mu^z_{10^5},\gauss)$ \\
\midrule
$0$ & $0.1963\pm0.0098$ & $0.1849\pm0.0087$ & $-0.87$ & $0.00215$ \\
$1$ & $0.2228\pm0.0133$ & $0.2076\pm0.0112$ & $-0.87$ & $0.00069$ \\
$2$ & $0.1548\pm0.0087$ & $0.1593\pm0.0084$ & $+0.37$ & $0.00111$ \\
\bottomrule
\end{tabular}
\end{wmtable}
\begin{wmtable}{\textbf{Lorenz rollout fidelity, survival, and scale.} Three training seeds and seven sampled horizons. $D_T$ is energy distance to data; $f_T$ is the fraction of finite chains; $m_T=d^{-1}\Exp\|z_T\|^2$ uses finite chains. All stationary-model chains are finite. Unconstrained distances are conditional on survival, and extreme values can be numerically unreliable.}{tab:seedsfree}
\small\setlength{\tabcolsep}{5pt}
\begin{tabular}{@{}rrrrrrrr@{}}
\toprule
& & \multicolumn{3}{c}{Stationary} & \multicolumn{3}{c}{Unconstrained} \\
\cmidrule(lr){3-5}\cmidrule(lr){6-8}
Seed & $T$ & $D_T$ & $f_T$ & $m_T$ & $D_T$ & $f_T$ & $m_T$ \\
\midrule
$0$ & $0$ & $-0.0018$ & $1.0000$ & $0.8057$ & $-0.0017$ & $1.0000$ & $1.0110$ \\
 & $1$ & $-0.0007$ & $1.0000$ & $0.8191$ & $-0.0022$ & $1.0000$ & $1.0161$ \\
 & $10$ & $0.0616$ & $1.0000$ & $0.8712$ & $-0.0025$ & $1.0000$ & $1.0048$ \\
 & $100$ & $0.1789$ & $1.0000$ & $0.9926$ & $9.25{\times}10^{17}$ & $0.9984$ & --- \\
 & $1000$ & $0.1851$ & $1.0000$ & $1.0035$ & $1.17{\times}10^{10}$ & $0.9801$ & --- \\
 & $10000$ & $0.1775$ & $1.0000$ & $1.0027$ & $1.07{\times}10^{29}$ & $0.8254$ & --- \\
 & $100000$ & $0.1849$ & $1.0000$ & $0.9965$ & $-0.0035$ & $0.1395$ & $2.04{\times}10^{9}$ \\
\midrule
$1$ & $0$ & $0.0034$ & $1.0000$ & $0.8566$ & $0.0032$ & $1.0000$ & $1.0137$ \\
 & $1$ & $0.0071$ & $1.0000$ & $0.8694$ & $0.0025$ & $1.0000$ & $1.0125$ \\
 & $10$ & $0.1740$ & $1.0000$ & $0.9535$ & $-0.0030$ & $1.0000$ & $1.0090$ \\
 & $100$ & $0.2129$ & $1.0000$ & $1.0271$ & $0.0025$ & $0.9999$ & $1.0069$ \\
 & $1000$ & $0.2158$ & $1.0000$ & $1.0228$ & $0.0077$ & $0.9994$ & $1.0031$ \\
 & $10000$ & $0.2113$ & $1.0000$ & $1.0252$ & $0.0047$ & $0.9932$ & $1.0136$ \\
 & $100000$ & $0.2076$ & $1.0000$ & $1.0322$ & $0.0026$ & $0.9394$ & $1.0172$ \\
\midrule
$2$ & $0$ & $0.0012$ & $1.0000$ & $0.6588$ & $0.0012$ & $1.0000$ & $1.0096$ \\
 & $1$ & $0.0028$ & $1.0000$ & $0.6685$ & $0.0008$ & $1.0000$ & $1.0061$ \\
 & $10$ & $0.0538$ & $1.0000$ & $0.7216$ & $-0.0007$ & $1.0000$ & $1.0055$ \\
 & $100$ & $0.1538$ & $1.0000$ & $0.9676$ & $3.68{\times}10^{13}$ & $0.9993$ & --- \\
 & $1000$ & $0.1634$ & $1.0000$ & $1.0113$ & $1.11{\times}10^{9}$ & $0.9925$ & --- \\
 & $10000$ & $0.1588$ & $1.0000$ & $1.0040$ & $3.54{\times}10^{9}$ & $0.9236$ & --- \\
 & $100000$ & $0.1593$ & $1.0000$ & $1.0100$ & $-0.0020$ & $0.4602$ & $9.53{\times}10^{15}$ \\
\bottomrule
\end{tabular}
\end{wmtable}
\begin{wmtable}{\textbf{Static reference scores of the unconstrained Lorenz models.} Mean $\pm$ Monte Carlo SE over eight static draws. These scores compare decoded Gaussian samples with data; the unconstrained rollout need not converge to that reference law.}{tab:seedsstatic}
\small\begin{tabular}{@{}rr@{}}
\toprule
Seed & $D(E^{-1}_\#\gauss,p_{\mathrm{data}})$ \\
\midrule
$0$ & $0.0464\pm0.0030$ \\
$1$ & $0.0308\pm0.0034$ \\
$2$ & $0.0226\pm0.0042$ \\
\bottomrule
\end{tabular}
\end{wmtable}

The three stationary models have different static distances, and each terminal
score is within one nominal combined Monte Carlo SE of its own prediction.
The unconstrained models have lower static reference scores but retain
$13.95\%$, $93.94\%$, and $46.02\%$ of their chains at $10^5$ steps.
Their static decoded-Gaussian scores are not predictions of their limiting
laws. In particular, near-zero survivor-conditioned distances at the last
horizon coexist with numerical failure and, in two seeds, extreme moments.
These comparisons do not establish that an unconstrained kernel lacks an
invariant distribution.

\subsection{Reweighting the marginal fit}
\label{app:wsweep}

This comparison tests whether improving the fitted marginal changes the rollout plateau.
The four constrained settings share data, schedule, and a 30,000-update budget;
the $w=\infty$ setting fits the marginal first and then fits the conditional
with the chart frozen. Two unconstrained conditional-flow arms use $w=1,10$.
Each static score is evaluated from the same fitted checkpoint as its rollout.

\begin{wmtable}{\textbf{Complete marginal-weight sweep.} Finite weights share a 30,000-update budget; $w=\infty$ uses marginal-first training. NLLs are nats per observation dimension. Static $D_\infty$ uses decoded Gaussian samples for every arm, without predicting an unconstrained limit. Distance SEs use eight Monte Carlo repetitions and sample SD. Terminal distances are omitted when chains are lost.}{tab:wsweep}
\small\setlength{\tabcolsep}{4pt}
\begin{tabular}{@{}llrrrrrr@{}}
\toprule
& & \multicolumn{3}{c}{Held-out NLL} & & & \\
\cmidrule(lr){3-5}
Class & $w$ & Pair & Marginal & Conditional & Static $D_\infty$ & $D_{10^5}$ & Finite \\
\midrule
Stationary & $1$ & $2.2679$ & $2.6477$ & $-0.3799$ & $0.1126\pm0.0098$ & $0.1066\pm0.0070$ & $1.0000$ \\
 & $3$ & $2.3891$ & $2.3382$ & $+0.0509$ & $0.0809\pm0.0056$ & $0.0881\pm0.0107$ & $1.0000$ \\
 & $10$ & $2.4548$ & $2.2713$ & $+0.1835$ & $0.0165\pm0.0042$ & $0.0162\pm0.0035$ & $1.0000$ \\
 & $\infty$ & $3.5045$ & $2.2582$ & $+1.2463$ & $0.0072\pm0.0026$ & $0.0028\pm0.0036$ & $1.0000$ \\
\midrule
Unconstrained & $1$ & $1.4906$ & $2.3240$ & $-0.8334$ & $0.0532\pm0.0031$ & --- & $0.0599$ \\
 & $10$ & $1.4511$ & $2.2648$ & $-0.8137$ & $0.0057\pm0.0025$ & --- & $0.0641$ \\
\bottomrule
\end{tabular}
\end{wmtable}

For constrained $w=1$ to $10$, held-out pair NLL increases by $0.1869$ nats per
dimension while the static distance decreases. The corresponding unconstrained
NLL changes by $-0.0395$. These are fitted, finite-budget outcomes; an ordering
of globally optimized objectives need not order their held-out scores.
Both unconstrained arms retain only about $6\%$ of chains at $10^5$.
The negative distance estimate $-2.26$ is therefore not evidence of better
full-distribution fidelity. The smallest constrained plateau is unresolved at
the experiment's sampling resolution.

\subsection{Soft stabilization objectives}
\label{app:stabilizers}

This comparison measures fidelity and survival under three soft stabilization objectives.
The three penalties share an unconstrained transition and use weights
$\{0.01,0.1,1\}$, 8,000 Adam updates, and matched initializations at seed zero.
They add an eight-step Huber rollout loss, a Jacobian-direction penalty, or a
four-step differentiable energy-distance penalty to total three-dimensional
pair NLL. The pair NLL entries below are consequently nats per pair, not nats
per dimension. The stationary arm uses the same encoder specification and
update budget with a different transition family.

Huber targets are detached encoder states; coordinates in the training penalty
are clamped to $[-10^3,10^3]$. The Jacobian probe uses two normalized
transpose-Jacobian iterations on a random direction. It is neither an exact
spectral norm nor an operator-contraction certificate. The measure penalty
uses rollout samples and independent real states, following the motivation of
invariant-measure regularization \citep{schiff2024dyslim}.
Checkpoints minimize recorded validation NLL; the seed table states the
unequal evaluation intervals for seed zero and seeds one--two.

\begin{wmtable}{\textbf{Long-horizon regularization on Lorenz-63.} Seed 0, 8,000 updates per configuration. The three penalties share the unconstrained architecture; \textsc{ours} uses a different transition family. $D_\infty$ is the static decoded-Gaussian reference score, not a predicted limit for the unconstrained models. Rollout distances use mean $\pm$ Monte Carlo SE over eight repetitions and are displayed only when every chain is finite.}{tab:stab}
\small\setlength{\tabcolsep}{5pt}
\begin{tabular}{@{}llrrrr@{}}
\toprule
Objective & $w$ & Pair NLL & $D_\infty$ & Finite at $10^5$ & $D_{10^5}$ \\
\midrule
\textsc{free} & --- & $5.604$ & $0.1108$ & $0.00027$ & --- \\
\midrule
Multistep & 0.01 & $5.836$ & $0.1334$ & $0.13644$ & --- \\
 & 0.1 & $6.253$ & $0.0997$ & $1.00000$ & $0.6776\pm0.0200$ \\
 & 1 & $5.840$ & $0.1572$ & $0.95813$ & --- \\
\midrule
Jacobian probe & 0.01 & $5.391$ & $0.1055$ & $0.00003$ & --- \\
 & 0.1 & $5.599$ & $0.0910$ & $0.00000$ & --- \\
 & 1 & $5.498$ & $0.1551$ & $0.00104$ & --- \\
\midrule
Measure penalty & 0.01 & $6.410$ & $0.2089$ & $0.00000$ & --- \\
 & 0.1 & $7.574$ & $0.2170$ & $1.00000$ & $0.1058\pm0.0142$ \\
 & 1 & $6.438$ & $0.1594$ & $0.43649$ & --- \\
\midrule
\textsc{ours} & --- & $8.003$ & $0.2039$ & $1.00000$ & $0.2125\pm0.0091$ \\
\bottomrule
\end{tabular}
\end{wmtable}
\begin{wmtable}{\textbf{Seed-specific outcomes for the selected multistep weight.} All runs use $w=1$ for multistep and 8,000 updates. Checkpoints minimize validation NLL over the recorded evaluations: every 1,000 updates for seed 0 and every 4,000 for seeds 1--2, plus the final update. Survivor-conditioned distances are included descriptively without Monte Carlo SE; they do not score the complete rollout law.}{tab:stabseeds}
\small\setlength{\tabcolsep}{6pt}
\begin{tabular}{@{}lrrrrr@{}}
\toprule
Arm & Seed & Finite & $D_\infty$ & $D_{10^5}$ & $D_{10^5}/D_\infty$ \\
\midrule
Multistep & $0$ & $0.95813$ & $0.1572$ & $0.0224$ & $0.143$ \\
 & $1$ & $0.99991$ & $0.2271$ & $0.0282$ & $0.124$ \\
 & $2$ & $0.31396$ & $0.1342$ & --- & --- \\
\midrule
\textsc{ours} & $0$ & $1.00000$ & $0.2039$ & $0.2125$ & $1.042$ \\
 & $1$ & $1.00000$ & $0.1380$ & $0.1444$ & $1.046$ \\
 & $2$ & $1.00000$ & $0.0924$ & $0.0836$ & $0.905$ \\
\bottomrule
\end{tabular}
\end{wmtable}

Only two of the nine penalized seed-zero configurations retain at least
$99\%$ of chains. Both retain every chain: multistep weight $0.1$ has distance
$0.6776$, and measure-penalty weight $0.1$ has distance $0.1058$, versus
$0.2125$ for the stationary model. The latter is a direct example of a soft
penalty attaining better long-run fidelity. Its static reference score does
not become a predicted limit by virtue of that success. At the selected
multistep weight one, only one of three seeds retains at least $99\%$ of chains;
the stationary model retains every chain in all three. These results identify
a survival--fidelity trade-off under the recorded budgets, without ruling out
successful stabilization by a different regularizer or tuning procedure.

\paragraph{Time-varying negative control.}\label{app:control}
On the ramped Lorenz system, stationary-model pair NLL is $4.7793$ versus
$3.5166$ for the unconstrained model, in nats per dimension. Its rollout
distance reaches $0.841$ at $10^3$ steps. The unconstrained model also fails:
its second moment is $1.6\times10^4$ by step ten and becomes nonfinite later.
The reference two-half distance is only $0.0052$, showing that this simple
finite-data diagnostic misses the imposed nonstationarity. A model's invariant
marginal need not track a changing data marginal.

\section{Fixed-latent pixel models: marginal predictability and conditional failure}
\label{app:dmcwm}\label{app:breadth}

Each task has a separately trained and frozen convolutional autoencoder.
All transition families use its same 256-dimensional latent, Gaussian smoothing
scale $0.1$, decoder, and data split. Walker's held-out reconstruction PSNR is
$30.25$ dB. These are transition comparisons in a fixed representation, not
end-to-end recurrent pixel-world-model comparisons. Learned invertible state
charts, where present, are trained separately for each transition model.

The data contain 64 training and 32 held-out uninterrupted trajectories,
yielding 959,936 and 479,968 pairs before burn-in. Burn-in is applied to both
training pairs and reference observations. For Walker the latent screen selects
6,000 observation steps, leaving 575,936 training and 287,968 held-out pairs.
Its matched 250-frame window baseline is $0.05408\pm0.01691$, with threshold
$0.11908$, half-to-half distance $0.03223$, estimated ESS 1,141, spread ratio
$0.986$, and no frozen coordinate. Observation-step burn-in differs from the
physical-state screen's simulator-step units. These diagnostics support the
retained finite window; they do not verify a population invariant law.

\begin{wmtable}{Walker transition budgets and held-out NLL, in nats per frozen-latent dimension. Parameter counts exclude the common autoencoder. A dash denotes an unevaluated or inapplicable likelihood component. Source: the recorded Walker comparison.}{tab:e11arms}
\small\setlength{\tabcolsep}{6pt}
\begin{tabular}{@{}lrrrrr@{}}
\toprule
& \multicolumn{2}{c}{Parameters (M)} & \multicolumn{3}{c}{NLL (nats/dim)} \\
\cmidrule(lr){2-3}\cmidrule(lr){4-6}
Model & Coordinates & Transition & Marginal & Conditional & Pair \\
\midrule
Stationary & $4.20$ & $7.20$ & $-0.0307$ & $-0.0732$ & $-0.1039$ \\
Free flow & $4.20$ & $5.52$ & $-0.0256$ & $-0.1549$ & $-0.1805$ \\
Gaussian with chart & $4.20$ & $8.37$ & $+0.1390$ & $+0.1197$ & $+0.2587$ \\
Gaussian raw & --- & $8.37$ & --- & $+0.4717$ & $+0.4717$ \\
Flow matching & --- & $8.47$ & --- & --- & --- \\
\bottomrule
\end{tabular}
\end{wmtable}

The Walker comparison uses transition heads with $5.52$--$8.47$M parameters;
capacity and optimization differ across families, including conditional flow matching
\citep{lipman2023flow}. Native recurrent-model comparisons appear in
Appendix~\ref{app:native-baseline}.
The twelve-task comparison includes the stationary transition, raw Gaussian,
and flow matching; the other two chart-based arms are additional Walker controls.

\subsection{Complete twelve-task outcomes}
\label{app:fixed-latent-results}

The task-wide comparison separates prediction of each model's own limiting
distance from accuracy relative to the data.

\begin{wmtable}{Static and terminal latent energy distances on all twelve pixel tasks. Levels use the same eight paired Monte Carlo repetitions of each frozen checkpoint. The paired difference is terminal minus static, with sample SE. These uncertainties do not measure training-seed or collection variation. Every constrained chain remains finite.}{tab:tier2}
\small\setlength{\tabcolsep}{7pt}
\begin{tabular}{@{}lrrr@{}}
\toprule
Task & $D_\infty$ & $D_{10^5}$ & Paired $\Delta\pm\mathrm{SE}$ \\
\midrule
Acrobot swingup & $0.143$ & $0.144$ & $+0.0008\pm0.0040$ \\
Cartpole swingup & $0.212$ & $0.212$ & $-0.0005\pm0.0024$ \\
Cartpole three poles & $0.116$ & $0.116$ & $+0.0002\pm0.0016$ \\
Finger spin & $0.169$ & $0.175$ & $+0.0065\pm0.0062$ \\
Hopper stand & $0.193$ & $0.193$ & $+0.0007\pm0.0038$ \\
Humanoid stand & $0.023$ & $0.023$ & $+0.0008\pm0.0016$ \\
Manipulator bring ball & $0.224$ & $0.229$ & $+0.0049\pm0.0036$ \\
Pendulum swingup & $0.082$ & $0.086$ & $+0.0037\pm0.0050$ \\
Quadruped walk & $0.114$ & $0.114$ & $-0.0009\pm0.0042$ \\
Reacher hard & $0.305$ & $0.299$ & $-0.0061\pm0.0038$ \\
Swimmer (15 links) & $1.100$ & $1.101$ & $+0.0010\pm0.0132$ \\
Walker walk & $0.160$ & $0.156$ & $-0.0036\pm0.0022$ \\
\bottomrule
\end{tabular}
\end{wmtable}

\begin{wmtable}{\textbf{Twelve-task survival and pixel fidelity at $T=10^5$.} Energy estimates use Euclidean distance on flattened, downsampled $16\times16$ decoder outputs with intensities in $[0,1]$; this ground distance is not RMS-normalized. The flow-matching/ours ratio is reported only when both methods retain every chain. The paired static-versus-rollout comparison is given separately in Table~\ref{tab:tier2}.}{tab:e11e}
\small\setlength{\tabcolsep}{5pt}
\begin{tabular}{@{}lrrrrrr@{}}
\toprule
& \multicolumn{3}{c}{Finite-chain fraction} & \multicolumn{2}{c}{Pixel distance} & \\
\cmidrule(lr){2-4}\cmidrule(lr){5-6}
Task & Ours & Gauss. raw & FM & Ours & FM & FM/ours \\
\midrule
Acrobot swingup & $1.00$ & $0.00$ & $1.00$ & $0.0106$ & $0.0031$ & $0.30$ \\
Cartpole swingup & $1.00$ & $0.00$ & $1.00$ & $0.0159$ & $0.0010$ & $0.06$ \\
Cartpole three poles & $1.00$ & $0.00$ & $1.00$ & $0.0022$ & $0.0141$ & $6.5$ \\
Finger spin & $1.00$ & $0.74$ & $1.00$ & $0.0153$ & $0.0013$ & $0.08$ \\
Hopper stand & $1.00$ & $0.95$ & $1.00$ & $0.0105$ & $0.0007$ & $0.07$ \\
Humanoid stand & $1.00$ & $0.00$ & $1.00$ & $0.0040$ & $0.0084$ & $2.1$ \\
Manipulator bring ball & $1.00$ & $0.00$ & $0.43$ & $0.0070$ & $0.5780$ & --- \\
Pendulum swingup & $1.00$ & $0.00$ & $1.00$ & $0.0028$ & $0.0002$ & $0.09$ \\
Quadruped walk & $1.00$ & $0.00$ & $1.00$ & $0.0154$ & $0.0831$ & $5.4$ \\
Reacher hard & $1.00$ & $0.98$ & $0.23$ & $0.0168$ & $0.4449$ & --- \\
Swimmer (15 links) & $1.00$ & $1.00$ & $1.00$ & $0.6031$ & $1.6668$ & $2.8$ \\
Walker walk & $1.00$ & $0.00$ & $1.00$ & $0.0277$ & $2.2742$ & $82$ \\
\bottomrule
\end{tabular}
\end{wmtable}
\noindent swimmer\_swimmer15 carries the metastability flag from the occupancy screen and is retained with that qualification.

The stationary models retain every chain on all twelve tasks, but their pixel
fidelity is not uniformly best. Of ten tasks on which flow matching also
retains every chain, five favour flow matching and five favour the stationary
model. The other two flow-matching scores condition on partial survival.
Swimmer's large stationary distance and metastability qualification remain
part of the result. Static-versus-rollout agreement in Table~\ref{tab:tier2}
therefore answers a different question from which model has the smallest
pixel distance.

\paragraph{Sampling conventions.}\label{app:tier2baselines}
Rollouts reproduce the external OU driver's coefficients, clipping, and
interleaved actions, initializing its state from the stationary driver law.
Conditional flow matching uses 16 Euler steps. Its one-step 8-versus-16-step
sensitivity reaches $14.6\%$ of observed displacement on Swimmer, exceeding
the $10\%$ sensitivity heuristic; the reported scores use 16 Euler steps.
The compact Walker comparison below separates stochastic samples from mean
steps, which define different rollout laws.

\begin{wmtable}{Walker decoded-pixel distance and numerical survival at $10^5$ steps. Entries summarize the recorded full trajectory grid; distances for partially surviving models are conditional on survival. A dash means no surviving sample. Finite-sample uncertainty for the paired stationary plateau is reported separately in Table~\ref{tab:tier2}.}{tab:e11rollout}
\small\setlength{\tabcolsep}{7pt}
\begin{tabular}{@{}llrr@{}}
\toprule
Transition & Rollout convention & Finite fraction & Pixel distance \\
\midrule
Stationary & Sampling & $1.0000$ & $0.02770$ \\
Free flow & Sampling & $0.6030$ & $0.03458$ \\
Gaussian with chart & Sampling & $0$ & --- \\
Gaussian with chart & Mean step & $0$ & --- \\
Gaussian raw & Sampling & $0$ & --- \\
Gaussian raw & Mean step & $0$ & --- \\
Flow matching & Sampling & $1.0000$ & $2.27422$ \\
Flow matching & Mean step & $1.0000$ & $2.26613$ \\
\bottomrule
\end{tabular}
\end{wmtable}

At $10^3$ steps, Walker mean-step flow matching and the stationary model have
similar distances ($0.02738$ and $0.02778$). By $10^5$, both flow-matching
conventions reach approximately $2.27$. The sampled flow-matching second
moment is $1.06$, closer to the data's $1.01$ than the stationary model's
$0.80$ despite its worse distributional distance: a matched moment alone does
not identify the more faithful distribution.

\subsection{Typical conditional fit and extreme losses}
\label{app:conditional-tails}

The conditional-loss distribution distinguishes a broad fit failure from a rare extreme tail.

\paragraph{Conditional-likelihood distributions.} Held-out pair costs distinguish poor fit on typical transitions from a small number of extreme losses. The tables report the empirical mean, median, maximum, and upper-tail fraction directly. Large observed losses do not establish infinite population moments. Pair-level standard errors would additionally require accounting for temporal dependence.
\begin{wmtable}{\textbf{Conditional NLL across held-out pairs.} All costs are nats per latent dimension. The tail fraction is the empirical proportion exceeding ten nats per dimension. Each row compares the same held-out pairs under the constrained model and the raw-latent Gaussian head.}{tab:e11f}
\small\setlength{\tabcolsep}{5pt}
\begin{tabular}{@{}lrrrrr@{}}
\toprule
& \multicolumn{4}{c}{Ours} & Gaussian raw \\
\cmidrule(lr){2-5}
Task & Median & Mean & Maximum & Fraction $>10$ & Median \\
\midrule
Acrobot swingup & $-0.5204$ & $-0.4951$ & $+4.19 \times 10^{2}$ & $1 \times 10^{-5}$ & $+0.6007$ \\
Cartpole swingup & $-0.4823$ & $+33.4751$ & $+8.43 \times 10^{6}$ & $3 \times 10^{-4}$ & $-0.0620$ \\
Cartpole three poles & $+0.4501$ & $+0.4126$ & $+1.26 \times 10^{2}$ & $6 \times 10^{-6}$ & $+0.9424$ \\
Finger spin & $-0.6457$ & $+0.5314$ & $+3.88 \times 10^{5}$ & $2 \times 10^{-5}$ & $-0.5393$ \\
Hopper stand & $-0.2634$ & $-0.2357$ & $+3.03 \times 10^{2}$ & $4 \times 10^{-5}$ & $+0.1258$ \\
Humanoid stand & $+0.4159$ & $+0.4170$ & $+3.2921$ & $0$ & $+0.9886$ \\
Manipulator bring ball & $+0.4458$ & $+1.61 \times 10^{3}$ & $+2.58 \times 10^{7}$ & $6 \times 10^{-3}$ & $+0.5640$ \\
Pendulum swingup & $-0.7428$ & $+2.84 \times 10^{3}$ & $+3.00 \times 10^{8}$ & $2 \times 10^{-5}$ & $-0.3783$ \\
Quadruped walk & $-0.3659$ & $+4.07 \times 10^{3}$ & $+6.50 \times 10^{7}$ & $8 \times 10^{-3}$ & $+0.4521$ \\
Reacher hard & $+0.8509$ & $+1.11 \times 10^{4}$ & $+1.42 \times 10^{8}$ & $4 \times 10^{-3}$ & $+0.4041$ \\
Swimmer (15 links) & $+1.43 \times 10^{7}$ & $+2.67 \times 10^{7}$ & $+2.94 \times 10^{8}$ & $0.97$ & $+0.2616$ \\
Walker walk & $-0.0797$ & $-0.0732$ & $+1.4903$ & $0$ & $+0.4798$ \\
\bottomrule
\end{tabular}
\end{wmtable}
\begin{wmtable}{\textbf{Sensitivity to extreme pair losses.} The trimmed mean removes the largest 0.1\% of costs and therefore changes the estimand. The 90\% column gives the fraction of pairs whose largest costs account for 90\% of the total, when defined. The final column is the Gaussian head\textquotesingle s median cost on the constrained model\textquotesingle s worst 64 pairs.}{tab:e11ftail}
\small\setlength{\tabcolsep}{6pt}
\begin{tabular}{@{}lrrrr@{}}
\toprule
Task & Trimmed mean & Fraction for 90\% & Marginal NLL & Gaussian on worst \\
\midrule
Acrobot swingup & $-0.4977$ & --- & $-0.4885$ & $+0.8250$ \\
Cartpole swingup & $-0.5002$ & $6 \times 10^{-6}$ & $-0.4935$ & $+1.3204$ \\
Cartpole three poles & $+0.4116$ & $0.73$ & $+0.4159$ & $+1.0937$ \\
Finger spin & $-0.6325$ & $3 \times 10^{-6}$ & $-0.6097$ & $+0.3829$ \\
Hopper stand & $-0.2397$ & --- & $-0.2151$ & $+0.6000$ \\
Humanoid stand & $+0.4165$ & $0.83$ & $+0.5093$ & $+1.1412$ \\
Manipulator bring ball & $+2.4806$ & $2 \times 10^{-4}$ & $+0.6216$ & $+0.5382$ \\
Pendulum swingup & $-0.7399$ & $1 \times 10^{-5}$ & $-0.7285$ & $+0.8168$ \\
Quadruped walk & $+29.7949$ & $3 \times 10^{-4}$ & $-0.2371$ & $-0.1430$ \\
Reacher hard & $+2.9868$ & $3 \times 10^{-4}$ & $+0.9779$ & $+0.0687$ \\
Swimmer (15 links) & $+2.66 \times 10^{7}$ & $0.43$ & $+6.6530$ & $+0.9427$ \\
Walker walk & $-0.0741$ & --- & $-0.0307$ & $+0.9733$ \\
\bottomrule
\end{tabular}
\end{wmtable}

The Swimmer median of $1.43\times10^7$ and $97\%$ of pairs above ten nats per
dimension identify broad conditional misfit. Reacher instead has a median
near $0.85$ with an extreme upper tail. The trimmed statistics in
Table~\ref{tab:e11ftail} summarize the retained pairs. The raw Gaussian
assigns moderate costs to many pairs in the constrained model's worst tail,
so those losses are specific to the fitted conditional model.

\subsection{Alternate observation-space diagnostics}
\label{app:metrics}

The frozen Walker checkpoints are evaluated in decoded $16\times16$ pixels,
fixed random CNN features, standardized state-probe outputs, and frozen
AE latents. All four metrics share trajectories within this evaluation, which
uses separate rollout draws from the preceding tables. Random CNN features
supply a nonlinear image geometry, not a validated perceptual metric.
The state probe is a two-hidden-layer multilayer perceptron (MLP) of width 512 trained on up to 60,000
real validation observations and 18 recorded physical observables. Its
$R^2=0.616$ is a checkpoint-selection score on the final ordered $20\%$ of
that pool, not an independent test score. Probe error, ambiguity, and failures
on generated images limit the physical interpretation.

\begin{wmtable}{\textbf{Alternate metrics on Walker walk.} Four energy distances on the same trajectories within this evaluation. Finite and in-box fractions are distinct; distances use the in-box subset. Values are means over up to four Monte Carlo repetitions at one checkpoint per arm, without training-seed uncertainty. The state column compares probe outputs with recorded standardized states, and its $T=0$ value includes probe error.}{tab:metrics}
\small\setlength{\tabcolsep}{4pt}
\begin{tabular}{@{}lrrrrrrr@{}}
\toprule
Arm & $T$ & Finite & In box & Pixels & Features & State probe & AE latent \\
\midrule
\textsc{ours} & $0$ & $1.0000$ & $1.0000$ & $0.0004$ & $0.0001$ & $0.0373$ & $0.0014$ \\
 & $1$ & $1.0000$ & $1.0000$ & $0.0011$ & $0.0003$ & $0.0787$ & $0.0072$ \\
 & $10$ & $1.0000$ & $1.0000$ & $0.0221$ & $0.0038$ & $0.1102$ & $0.1261$ \\
 & $100$ & $1.0000$ & $1.0000$ & $0.0301$ & $0.0072$ & $0.1190$ & $0.1608$ \\
 & $1000$ & $1.0000$ & $1.0000$ & $0.0300$ & $0.0073$ & $0.1221$ & $0.1595$ \\
 & $10000$ & $1.0000$ & $1.0000$ & $0.0285$ & $0.0065$ & $0.1186$ & $0.1544$ \\
 & $100000$ & $1.0000$ & $1.0000$ & $0.0290$ & $0.0068$ & $0.1176$ & $0.1582$ \\
\midrule
\textsc{gauss\_E} & $0$ & $1.0000$ & $1.0000$ & $0.0004$ & $0.0001$ & $0.0373$ & $0.0014$ \\
 & $1$ & $1.0000$ & $1.0000$ & $0.0035$ & $0.0011$ & $0.0875$ & $0.0140$ \\
 & $10$ & $1.0000$ & $0.6594$ & $0.1105$ & $0.0139$ & $0.1499$ & $1.2502$ \\
 & $100$ & $0.0000$ & $0.0000$ & --- & --- & --- & --- \\
 & $1000$ & $0.0000$ & $0.0000$ & --- & --- & --- & --- \\
 & $10000$ & $0.0000$ & $0.0000$ & --- & --- & --- & --- \\
 & $100000$ & $0.0000$ & $0.0000$ & --- & --- & --- & --- \\
\midrule
\textsc{free} & $0$ & $1.0000$ & $1.0000$ & $0.0004$ & $0.0001$ & $0.0373$ & $0.0014$ \\
 & $1$ & $1.0000$ & $1.0000$ & $0.0028$ & $0.0003$ & $0.0404$ & $0.0181$ \\
 & $10$ & $1.0000$ & $0.9998$ & $0.0075$ & $0.0007$ & $0.0642$ & $0.0274$ \\
 & $100$ & $0.9979$ & $0.9977$ & $0.0116$ & $0.0016$ & $0.0721$ & $0.0382$ \\
 & $1000$ & $0.9932$ & $0.9932$ & $0.0129$ & $0.0021$ & $0.0728$ & $0.0430$ \\
 & $10000$ & $0.9492$ & $0.9492$ & $0.0141$ & $0.0023$ & $0.0758$ & $0.0529$ \\
 & $100000$ & $0.6082$ & $0.6079$ & $0.0211$ & $0.0044$ & $0.0750$ & $0.0587$ \\
\bottomrule
\end{tabular}
\end{wmtable}
\begin{wmtable}{\textbf{Monte Carlo variation for the stationary model under alternate metrics.} Mean $\pm$ SE from four evaluation repetitions, using sample SD. Every chain is finite and inside the box. These SEs describe evaluation sampling at the fixed checkpoint and probe.}{tab:metricsse}
\small\setlength{\tabcolsep}{5pt}
\begin{tabular}{@{}rrrrr@{}}
\toprule
$T$ & Pixels & Features & State probe & AE latent \\
\midrule
$0$ & $0.0004\pm0.0004$ & $0.0001\pm0.0002$ & $0.0373\pm0.0022$ & $0.0014\pm0.0010$ \\
$1$ & $0.0011\pm0.0007$ & $0.0003\pm0.0001$ & $0.0787\pm0.0033$ & $0.0072\pm0.0021$ \\
$10$ & $0.0221\pm0.0006$ & $0.0038\pm0.0002$ & $0.1102\pm0.0048$ & $0.1261\pm0.0017$ \\
$100$ & $0.0301\pm0.0007$ & $0.0072\pm0.0003$ & $0.1190\pm0.0026$ & $0.1608\pm0.0021$ \\
$1000$ & $0.0300\pm0.0011$ & $0.0073\pm0.0003$ & $0.1221\pm0.0022$ & $0.1595\pm0.0030$ \\
$10000$ & $0.0285\pm0.0014$ & $0.0065\pm0.0002$ & $0.1186\pm0.0009$ & $0.1544\pm0.0039$ \\
$100000$ & $0.0290\pm0.0010$ & $0.0068\pm0.0001$ & $0.1176\pm0.0033$ & $0.1582\pm0.0044$ \\
\bottomrule
\end{tabular}
\end{wmtable}

These distances retain finite AE latents within a box whose center is the
data midpoint and whose half-width is five times its coordinate range. The
stationary arm keeps all chains; its late-distance ranges vary little under
each of the four geometries. The free flow has lower survivor-conditioned
distances in every geometry, while only $60.82\%$ remain finite and $60.79\%$
remain inside the box at the final horizon. The probe distance $0.0373$ already
present at initialization is a diagnostic reference, not a subtractable physical
error floor. All reported retained fractions remain essential to interpreting
these alternative metrics.

\section{Architectural contraction, fidelity cost, and recurrent memory}
\label{app:e25}

The cap experiment uses the sequential block scan of
Appendix~\ref{app:arch}. Its uniform
coefficient cap gives the $L^2$ operator bound of Theorem~\ref{thm:C}.
For finite initial density-ratio $L^2$ error, the sufficient horizon for a
100-fold relative reduction is
$\lceil\log(10^{-2})/\log c_\star\rceil$. This gives 44, 459, 4,603, and
46,050 steps for caps $0.9,0.99,0.999,0.9999$. The default numerical noise
floor $\sin\phi_{\min}=10^{-4}$ gives approximately $9.2\times10^8$ steps.
These relative guarantees are not absolute pixel-error tolerances and do not
apply directly to an atomic initial law. The parallel finite-start argument
in Section~\ref{sec:certificates} addresses the latter issue with a different construction.

\begin{wmtable}{\textbf{Angle-capped transitions on twelve pixel tasks.} All models use $c_\star=0.999$, giving a $100$-fold contraction of initial $L^2$ error within $4{,}603$ steps when that error is finite. $D_\infty$ and $D_{10^5}$ use the same eight paired evaluation repetitions; $t$ is their mean difference divided by its Monte Carlo standard error, computed with sample variance. $\widehat\sigma_2$ is a finite-basis estimate, not a certified lower endpoint. $\Delta$NLL compares with the frozen default-cap checkpoint. Dashes indicate unavailable rate estimates.}{tab:e25}
\small\setlength{\tabcolsep}{4pt}
\begin{tabular}{@{}lrrrrr@{}}
\toprule
Task & $\widehat\sigma_2$ & $D_\infty$ & $D_{10^5}$ & Paired $t$ & $\Delta$NLL \\
\midrule
\texttt{acrobot\_swingup} & --- & $0.146$ & $0.144$ & $-0.51$ & $-0.0034$ \\
\texttt{cartpole\_swingup} & --- & $0.211$ & $0.207$ & $-1.28$ & $-33.9489$ \\
\texttt{cartpole\_three\_poles} & $0.9613$ & $0.119$ & $0.117$ & $-0.87$ & $-0.0009$ \\
\texttt{finger\_spin} & $0.9773$ & $0.180$ & $0.182$ & $+0.68$ & $-1.1688$ \\
\texttt{hopper\_stand} & $0.9591$ & $0.189$ & $0.192$ & $+0.75$ & $-0.0006$ \\
\texttt{humanoid\_stand} & $0.9823$ & $0.049$ & $0.046$ & $-1.03$ & $-0.0029$ \\
\texttt{manipulator\_bring\_ball} & $0.9847$ & $0.232$ & $0.230$ & $-0.45$ & $-1.6\times10^{3}$ \\
\texttt{pendulum\_swingup} & $0.9815$ & $0.204$ & $0.212$ & $+0.51$ & $-2.8\times10^{3}$ \\
\texttt{quadruped\_walk} & $0.9906$ & $0.084$ & $0.087$ & $+1.54$ & $-4.1\times10^{3}$ \\
\texttt{reacher\_hard} & $0.9597$ & $0.309$ & $0.301$ & $-5.28$ & $-1.1\times10^{4}$ \\
\texttt{swimmer\_swimmer15} & $0.9928$ & $1.105$ & $1.123$ & $+0.72$ & $-2.7\times10^{7}$ \\
\texttt{walker\_walk} & $0.9816$ & $0.170$ & $0.166$ & $-1.85$ & $+0.0000$ \\
\bottomrule
\end{tabular}
\end{wmtable}

All twelve capped runs pass the specified flatness check between $10^4$ and
$10^5$: the difference is within three combined Monte Carlo SEs plus
$0.01D_\infty$. Both horizons exceed 4,603, so that observation does not
locate convergence or verify completion by the theoretical horizon.
Eleven of the twelve static-versus-terminal paired gaps have $|t|<3$.
Reacher's is $-0.0079\pm0.0015$ ($t=-5.28$); the static resampling of
Appendix~\ref{app:cpu-diagnostics} and the independent rollout replication of
Appendix~\ref{app:reacher-replication} are consistent with variation in the original static draw.

\begin{wmtable}{\textbf{Contraction horizon and fitted-model trade-offs.} $T^\star=\lceil\log(10^{-2})/\log c_\star\rceil$ bounds the horizon for a $100$-fold reduction of finite initial $L^2$ error. Default refers to the numerical noise floor $\sin\phi_{\min}=10^{-4}$, which gives a finite but impractical bound of approximately $9.2\times10^8$ steps. $\Delta$NLL is relative to the frozen default checkpoint for that task; these comparisons include variation between fitted checkpoints.}{tab:e25trade}
\small\setlength{\tabcolsep}{6pt}
\begin{tabular}{@{}llrrr@{}}
\toprule
Task & $c_\star$ & $T^\star$ & $D_\infty$ & $\Delta$NLL \\
\midrule
\texttt{pendulum\_swingup} & $0.9$ & $44$ & $0.128$ & $-2842.147$ \\
\texttt{pendulum\_swingup} & $0.99$ & $459$ & $0.108$ & $-2842.134$ \\
\texttt{pendulum\_swingup} & $0.999$ & $4603$ & $0.211$ & $-2842.084$ \\
\texttt{pendulum\_swingup} & $0.9999$ & $46050$ & $0.203$ & $-2841.520$ \\
\texttt{pendulum\_swingup} & Default & $9.2\times10^8$ & $0.076$ & $+0.000$ \\
\midrule
\texttt{reacher\_hard} & $0.9$ & $44$ & $0.337$ & $-11114.228$ \\
\texttt{reacher\_hard} & $0.99$ & $459$ & $0.322$ & $-11114.865$ \\
\texttt{reacher\_hard} & $0.999$ & $4603$ & $0.310$ & $-11114.873$ \\
\texttt{reacher\_hard} & $0.9999$ & $46050$ & $0.298$ & $-11114.307$ \\
\texttt{reacher\_hard} & Default & $9.2\times10^8$ & $0.306$ & $+0.000$ \\
\midrule
\texttt{quadruped\_walk} & $0.9$ & $44$ & $0.091$ & $-4069.268$ \\
\texttt{quadruped\_walk} & $0.99$ & $459$ & $0.108$ & $-4068.519$ \\
\texttt{quadruped\_walk} & $0.999$ & $4603$ & $0.084$ & $-4064.337$ \\
\texttt{quadruped\_walk} & $0.9999$ & $46050$ & $0.092$ & $-4009.974$ \\
\texttt{quadruped\_walk} & Default & $9.2\times10^8$ & $0.118$ & $+0.000$ \\
\bottomrule
\end{tabular}
\end{wmtable}

Tighter caps shorten the relative-contraction bound but do not order fitted
fidelity or held-out likelihood. Moving from $0.999$ to $0.9$ improves
conditional NLL on Pendulum and Quadruped while worsening Reacher by about
$0.65$ nats per dimension. Default-cap retraining still gives very poor pair
NLL on Pendulum and Quadruped ($9,123$ and $25,438$, compared with frozen
values $2,841$ and $4,069$). Thus the extreme default losses persist under
these additional fits. Static distances also change: the best capped level is
worse than the matched default on Pendulum and Reacher, and better on Quadruped.
The shortest bound and smallest static discrepancy occur at different cap settings.

On matched Lorenz seeds, cap
$0.999$ increases pair NLL by $0.075$--$0.162$, and cap $0.99$ by
$0.376$--$0.503$. Enforcing a noise floor can therefore cost accuracy on sharp
conditionals while reducing extreme held-out log-loss on other tasks.

\subsection{Deterministic memory and partial noise}
\label{app:e26}\label{app:rssm}

The synthetic memory comparison has $d_h=d_z=8$ and trains six configurations
for 4,000 updates. Rotation-only memory conserves $\|h\|^2$ pathwise; quantile
swirls remove that radius invariant but do not by themselves supply strict
one-step contraction. A cap on the noisy $z$ block does not certify the whole
$(h,z)$ state when deterministic memory remains. The recurrent pixel model of
Appendix~\ref{app:gpu-protocol} instead includes every persistent coordinate in its stochastic prior.

The initial memory variance is $1.44$, with the other coordinates standard
Gaussian, giving finite divergence
\[
 \chi^2(\mu_0\|\gauss)
 =\{1.44(2-1.44)\}^{-d_h/2}-1=1.365.
\]
Finite divergence alone does not imply contraction.

\begin{wmtable}{\textbf{Relaxation with finite initial divergence.} All models have $d_h=d_z=8$ and start with memory variance $1.44$ ($\chi^2=1.365$). The operator column is a sampled estimate, not a bound. $c_\star$ is the architectural upper bound, and $T_{\mathrm{rel}}$ is sufficient for a 100-fold reduction of $L^2$ error. The default noise floor gives $c_\star=\sqrt{1-10^{-8}}$. Distances at $T=1000$ are normalized by a single independent-Gaussian control distance; this ratio has no calibrated significance threshold.}{tab:e26}
\small\setlength{\tabcolsep}{5pt}
\begin{tabular}{@{}llrrrr@{}}
\toprule
Memory update & Cap & $\hat\sigma$ & $T_{\mathrm{rel}}$ & $\Exp\lVert h_T\rVert^2$ & $D_T/D_{\mathrm{ctrl}}$ \\
\midrule
Joint update & Default & $0.974$ & $9.2\!\times\!10^8$ & $8.01$ & $1.24$ \\
Rotation & $1$ & $1.002$ & --- & $11.51$ & $11.64$ \\
Quantile swirl & $1$ & $1.016$ & --- & $11.45$ & $11.66$ \\
Stochastic memory & Default & $0.973$ & $9.2\!\times\!10^8$ & $8.02$ & $1.09$ \\
Rotation & $1$ (requested .99) & $1.002$ & --- & $11.51$ & $12.45$ \\
Stochastic memory & $0.99$ & $0.973$ & $459$ & $8.02$ & $0.99$ \\
\bottomrule
\end{tabular}
\end{wmtable}

Rotation memory retains second moment $11.51$ instead of eight. The quantile
swirl remains near that value at 1,000 steps, a finite-horizon non-relaxation
rather than evidence about asymptotic mixing. Both stochastic-memory models
reach second moments near eight and distances near the single Gaussian control.
With cap $0.99$ on every updated block, the relative horizon is 459 and the
separate absolute bound in the paper's $L^1$-TV convention is
$\|\mu_0K^T-\gauss\|_{\mathrm{TV}}\le0.99^T\sqrt{1.365}\le0.01$
for $T\ge474$ (total variation, not the energy statistic of
Table~\ref{tab:e26}). The stochastic-memory pair has NLL $0.4130$ with the cap
and $0.4153$ with the default floor --- one fitted pair, not a general cost
estimate; deterministic-memory scores concern fewer coordinates and are not
directly comparable to that pair.

\section{Static-reference and finite-start diagnostics}
\label{app:cpu-diagnostics}

\subsection{Independent static resampling at a fixed Reacher checkpoint}
\label{app:static-resampling}

The static estimate in 256-dimensional frozen-autoencoder coordinates, using unnormalized
Euclidean distance, is resampled at the capped Reacher checkpoint
independently: eight fresh Gaussian clouds of 2,048 points per reference pool,
averaged within pool, giving eight reference-level values for uncertainty
estimation. The 64 clouds are not treated as 64 independent replications of
the reference data or model; the recorded terminal rollout samples are reused.

\CPUStaticAuditTable

The resampled static average sits below the recorded static draw, and both
recorded rollout laws are closer to the resampled average than to the recorded
draw, with unresolved residual contrasts (Table~\ref{tab:cpu-static}). Scoring
the same Gaussian sample with float32 and float64 decoder arithmetic changes
this static distance by about $9\times10^{-9}$, six orders of magnitude below
the sampling contrast. An independent replication with 32 fresh $10^5$-step
rollout pairs also gives a zero-compatible gap
(Appendix~\ref{app:reacher-replication}). Together, the results are consistent with
sampling variation in the original static draw.

\subsubsection{Independent Reacher rollout replication}
\label{app:reacher-replication}
The capped Reacher checkpoint is evaluated with 32 fresh paired repetitions of
2,048 points each at $T=100000$. The data-initialized gap $D_T-D_{\rm ref}$ is
$0.001711\pm0.002227$, with interval $[-0.002831, 0.006252]$ and cross/within
sample covariance $-0.007606$. The Gaussian-initialized gap
$D_T^{\gamma}-D_{\rm ref}$ is $0.0008676\pm0.002131$, with interval
$[-0.003478, 0.005213]$ and cross/within sample covariance $-0.005885$.
The initialization contrast $D_T-D_T^{\gamma}$ is $0.000843\pm0.002063$, with
interval $[-0.003365, 0.005051]$ and cross/within sample covariance $-0.008885$.
Uncertainties are sample SEs with 95\% Student intervals, paired per repetition.

Action, initialization, Gaussian-start, transition, reference, and static streams
use explicit seeds. Disjointness from the original evaluation is verified for all
except the transition stream, whose original seed is unavailable. The paired
uncertainty conditions on this one checkpoint and the fixed held-out data pool.
This evaluation uses frozen-autoencoder coordinates and clipped OU actions;
it measures rollout-sampling variation without replicating training.

\subsection{Reference mixtures bound losses without restoring dynamics}
\label{app:reference-mixture}

This rescoring measures how much of the conditional density is supplied by the
reference component when a mixture controls extreme loss.
For frozen checkpoint density $p$ and its decoded-reference density $p_E$,
the post-hoc mixture
\[
 p_r(x'\mid x,a)
 =(1-r)p(x'\mid x,a)+r p_E(x')
\]
retains reference preservation when $p$ has it. It bounds per-dimension NLL
above by $[-\log p_E(x')-\log r]/d$, and increases the original NLL
by at most $-\log(1-r)/d$ pointwise. For
$r=10^{-3}$ and $d=256$, the latter is
$3.9082\times10^{-6}$. The rescoring evaluates all validation pairs at this
common reset probability, without retraining or new rollouts.

\CPUGuardAuditTable

Reacher's default extreme losses become bounded by the reference component;
the median changes little. Its worst-pair decomposition identifies very small
conditional scales and large standardized residuals as the main source of the
loss, rather than the state-chart Jacobian. The corresponding capped checkpoint
is much better conditioned on those pairs. The CPU and GPU mean scores differ
by approximately $0.6\%$ for the uncapped Reacher model; both reveal the same
severe conditional tail.

For Swimmer, the reference component has posterior density responsibility above one half on
$99.15\%$ of default-checkpoint pairs and $98.26\%$ of capped-checkpoint pairs;
its sampling probability remains $0.1\%$. The guarded default
median remains about $6.70$ nats per dimension, and $12.66\%$ of pairs remain
above ten. The reduction in scored log-loss is supplied mainly by the static reference component.

\subsection{Untrained parallel kernels and atomic starts}
\label{app:atomic-diagnostics}

The parallel-kernel CPU experiment uses random, untrained weights. It checks
Gaussian starts, a fixed uniform 16-point bank in eight dimensions, and exact
conditional tanh-Gaussian policy shifts. The bank's initial KL to the Gaussian
is infinite, while its saved second moment is finite. The implemented bound
therefore begins after the first noisy transition. Bounded-observable intervals
use Hoeffding's inequality with a union bound over at most 200 comparisons and
family failure probability $0.01$; the independent unit is a rollout chain at
a fixed horizon. All 114 reported comparisons are compatible with their bounds.
These finite-sample checks test implementation consistency.

\CPUAnalyticBoundTable

The two-dimensional analytic specialization supplies a complementary direct
numerical integral. Its one-step implemented density agrees with the analytic
conditional Gaussian density at the checked inputs, and the reported mixture KL and $L^1$-TV integrals lie
below their respective finite-start bounds. Quadrature and finite-sample checks
do not replace the proof. In a conditioning negative control, both correct
$\alpha\sim\gauss_1$ independent of $z\sim\gauss_2$ and incorrect
$\alpha=-z_1$ have the correct marginal action law. After the same $\pi/4$
joint rotation and coefficient-$0.8$ noise update, the first-coordinate
variances are $0.998$ and $1.650$, against theoretical values $1$ and $1.64$.
Thus marginal Gaussian action fit alone does not supply the conditional chart.

\CPUScalingTable

The timing experiment is a four-thread CPU microbenchmark of the prior alone.
At latent widths $256$, $512$, and $1024$, the kernel uses $0.30$, $0.48$, and $0.85$ million
parameters (Table~\ref{tab:cpu-scaling}; timing medians are plotted in
Figure~\ref{fig:certificate}c). The benchmark excludes observation networks, optimization, and
data loading; each fitted recurrent checkpoint additionally records GPU memory and full-model
throughput.
Numerical GPU results for the recurrent model are
reported under Appendix~\ref{app:gpu-protocol}.

\section{Controlled occupancy and downstream planning}
\label{app:controlled-exp}

\subsection{A system with an exact conditional action chart}
\label{app:controlled-gaussian}

This system separates the cost of an unconditional action chart from the cost
of requiring preservation at every action.
The controlled Gaussian data follow
\begin{equation}
 a\mid z\sim\mathcal N(\tau z,1-\tau^2),\qquad
 z'=\varrho a+\sqrt{1-\varrho^2}\,\xi,\qquad \xi\sim\gauss_1.
 \label{eq:e12sys}
\end{equation}
Starting with $z\sim\gauss_1$ gives standard Gaussian state and action
marginals for $|\tau|,|\varrho|<1$. Their dependence changes with $\tau$.
The exact chart $\alpha=(a-\tau z)/\sqrt{1-\tau^2}$ removes that dependence;
a marginal action transform cannot, since $a$ is already marginally Gaussian.
The excess action-chart NLL from modeling $H(a)$ instead of $H(a\mid z)$ is
$-\frac12\log(1-\tau^2)$, or $0.2231$ and $0.8304$ at $\tau=0.6,0.9$.
This chart cost is distinct from transition NLL.

For independent actions, the optimal excess transition NLL over all
per-action-preserving kernels is $-\frac12\log(1-\varrho^2)$
(Proposition~\ref{prop:controlled-floor}). For dependent actions we optimize
the narrower implemented family
$z'=c(a)z+\sqrt{1-c(a)^2}\varepsilon$ by minimizing its Gaussian conditional
cross entropy for each action and integrating over $a\sim\gauss_1$.
At $\varrho=0.8$, its excess costs are $0.5108,0.3681,0.1408$ for
$\tau=0,0.6,0.9$. Only the first is an unrestricted per-action-class floor.
A $250\times250$ copula-grid optimization followed by truncated 24-node
action quadrature gives approximately $0.32$ and $0.04$ for the latter two
settings. These approximate, uncertified class optima show that the scalar
OU restriction contributes materially to the measured advantage.

\begin{wmtable}{Controlled Gaussian conditional fit. Excess transition NLL is relative to the true conditional entropy. The scalar OU optimum equals the unrestricted per-action floor only at $\tau=0$. Each configuration has one training seed.}{tab:controlled}
\small\setlength{\tabcolsep}{6pt}
\begin{tabular}{@{}cccrrrr@{}}
\toprule
& & & \multicolumn{2}{c}{Transition excess NLL} & \multicolumn{2}{c}{Per-action OU family} \\
\cmidrule(lr){4-5}\cmidrule(lr){6-7}
$\varrho$ & $\tau$ & Chart & Occupancy & Per-action & Optimum & Loss/optimum \\
\midrule
$0.3$ & $0.0$ & Conditional & $0.0011$ & $0.0470$ & $0.0472$ & $1.00$ \\
$0.5$ & $0.0$ & Conditional & $0.0011$ & $0.1428$ & $0.1438$ & $0.99$ \\
$0.8$ & $0.0$ & Conditional & $0.0011$ & $0.5089$ & $0.5108$ & $1.00$ \\
$0.95$ & $0.0$ & Conditional & $0.0011$ & $1.1623$ & $1.1640$ & $1.00$ \\
$0.8$ & $0.6$ & Conditional & $0.0011$ & $0.3668$ & $0.3681$ & $1.00$ \\
$0.8$ & $0.6$ & Marginal & $0.0011$ & $0.3667$ & $0.3681$ & $1.00$ \\
$0.8$ & $0.9$ & Conditional & $0.0011$ & $0.1399$ & $0.1408$ & $0.99$ \\
$0.8$ & $0.9$ & Marginal & $0.0011$ & $0.1399$ & $0.1408$ & $0.99$ \\
\bottomrule
\end{tabular}

\end{wmtable}

Each configuration uses 400,000 training and 200,000 held-out transitions,
20,000 updates, and 32,768 Gaussian-start chains through $10^5$ steps.
The occupancy model's approximately $0.0011$-nat excess is a finite-sample
fit statistic, not a certified optimization optimum.

\begin{wmtable}{\textbf{Action-chart and closed-loop diagnostics.} A dash indicates an unrecorded chart diagnostic; correlation refers to the chart output and state in the occupancy model. Second moments are measured at $T=10^5$ under the behaviour policy, and all listed chains remain finite. The final column is a separate 200-step Gaussian-start KS diagnostic for the per-action model under that policy. It is not a test of preservation for a fixed action.}{tab:e12controls}
\small\setlength{\tabcolsep}{5pt}
\begin{tabular}{@{}cccrrrr@{}}
\toprule
& & & & \multicolumn{2}{c}{$\Exp|z_{10^5}|^2$} & \\
\cmidrule(lr){5-6}
$\varrho$ & $\tau$ & Chart & $\mathrm{Corr}(\alpha,z)$ & Occupancy & Per-action & Per-action KS \\
\midrule
$0.3$ & $0.0$ & Conditional & --- & $1.0071$ & $0.9936$ & $0.0014$ \\
$0.5$ & $0.0$ & Conditional & --- & $1.0089$ & $0.9936$ & $0.0015$ \\
$0.8$ & $0.0$ & Conditional & --- & $1.0107$ & $0.9936$ & $0.0014$ \\
$0.95$ & $0.0$ & Conditional & --- & $1.0102$ & $0.9936$ & $0.0015$ \\
$0.8$ & $0.6$ & Conditional & $0.0011$ & $1.0108$ & $1.0666$ & $0.0089$ \\
$0.8$ & $0.6$ & Marginal & $0.6003$ & $1.0107$ & $1.0666$ & $0.0088$ \\
$0.8$ & $0.9$ & Conditional & $0.0039$ & $1.0132$ & $1.1256$ & $0.0179$ \\
$0.8$ & $0.9$ & Marginal & $0.9004$ & $1.0130$ & $1.1256$ & $0.0179$ \\
\bottomrule
\end{tabular}
\end{wmtable}
\begin{wmtable}{\textbf{Held-out chart likelihood.} Costs are nats per action coordinate. The entropy is analytic for the Gaussian behaviour policy. The difference from this entropy is an empirical excess cost and can be slightly negative from sampling.}{tab:e12chart}
\small\setlength{\tabcolsep}{7pt}
\begin{tabular}{@{}cccrrr@{}}
\toprule
$\varrho$ & $\tau$ & Chart & NLL & $H(a\mid z)$ & Excess \\
\midrule
$0.8$ & $0.6$ & Conditional & $1.1953$ & $1.1958$ & $-0.0005$ \\
$0.8$ & $0.6$ & Marginal & $1.4186$ & $1.1958$ & $+0.2228$ \\
$0.8$ & $0.9$ & Conditional & $0.5880$ & $0.5886$ & $-0.0005$ \\
$0.8$ & $0.9$ & Marginal & $1.4204$ & $0.5886$ & $+0.8318$ \\
\bottomrule
\end{tabular}
\end{wmtable}

At $\tau=0.9$, the fitted conditional and marginal charts have state
correlations $0.0039$ and $0.9004$; their chart NLL gap is $0.8324$ nats.
These check the learned transform against the analytic $0.8304$-nat reference.
The ideal transform, rather than this finite correlation test, gives exact
independence.

\paragraph{A fixed-model dependence intervention.}\label{app:controlled-guard}
A raw-action-conditioned OU kernel preserves the reference for each external
action while potentially failing under state-dependent actions. At
$\tau=0.9$, the fitted model has 200-step KS discrepancy $0.0179$ under the
behaviour policy versus a Gaussian control of $0.0020$.
A separate fixed-model intervention uses 200,000 training transitions, 8,000
updates, and 2,000 rollout steps. State-dependent driving gives second moment
$1.115$; independent driving of the same model gives $0.9985$, with KS
$0.0014$ versus control $0.0015$. The action marginals agree under the reference
state law, while their state dependence differs. This intervention isolates
that dependence in the tested system.

\subsection{Planning on the stochastic Duffing oscillator}
\label{app:planning}

This comparison measures planning error under learned transitions with different
preservation constraints.
The system is $\dd x=v\dd t$ and
$\dd v=(-\delta v-\alpha_D x-\beta x^3+\mathrm{drive}\,a)\dd t+\sigma\dd W$,
with $\dd t=0.05$ and
$(\alpha_D,\beta,\delta,\mathrm{drive},\sigma)=(-1,1,0.2,2.5,0.25)$. Collection uses
$a_t=\operatorname{clip}(0.6\tanh x_t+u_t,-1,1)$ with an OU driver $u_t$.
The learned action chart is therefore approximate: clipping creates boundary
mass and the driver retains memory. This planning dataset does not have the
known nonsingular innovation chart used in the recurrent protocol of
Appendix~\ref{app:gpu-protocol}.

Each model uses the same 400,000 transitions and 20,000 updates. The state
chart and action chart have 534,576 and 78,944 parameters; the occupancy
transition has 503,004 and the Gaussian transition 792,070. The ablated mixer
has 269,364 parameters. The per-action arm feeds the raw action to its
transition, so its name here differs from the conditioner-only ablation of the
recurrent study.
The objective adds total state-marginal, action-chart, and conditional NLL;
the implementation does not divide the two-dimensional conditional NLL by two.

CEM uses 512 candidates, four iterations, and 64 elites to minimize a
swing-up-and-hold cost. The real simulator executes each selected action.
Three seeds share evaluation episodes and planner randomness across models
at horizons $5,15,50,150,500$. The independent replication unit is the seed,
with correlated repeated horizons within a seed.

\begin{wmtable}{\textbf{Paired planning costs on the controlled Duffing oscillator.} Oracle cost and each model's excess cost $\Delta$ over the oracle are mean $\pm$ sample standard deviation across three seeds. Models share evaluation episodes and planner seeds within each replication. Horizons are repeated measurements; they are not treated as independent samples.}{tab:planning}
\small\setlength{\tabcolsep}{4pt}
\begin{tabular}{@{}rrrrr@{}}
\toprule
$H$ & Oracle cost & Ours: $\Delta$ & Per-action: $\Delta$ & Gaussian: $\Delta$ \\
\midrule
$5$ & $0.204\pm0.083$ & $-0.004\pm0.004$ & $+2.047\pm0.988$ & $+0.003\pm0.002$ \\
$15$ & $0.212\pm0.091$ & $-0.000\pm0.007$ & $+1.746\pm0.802$ & $+0.002\pm0.005$ \\
$50$ & $0.253\pm0.117$ & $+0.005\pm0.022$ & $+1.630\pm0.789$ & $+0.003\pm0.004$ \\
$150$ & $0.433\pm0.250$ & $+0.016\pm0.099$ & $+1.441\pm0.630$ & $-0.013\pm0.023$ \\
$500$ & $0.796\pm0.550$ & $+0.242\pm0.181$ & $+1.060\pm0.745$ & $+0.001\pm0.015$ \\
\bottomrule
\end{tabular}
\end{wmtable}
\begin{wmtable}{\textbf{Oracle-normalized calibration gaps.} The gap is the absolute difference between predicted and achieved cost, divided by the oracle's gap at the same horizon and seed. Entries are mean $\pm$ sample standard deviation across three seeds; the oracle is $1$ by definition.}{tab:planning-calibration}
\small\setlength{\tabcolsep}{10pt}
\begin{tabular}{@{}rrrr@{}}
\toprule
$H$ & Ours & Per-action & Gaussian \\
\midrule
$5$ & $1.44\pm0.07$ & $2.27\pm0.37$ & $1.04\pm0.03$ \\
$15$ & $1.48\pm0.14$ & $3.06\pm0.30$ & $0.97\pm0.08$ \\
$50$ & $1.43\pm0.34$ & $4.87\pm0.48$ & $0.99\pm0.14$ \\
$150$ & $1.18\pm0.20$ & $2.28\pm0.38$ & $1.07\pm0.07$ \\
$500$ & $0.73\pm0.15$ & $0.96\pm0.10$ & $1.01\pm0.06$ \\
\bottomrule
\end{tabular}
\end{wmtable}

The Gaussian baseline is close to the oracle at every reported horizon. The
occupancy model is also close at short horizons but has mean excess $0.242$
at horizon 500. The per-action arm's excess cost is positive in all 15
seed--horizon cells. Averaging horizons within each seed gives excess costs
$1.325,0.976,2.453$ for per-action, $-0.009,0.089,0.076$ for occupancy, and
$-0.001,0.002,-0.003$ for Gaussian. A two-sided sign test with only three
independent seeds gives $p=0.25$ for per-action and $p=1$ for each other arm.

The fixed-budget oracle itself becomes less effective at long horizons:
its cost increases from $0.204$ to $0.796$. The normalized calibration ratio
therefore depends strongly on the oracle denominator and is not monotone in
horizon. The per-action deficit also coincides with worse conditional fit
($-8.70$ versus $-9.34$ nats per transition). The comparison does not isolate
occupancy drift as its cause. One Gaussian seed has a nonfinite dispersion
diagnostic by $10^3$ autonomous steps despite $98.86\%$ survival, while its
planner remains competitive, showing that autonomous-rollout and planning evaluations
measure different properties.

\section{Numerical exactness and structural diagnostics}
\label{app:exactness}

\subsection{Independent density and preservation checks}
\label{app:rigidity}

For a diffeomorphism $T$, change of variables gives
\begin{equation}
 \KL(T_\#\gauss\|\gauss)
 =\Exp_{z\sim\gauss}\left[
 \tfrac12(\|T(z)\|^2-\|z\|^2)-\log|\det J_T(z)|\right].
 \tag{$\ddagger$}\label{eq:ddagger}
\end{equation}
A preserving map makes the integrand zero pointwise. A small sampled mean
alone is insufficient because large signed residuals can cancel. Independent
automatic differentiation and general matrix factorization check the analytic
Jacobian against this identity.

\begin{wmtable}{\textbf{Pointwise Gaussian identity.} Six perturbed transition families, 8,192 inputs, $d=32$, CPU double precision. The analytic log-determinant is independently checked against 16 autograd Jacobians per family. Ranges show the two sides of \eqref{eq:dagger}; exactness is an architectural property.}{tab:pointwise}
\small\setlength{\tabcolsep}{4pt}
\begin{tabular}{@{}lrccc@{}}
\toprule
Layer & Max identity defect & $\log|\det J|$ & $\tfrac12\Delta\|z\|^2$ & Exact? \\
\midrule
Rotation & $3.1\times10^{-13}$ & $[0,0]$ & $[0,0]$ & yes \\
Quantile swirl & $4.1\times10^{-13}$ & $[-0.015,+0.014]$ & $[-0.015,+0.014]$ & yes \\
Grid swirl & $4.2\times10^{-13}$ & $[-0.074,+0.064]$ & $[-0.074,+0.064]$ & yes \\
Rotation + quantile & $4.6\times10^{-13}$ & $[-0.012,+0.012]$ & $[-0.012,+0.012]$ & yes \\
Affine & $7.8\times10^{-1}$ & $[-0.309,+0.121]$ & $[-0.856,+0.567]$ & no \\
Additive & $2.2\times10^{-1}$ & $[0,0]$ & $[-0.194,+0.217]$ & no \\
\bottomrule
\end{tabular}
\end{wmtable}

The additive control's zero log-determinant does not preserve the Gaussian:
its identity residual is $0.217$. The largest analytic-versus-autograd
log-determinant discrepancy among the tested map families is
$8.8\times10^{-14}$. A separate bijective-noise sampler check compares
$\log p(z'\mid z,a)$ with
$\log\gauss(\varepsilon)-\log|\det(\partial z'/\partial\varepsilon)|$.
Across seven cells with dimensions 4--16, one to four noise blocks, and
parameter perturbations through scale three, its maximum discrepancy is
$4.4\times10^{-12}$. These are finite numerical checks, not proofs for
arbitrary weights or conditioning.

Orthogonal-map tests on 8,192 inputs show residual scales near $10^{-14}$
in float64, $10^{-6}$ in float32, and $10^{-2}$ with TF32 enabled. Exactness
checks disable TF32. Checkpoint geometry checks also find small, nonzero
floating-point orthogonality defects. Neither these local measurements nor
agreement of a few marginal moments supplies a uniform long-run roundoff bound.

\begin{wmtable}{Sequential-scan rotation--OU arithmetic check: 16,384 chains, dimension 32, two noise blocks, float32 with TF32 disabled, and independent actions. Selected horizons from the complete recorded sweep. The independent Gaussian row has the same sample size. This checks numerical consistency at fixed weights, not a uniform error theorem.}{tab:sixdecades}
\small\setlength{\tabcolsep}{8pt}
\begin{tabular}{@{}rrrrr@{}}
\toprule
Horizon & RMS & Maximum mean & Maximum covariance defect & Kurtosis \\
\midrule
$1$ & $0.997796$ & $0.0156$ & $0.0266$ & $2.9885$ \\
$10^2$ & $1.000537$ & $0.0172$ & $0.0326$ & $2.9993$ \\
$10^4$ & $0.999767$ & $0.0181$ & $0.0291$ & $2.9894$ \\
$10^6$ & $1.001119$ & $0.0158$ & $0.0325$ & $2.9985$ \\
Independent Gaussian & $1.000560$ & $0.0204$ & $0.0247$ & $3.0110$ \\
\bottomrule
\end{tabular}
\end{wmtable}

The conditioning graph is equally important: additive layers with angle-tied
noise give RMS $1.0014$ but KS $0.0361$, and choosing identity for $|x|<1$
with an independent Gaussian refresh otherwise gives standard deviation
$0.7186$ and KS $0.1087$, versus $1.0003$ and $0.0023$ when the branches are
selected by an independent draw instead. Each component preserves the
reference; selection based on the coordinate being updated need not.

\subsection{Architecture scope: parallel prior and sequential scan}
\label{app:arch}\label{app:capacityceiling}

The main construction rotates the joint state--innovation input, reads only the
transformed innovation coordinates to determine its diagonal noise coefficients,
and updates all state coordinates in parallel. Its postmaps are Gaussian
rotations with unit Jacobian determinant. A full-reference reset gives an exact
two-component density mixture. The cap applies to every persistent stochastic
state coordinate. The proof of its finite-start entropy bound concerns this
conditioning graph and must not be inferred from a finite-basis operator estimate.

The sequential-scan experiments use a scan whose block-$i$ conditioner
reads $(y_{<i},x_{>i},a)$ and excludes the updated $x_i$. Rotation couplings and
CDF-conjugated quantile swirls surround the scan. A quantile swirl uses smooth,
compactly supported planar rotations on the open unit cube; conjugating by the
Gaussian CDF balances its radius change against its Jacobian. These maps can
preserve the Gaussian without preserving Euclidean volume in Gaussian coordinates.
The sequential-scan $L^2$ bound (Theorem~\ref{thm:C}) and the parallel bound
(Section~\ref{sec:certificates}, Appendix~\ref{app:parallel-certificate}) are
proved under their respective conditioning graphs.

\paragraph{Finite copulas and noise-only maps.}\label{app:spend}
The copula family used with the sequential scan mixes Gaussian copulas and finite cosine expansions
with positive normalized weights and zero-integral cosine basis functions.
A sufficient coefficient bound enforces positivity while retaining uniform
margins. Its conditioner excludes the updated coordinate; sampling inverts
the monotone conditional CDF using 40 double-precision bisection steps.
This is a finite preserving subfamily of the scalar class in
Theorem~\ref{thm:L}. Gaussian-preserving transformations of fresh noise alone
leave its conditional law, hence the transition kernel, unchanged
(Lemma~\ref{lem:N}). This observation does not restrict useful state--noise
interactions or complementary-coordinate conditioners.

\subsection{A scalar illustration of error sensitivity}
\label{app:knife}

At $\lambda=0.999$, the coefficient perturbation $\delta=10^{-4}\lambda$ is a relative
error of $0.01\%$. It increases the stationary variance by about $11\%$ while costing
$2.5\times10^{-6}$ nats of conditional KL averaged under the true unit-variance state law
(Proposition~\ref{prop:knife} states and proves the calculation). Tying the
fitted coefficient and scale through $m^2+s^2=1$ removes this particular
variance error.

\section{Capacity and operator diagnostics}
\label{app:capacity}\label{sec:exp:capacity}\label{app:certest}

\subsection{Finite transition capacity}
\label{app:finite-capacity}

This experiment compares the finite transition families on a fixed preserving target.
Combining rotations and swirls improves on either
primitive alone. Broader stochastic dependence also helps, but the unconstrained
models remain better at this training budget. Their substantial late-training
improvement makes these outcomes fitted-model comparisons, rather than
converged approximation limits.

\begin{wmtable}{Conditional capacity at a fixed Gaussian reference with no learned encoder. All models use 6,000 updates on the same preserving target kernel. NLL is per latent dimension. Deterministic depth and noise expressivity are distinct changes.}{tab:kernelonly}
\small\setlength{\tabcolsep}{6pt}
\begin{tabular}{@{}lrr@{}}
\toprule
Transition & Parameters & NLL/dim \\
\midrule
Quantile swirls, two OU blocks & $670,406$ & $0.6981$ \\
Rotations, two OU blocks & $672,366$ & $0.6609$ \\
Both primitives, two OU blocks & $671,386$ & $0.6265$ \\
Both, twice the deterministic depth & $1,209,134$ & $0.5943$ \\
Both, four OU blocks & $738,202$ & $0.5824$ \\
Both, two copula blocks & $705,442$ & $0.5394$ \\
Unconstrained affine stack & $607,089$ & $-0.6276$ \\
Unconstrained conditional flow & $544,872$ & $-0.7622$ \\
Unconstrained flow, twice the depth & $1,089,744$ & $-0.9843$ \\
\bottomrule
\end{tabular}
\end{wmtable}

The kernel-only copula improvement does not automatically transfer to joint
chart--transition learning: at 60,000 updates with three noise blocks, the
OU pair NLL is $2.1777$ nats per dimension against $2.3221$, $2.1895$, and
$2.3104$ for the three broader copula variants despite their $3$--$10\%$ more
parameters. The representation theorem describes an ideal class of stationary
pair laws; it does not certify the capacity or optimization of these finite
families.

\subsection{A non-Gaussian copula example}
\label{app:chartsub}

Let $z,\alpha$ be independent standard Gaussians and define
\[
 s=\operatorname{frac}(\Phi(z)+\Phi(\alpha)),\qquad
 c(s,q')=1+0.9\cos(2\pi s)\cos(\pi q'),\qquad z'=\Phi^{-1}(q').
\]
The density is at least $0.1$ and has uniform margins. For each fixed
$\alpha$, $s$ is uniform when $z$ is Gaussian, so the kernel also preserves
the Gaussian action by action. It separates a scalar Gaussian-copula
restriction from the occupancy-versus-per-action distinction.
Numerical quadrature gives information $\int c\log c=0.111716$ nats.
For $x=\Phi^{-1}(s)$ and $y=\Phi^{-1}(q')$, symmetry gives $\Exp[xy]=0$.
A global Gaussian-copula correlation $\rho$ therefore has expected log density ratio
$-\frac12\log(1-\rho^2)-\rho^2/(1-\rho^2)\le0$, with equality at independence.
That scalar family loses the entire information in this example.

\begin{wmtable}{\textbf{A non-Gaussian copula witness.} Scores are held-out log density ratios relative to independent Gaussian output, in nats, with mean $\pm$ sample SD across three training/data seeds. The truth has population score $0.111716$ by numerical quadrature. The adaptive Gaussian arm is a conditional-density baseline without a guaranteed Gaussian output margin.}{tab:chartsub}
\small\setlength{\tabcolsep}{5pt}
\begin{tabular}{@{}lrrr@{}}
\toprule
Family & Parameters & Recovered score & Truth minus score \\
\midrule
Global Gaussian copula (analytic optimum) & $1$ & $0$ & $0.1117$ \\
Input-dependent Gaussian conditional & $33,537$ & $0.0636\pm0.0005$ & $+0.0481$ \\
Cosine copula, order 3 & $9$ & $0.1123\pm0.0005$ & $-0.0005$ \\
\bottomrule
\end{tabular}
\end{wmtable}

The order-three learned cosine basis constrains
$\sum_{kl}|\beta_{kl}|\le0.499$ and contains the truth at
$\beta_{2,1}=0.45$. Its three-seed score $0.1123\pm0.0005$ uses sample SD;
the slight excess over population information is held-out sampling variation.
The adaptive Gaussian arm recovers $0.0636\pm0.0005$ by letting its coefficient
depend on the scalar being updated. It is a normalized conditional density,
but violates the exact noise stage's conditioning rule and has no guaranteed
Gaussian output marginal. Its remaining gap is a fitted-family result, not an
ablation of an exact preserving component. The cosine basis contains this
specific truth, without parameterizing every positive copula.

\subsection{Finite-basis operator estimates}
\label{app:operator-estimates}

Population compressions onto a finite basis cannot certify the full Markov
operator norm. Sampling error can further inflate their estimates, while
truncation can omit slow modes. Constant-OU controls recover known norms within
$0.0047$ and orthogonal deterministic maps return norm one. A non-normal
Gaussian example has singular norm $0.900$ but spectral radius $0.474$,
showing why those quantities cannot be interchanged. An OU step with
coefficient $0.9$ followed by a fair global sign flip removes odd correlations:
a linear-only estimate is $0.004$, while quadratic estimates $0.806$ and
$0.812$ detect the true quadratic mode $0.810$.

\end{document}